\documentclass[twoside,11pt]{article}

\usepackage{jmlr2e}

\firstpageno{1}

\usepackage[utf8]{inputenc}
\usepackage{amssymb}
\usepackage{amsmath}
\usepackage{mathtools}
\usepackage{amsthm}
\usepackage{bbm}
\usepackage[normalem]{ulem}

\usepackage[ruled,linesnumbered]{algorithm2e}
\LinesNumbered

\newcommand{\nosemic}{\renewcommand{\@endalgocfline}{\relax}}
\newcommand{\dosemic}{\renewcommand{\@endalgocfline}{\algocf@endline}}
\let\oldnl\nl
\newcommand{\nonl}{\renewcommand{\nl}{\let\nl\oldnl}}

\newcommand{\bpi}{\boldsymbol\pi}

\usepackage{comment}

\newtheorem{theorem}{Theorem}
\newtheorem{lemma}[theorem]{Lemma} 
\newtheorem{proposition}[theorem]{Proposition}
\newtheorem{assumption}[theorem]{Assumption} 
\newtheorem{remark}[theorem]{Remark}
\newtheorem{corollary}[theorem]{Corollary}
\newtheorem{definition}[theorem]{Definition}

\usepackage{interval}
\usepackage{xcolor}
\usepackage{breqn}
\usepackage{graphicx}
\usepackage[font=small,labelfont=bf]{caption}
\usepackage{float}

\usepackage{multicol}
\usepackage{multirow}
\usepackage{blindtext}

\usepackage{scalerel}

\DeclareMathOperator*{\w}{\text{\textbf{w}}}

\DeclareMathOperator*{\argmax}{arg\,max}
\DeclareMathOperator*{\argmin}{arg\,min}

\begin{document}
	
	\title{A Subgame Perfect Equilibrium Reinforcement Learning Approach to Time-inconsistent Problems}
	
	\author{\name Nixie S. Lesmana \email nixiesap001@e.ntu.edu.sg \\
		\name Chi Seng Pun \email cspun@ntu.edu.sg \\
		\addr School of Physical and Mathematical Sciences\\
		Nanyang Technological University\\
		Singapore}
	
	\editor{TBA}
	
	\maketitle
	
	\begin{abstract}
		In this paper, we establish a subgame perfect equilibrium reinforcement learning (SPERL) framework for time-inconsistent (TIC) problems. In the context of RL, TIC problems are known to face two main challenges: the non-existence of natural recursive relationships between value functions at different time points and the violation of Bellman's principle of optimality that raises questions on the applicability of standard policy iteration algorithms for unprovable policy improvement theorems. We adapt an extended dynamic programming theory and propose a new class of algorithms, called backward policy iteration (BPI), that solves SPERL and addresses both challenges. To demonstrate the practical usage of BPI as a training framework, we adapt standard RL simulation methods and derive two BPI-based training algorithms. We examine our derived training frameworks on a mean-variance portfolio selection problem and evaluate some performance metrics including convergence and model identifiability.
	\end{abstract}
	
	\vspace{1mm}
	

	\section{Introduction} \label{sec: intro}
	Time-inconsistent (TIC\footnote{In this paper, the abbreviation TIC could refer to time-inconsistent as an adjective or time inconsistency as a noun, whichever is appropriate.}) performance criterion arises as a result of decision-theoretic planning in order to reflect more closely human's preferences that are prone to bounded rationality (see \cite{Simon1955,Simon2008}), biases, and fallacies. In these preference models, decision alternatives are evaluated using various psychological principles that include but are not limited to present bias, loss aversion, nonlinear probability weighting, and projection bias; see the (cumulative) prospect theory in behavioral economics developed in \cite{Kahneman1979,Tversky1992}. As a result of these biases, TIC can then be described as a situation in which a plan, consisting of current and future actions, could be optimal today but might not be optimal in the future. In the context of dynamic game theory, TIC, also known as dynamic inconsistency, manifests itself through the violation of Bellman's principle of optimality (BPO) by the dominant player (i.e. the future); see \cite{Simaan1973}. Examples include endogenous habit formation in economics (\cite{Fuhrer2000}) and mean-variance criteria in finance (\cite{Markowitz1952}).

Recently, TIC criterion has grown in prevalence alongside artificial intelligence (AI) proliferation, as more and more AI agents are centered around human, e.g., assistive, autonomous, and humanoid robots.
However, modelling realistic human's preferences is not the only cause of TIC. In fact, any efforts to modify a time-consistent (TC) optimization/control criterion may result in TIC. One of the largest contributors to such AI advances is arguably the field of reinforcement learning (RL). An RL agent aims to solve decision-making/control problems, but with methods distinguishable from analytical control solvers, resulting in a broader scope of capabilities. In particular, RL training methods are often drawn from animal learning psychology, where it is a common conception that on top of goal-directed behaviour, it is also important to encode seemingly unrelated behaviour to help achieve the specified goal. Such an encoding can range from simple reward engineering to the field of (optimal) reward design with some involved modifications of reward functions as well as the control criterion itself. These modifications, while maybe inspired by human's cognition, are often practically motivated to overcome a designer's limited ability to observe all intricacies related to the environment and capture them into a single goal functional at initialization. For instance, modifications are necessary when bounded resources cause an RL agent to behave unexpectedly, failing to achieve the intended behavior encoded in the original goal functional, or when designer wants to encode abilities to promote adaptivity and autonomy; e.g., \cite{Schmidhuber1991,Chentanez2004,Bellemare2016,Achiam2017}. An interesting crossover between the two is shown by \cite{Fedus2019}, where behaviorally-motivated (TIC) \textit{hyperbolic-agent} can also serve as a practically-motivated auxiliary task to improve performance against TC \textit{exponential-agent} in several domains. It is noteworthy that an RL framework for the problems under a TIC criterion, e.g., hyperbolic discounting (see \cite{Laibson1997,Frederick2002}), enhances the modelling feasibility, where psychological principles can be introduced to reflect the agent's intrinsic objective beyond the expected utility theory. 

In his seminal work, \cite{Strotz1955} summarizes two types of decision makers, who tackle with TIC seriously: (i) a \textit{pre-committer}, who is TIC-aware but chooses to focus on the planning solely at the initial time point by determining her plan once at the beginning and committing to it throughout the planning horizon, and (ii) a \textit{sophisticated agent}, who is also TIC-aware but unable or unwilling to pre-commit and thus chooses consistent planning as a compromise by finding a plan that is optimal at any given time in consideration of her future disobedience. From an optimization and control perspective, a pre-commiter solves a globally optimal solution to the TIC problem, while a sophisticated agent solves a intrapersonal equilibrium solution.

Contrary to the given terminologies, there are several circumstances in which we prefer consistent plans than globally optimal ones.
First, by definition, the so-called globally optimal plans are only optimal at initial time and state disregarding the fact that such plan might not be optimal when assessed at future time point. The main drawback for pre-committers is their commitment to a plan determined at the beginning while sticking with it throughout. Especially for a stochastic environment, as time evolves, it could deviate from what the agent anticipated at the beginning, resulting in the degeneracy of a pre-commitment plan, and the problem is more pronounced when the planning horizon is long. Empirical evidence reveals that in their natural state (i.e. without relying on commitment devices), humans tend to reassess their plan at some future times and states making them prone to future deviations for a lack of self control; see \cite{Kahneman1979a,Bondt1985}. Therefore, in the event when commitment devices are unavailable or not useful, being (time-)consistent is a rational choice. For instance, let us consider an online learning agent that is acting on behalf of human. In this case, we want our agent to keep its plan open for re-evaluation at future times and states to anticipate some changes in the environment and update its belief or information set accordingly. Consistent plans are robust\footnote{Here, robustness refers strictly to the agent's TIC sources (which does not concern the risk and uncertainty coming from the external environment). TIC sources are an agent's internal attributes and thus, assumed to be known at initialization when agent is TIC-aware. Sophisticated agent exploits this knowledge to devise its "robust" consistent plan, while precommiter does not.} under such future re-evaluations, while globally optimal plans are not. Second, there is lack of a pivotal tool to identify pre-commitment plans. In the context of stochastic controls, BPO violation renders standard DP techniques inapplicable. Although an embedding technique can be employed for some specific criteria such as mean-variance, initiated by \cite{Li2000} and later adapted to RL by \cite{Wang2020}, such technique is difficult to extend to general problem specifications, e.g., state-dependent problems. In fact, the embedding technique is mainly useful in handling the TIC source from the nonlinearity with respect to the reward-to-go, while a more general approach is to formulate and attack the problem with McKean--Vlasov DP; see \cite{Pham2017,Lei2020}. Similar approaches to such a technique can also be found in the RL context, such as \cite{Mannor2011,Evans2016}. These approaches also involve a TC reformulation to the original problem of searching globally optimal solution to a specific TIC criterion but in an approximate or less formal sense. The remaining alternative is then to use trajectory enumeration techniques that apply for more general TIC sources. However, such techniques are not efficient and quickly become intractable in cases such as stochastic environments.

The problem of identifying consistent plans in (finite-horizon) TIC control has been explored extensively since the seminal work of \cite{Strotz1955}. The current popularity of consistent plans in TIC control literature can be attributed to its formalism as \textit{intrapersonal equilibria}, in which an agent's selves at different times (while with the same terminal time) are considered to be self-interested players of a game and a consistent plan is characterized by a subgame perfect equilibrium (SPE) of the game, where no selves have the incentive to deviate. Such formalism was initiated through the study of classical (continuous-time) Ramsey model with a non-exponential discount function in a series of papers; see \cite{Ekeland2006,Ekeland2008,Ekeland2010}. Similar works that adopt game-theoretic line of reasoning, but less formally, can be found in \cite{Pollak1968,Phelps1968,Peleg1973,Barro1999,Luttmer2003}, where different application problems in both discrete- and continuous-time setting are considered. In an effort to unify the game-theoretic views, \cite{Bjoerk2014} proposes an extended DP theory in the context of discrete-time TIC stochastic controls. In this work, the authors derive a Bellman equation like system for relatively general TIC criterion to characterize an equilibrium value function and then solve the system recursively backward to obtain the corresponding SPE policy in various application examples. A continuous-time extension to this general theory is proposed in \cite{Bjoerk2017}, where a system of Hamilton--Jacobi--Bellman (HJB) equations is derived; if solvable, solution to such system can be obtained by the use of the partial differential equation (PDE) tools (see \cite{Lei2021}). Standard procedure can then be applied for practical implementation of the SPE policy derived as above that is, by estimating the parameters in the modelled state dynamics (i.e. stochastic differential equations, transition probabilities) and substituting these estimates into the analytically derived SPE policy form. Despite the close connection between DP and RL, the exploration of SPE policy and extended DP remains scarce in RL literature, except for a few related works on specific tasks, whose subtle difference from ours is illustrated technically in Section \ref{sec:relatedworks} below.

In this paper, we formalize the search of SPE policy as a reinforcement learning (SPERL) problem under general (task-invariant) TIC objectives and limit the scope of our study to finite-horizon, discrete-time problems. In an RL context, TIC problems are known to exhibit two challenges: (i) the non-existence of natural (action-)value recursion given fixed policy, and (ii) the questionable applicability of standard policy iteration algorithms due to unprovable policy improvement theorems (PIT). We propose a new class of algorithms called backward policy iteration (BPI) that addresses both challenges. First, by extending \cite{Bjoerk2014}'s value recursion to action-value recursion, we obtain TIC-adjusted temporal-difference (TD)-based policy evaluation method that recursively links Q-functions at different time points. Second, by applying SPE notion of optimality to structure our policy search, we show that BPI is lexicographically monotonic; this result is parallel to PIT in standard RL problems. Building on this monotonicity result, we can obtain the theoretical guarantees for BPI, such as finite termination/convergence and characterization of converged policy as SPE policy.

Our primary contribution is to propose BPI as a SPERL solver and study its analytical properties as mentioned above. To address some intractability issues of BPI, we further investigate the adaptation of standard RL simulation methods to BPI. We explore both tabular and function approximation cases and derive training algorithms, similar to Q-learning and Deterministic Actor-Critic, which in course reveals both contrasts and similarities between BPI-based and standard PI-based algorithms. We consider a mean-variance analysis application to exemplify how the derived algorithms can be used in practice and evaluate some performance metrics, including convergence and model identifiability.

The remainder of this paper is organized as follows. Section \ref{sec: prelim} formulates a general TIC problem for our RL study. The concept of SPE is elaborated in Section \ref{sec:SPE}, while the related works of both SPE-alike and non-SPE solutions
are reviewed in Section \ref{sec:relatedworks}. Section \ref{sec: SPERL-foundation} lays out the theoretical foundations to our proposed RL framework for training SPE policies structured along the line of policy iteration, namely policy evaluation and policy improvement, where the new BPI algorithm for both discrete and continuous state-action spaces is proposed. Section \ref{sec: training-algo} incorporates the standard RL simulation methods into BPI and specifies how to adapt BPI's key rules. Section \ref{sec: illustrative-example} elaborates the training algorithms under the proposed framework with a well-known financial example of dynamic portfolio selection. Section \ref{sec: conclusion-future} concludes.
	
	\section{Preliminaries} \label{sec: prelim}
	In this section, we will define the subgame perfect equilibrium reinforcement learning (SPERL) problem that we are addressing and provide some backgrounds on its two central concepts: the finite-horizon TIC-RL problems and the SPE notion of optimality.

\subsection{Finite-Horizon TIC Control Problems} \label{sec: finite-horizon-TIC}
Let $\mathcal{T} \doteq \{0, 1, \dotso, T-1\}$ be a discrete time set of decision epochs with finite time horizon $T \in \mathbb{Z}^{+}$. A finite-horizon TIC-RL problem is then defined as policy search in finite-horizon TIC MDP which consists of the standard finite-horizon tuple $(\mathcal{T}, \{\mathcal{X}_t\}_{t \in \mathcal{T}}, \{\mathcal{U}_t\}_{t \in \mathcal{T}}, P)$ and a TIC performance criterion, each of which will be detailed in the following. 

For each $t \in \mathcal{T}$, we assume general state-spaces $\mathcal{X}_t$ and action-spaces $\mathcal{U}_t$ and define transition probability measures $p^u_{t,x}(\cdot) \doteq P(X_{t+1} = \cdot|X_t = x, U_t = u)$ on $\mathcal{X}_{t+1}$. We note on the stationary transition assumption imposed here. Let us denote by $\Pi^{MD}$ the set of all deterministic Markovian policies $\bpi \doteq \{\pi_t: t \in \mathcal{T}\}$, where $\pi_t: \mathcal{X}_t \rightarrow \mathcal{U}_t, \forall t$. Here onward, we will restrict our attention to the policies in the set $\Pi^{MD}$. To ease the notational burden, we introduce some sequential notations to denote the subprocesses and subsets corresponding to a subsequence of $\mathcal{T}$.
\begin{definition}[Policy Sequences]
	We define some sequential notations for policies $\bpi \in \Pi^{MD}$ and their truncations,
	\begin{enumerate}
		\item $\forall k,n \in \mathcal{T}, k \leq n$, ${}{}_k^n\bpi \doteq \{\pi_t: k \leq t \leq n\}$. 
		\item When the last time index $n=T-1$, we shorten our notation by ${}{}_k\bpi \doteq {}{}_k^{T-1}\bpi$.
		\item When the start time index $k=0$, we shorten our notation by ${}{}^n\bpi \doteq {}{}_0^n\bpi$ and in particular $\bpi \doteq {}{}_0^{T-1}\bpi$.
	\end{enumerate}
\end{definition}
Note that the right superscript of policy sequences is reserved to distinguish different policies and the right subscript is to indicate the action at the corresponding time. Similarly for the set notations, we denote by ${}{}_t\mathcal{T}\doteq\{t,t+1,\ldots,T-1\}$ and by ${}{}_t\Pi^{MD}$ the set of all deterministic Markovian policies ${}{}_t\bpi$.

Given a decision epoch $t$ and an observed current system state $x_t \in \mathcal{X}_t$, ${}{}_t\bpi$ prescribes an action selection $\pi_t(x_t) \in \mathcal{U}_t$ which then drives the transition of our MDP to the next system state $x_{t+1} \in \mathcal{X}_{t+1}$ according to $p^{\bpi}_{t,x}(x_{t+1}) \doteq P(X_{t+1} = x_{t+1}|X_t = x_t, U_t = \pi_t(x_t))$. 

In standard finite-horizon RL problems, a TC performance criterion takes the form
\begin{equation} \label{std-objective}
	V^{\bpi}_{\text{std}, \tau}(y) \doteq \mathbb{E}_{\tau, y}\left[ \sum_{t = \tau}^{T-1} \mathcal{R}_t(X_t^{\bpi}, \pi_t(X^{\bpi}_t)) + \mathcal{F}(X_T^{\bpi}) \right] \doteq \mathbb{E}_{\tau, y}\left[ \sum_{t = \tau}^{T} R^{\bpi}_t \right],
\end{equation}
where $(\tau, y = X^{\bpi}_t)$ represents the current/evaluation time and state and the superscript $\bpi$ indicates the policy being followed. We note that the latter presentation is more commonly found in RL literature to account for random intermittent rewards $R_t \doteq \mathcal{R}_t(X_t, U_t)$ for $t\in\mathcal{T}$ and random terminal reward $R_T \doteq \mathcal{F}(X_T)$ emitted by the environment upon hitting the state-action pair $(X_t, U_t)$ at time $t < T$ and state $X_T$ at time $t = T$.

A SPERL problem instead considers the following TIC performance criterion of the form
\begin{equation} \label{TIC-objective}
	V^{\bpi}_{\tau}(y) \doteq \mathbb{E}_{\tau, y} \left[\sum_{t = \tau}^{T-1} \mathcal{R}_{\tau, t}\left(y, X_t^{\bpi}, \bpi_t(X_t^{\bpi})\right) + \mathcal{F}_{\tau}(y, X_T^{\bpi})\right] + \mathcal{G}_{\tau}\left(y, \mathbb{E}_{\tau, y}[X_T^{\bpi}]\right).
\end{equation}
As compared to the TC criterion in (\ref{std-objective}), we can observe some notable differences which make up the TIC sources\footnote{We refer readers to \cite[Sections 1.2, 7-9]{Bjoerk2014} for domain-specific examples of each source which include non-exponential discounting for type (i) and variance-related term for type (ii).} of the criterion \eqref{TIC-objective}: (i) the dependency of reward functions $\mathcal{R}$ and $\mathcal{F}$ on the current time and state, $(\tau, y)$, and (ii) the appearance of term $\mathcal{G}_{\tau}(y,z)$ that is non-linear in the $z$-variable.

\begin{remark}[Random rewards] \label{remark: random-rewards}
	It is possible to incorporate random rewards into the TIC performance criterion above by imposing additional assumption on the reward functions $\mathcal{R}$ and $\mathcal{F}$. For instance, we can define $\mathcal{R}_{\tau, t}(y, X_t, U_t) \doteq \mathcal{H}(\tau, y, \mathcal{R}_t(X_t, U_t)) \doteq \mathcal{H}(\tau, y, R_t)$ and $\mathcal{F}_{\tau}(y, X_T) \doteq \mathcal{H}(\tau, y, \mathcal{F}(X_T)) \doteq \mathcal{H}(\tau, y, R_T)$ where $\mathcal{H}$ can be considered as some TIC transformation of raw rewards and is required to be deterministic. We may then fit some popular TIC rewards into this form: $\mathcal{H}(\tau, y, R) \doteq \frac{R}{1 + h(T- \tau)}$ in hyperbolically discounted problems or $\mathcal{H}(\tau, y, R) \doteq \frac{\gamma}{y}R$ in state-dependent problems for some constants $h$ and $\gamma$. However, random rewards are rarely considered in TIC control literature due to their focus on analytical solutions. Thus, for most derivations in this paper, we will stick to the form in (\ref{TIC-objective}) and revisit the random reward case as a short remark in Section \ref{sec: training-algo}. 
\end{remark}

\subsubsection{BPO Violation and TIC}
We now describe the issue of BPO violation under TIC that lead to the splitting of globally optimal and locally optimal policies, which eventually motivates the SPE notion of optimality. Let us introduce some notations for a generic criterion that could be \eqref{std-objective} or \eqref{TIC-objective}. Under the standard notion of optimality, a policy search given a criterion aims to find a (globally) optimal policy at the beginning time 0: $\bpi^{*0} \doteq \argmax_{\bpi\in\Pi^{MD}} V^{\bpi}_0(x_0)$. Let us also consider local problems $\mathcal{P}_{\tau, y}$ indexed by the initial time $\tau\in\mathcal{T}$ and state $y$. Similarly, we may obtain under the standard notion of (local) optimality, $\bpi^{*\tau} \doteq \argmax_{\bpi\in{}{}_{\tau}\Pi^{MD}}V^{\bpi}_{\tau}(y)$. 

Under standard TC criterion as in (\ref{std-objective}), the optimal solutions for the local problem $\mathcal{P}_{\tau, y}$ and the global problem $\mathcal{P}_{0, x0}$ are linked by BPO which states
\begin{equation} \label{eq:BPO}
	{}{}_{s}\bpi^{*0} = {}{}_{s}\bpi^{*\tau}, \quad \forall \tau\in\mathcal{T},~s \in {}{}_\tau\mathcal{T}.
\end{equation}
In fact, \eqref{eq:BPO} also implies that ${}{}_{s}\bpi^{*t} = {}{}_{s}\bpi^{*\tau}$ for any $t\le \tau$ and $s\in{}{}_\tau\mathcal{T}$. In other words, globally optimal and locally optimal solutions are identical and in this case, the so-called pre-commitment policy is also SPE from a game-theoretic perspective. 

However, under the TIC criterion \eqref{TIC-objective}, the BPO relation \eqref{eq:BPO} no longer holds, causing the split between the locally optimal policy $\bpi^{*\tau}$ and the globally optimal policy $\bpi^{*0}$. The globally optimal policy $\bpi^{*0}$ is called the pre-commitment policy and it is usually found by means other than DP, which is no longer applicable without BPO. Moreover, $\bpi^{*0}$ is not SPE. The BPO violation can be viewed as the consequence of conflicting objectives in the collection of local problems $\{\mathcal{P}_{t,x}: t \in \mathcal{T}, x \in \mathcal{X}_t\}$, motivating the game-theoretic reformulation of TIC problems as an intrapersonal game, which will be detailed in the next subsection. The intrapersonal equilibrium (i.e. SPE) solution, if exists, recovers the BPO relation \eqref{eq:BPO} by modifying how we define $\bpi^*$ for all local problems.

\begin{remark}
	In the literature on TIC problems, there is the third class of `optimality' other than the pre-commitment and the SPE notions, namely dynamically optimality; see \cite{Pedersen2016}. A dynamically optimal policy is constructed by the collection of locally optimal solutions at all time points, i.e. $\{\bpi^{*t}_t:t\in\mathcal{T}\}$, where $\bpi^{*t}$ is the optimal pre-commitment policy to the local problem $\mathcal{P}_{t,x}$. However, this formation does not exploit the linkage among the local problems and such a construction needs to be justified. A closely related concept regarding the latter is called time consistency in efficiency; see \cite{Cui2017,Pun2021}. However, since dynamically optimal policies are generally lack of interpretation power and theoretical guarantees, we focus on the SPE policies.
\end{remark}

\subsection{SPE Notion of Optimality} \label{sec:SPE}
Given the finite-horizon discrete-time TIC criterion in (\ref{TIC-objective}), SPERL's objective is to find a SPE policy, which will be defined shortly after a few definitions and notational assumptions.
\begin{definition}[Delaying Operator]
	We denote by $u\oplus_t\bpi$ the concatenation between the use of action $u \in \mathcal{U}_t$ at any given time $t\in\mathcal{T}$ and the delayed use of policy ${}{}_{t+1}\bpi$. Hereafter, we adopt a convention that $u\oplus_t{}{}^m\bpi$ for $m\le t$ is simply an action $u\in \mathcal{U}_t$ at time $t\in\mathcal{T}$.
\end{definition}
\begin{definition}[Action-Value Functions] \label{def: Policy-dependent VFs}
	Given any fixed policy $\bpi \in \Pi^{MD}$ and its corresponding value function $V^{\bpi}_{t}(x)$ defined in \eqref{TIC-objective} for $t\in\mathcal{T}$, we define (policy-dependent) action-value function
	\begin{equation} \label{eq1: Policy-dependent VFs}
		Q^{\bpi}_{t}(x, u) \doteq V_{t}^{u\oplus_t\bpi}(x).
	\end{equation}
\end{definition}
\begin{remark}
	In Definition \ref{def: Policy-dependent VFs}, the policy ${}{}^{t}\bpi$ does not play a role and the policy ${}{}_{t+1}\bpi$ is fixed, while $u\in\mathcal{U}_t$ is the action variate at time $t$ with the state $x$.
\end{remark}
\begin{definition}[SPE Policy] \label{def: SPE-Policy}
	A policy $\bpi^* \in \Pi^{MD}$ is an SPE policy if
	\begin{equation} \label{eq1: SPE-policy}
		Q^{\bpi^*}_{t}(x, \bpi^*_{t}(x)) \geq Q^{\bpi^*}_{t}(x, u), \quad \forall t \in \mathcal{T},~x \in \mathcal{X}_t,~u \in \mathcal{U}_t.
	\end{equation}
\end{definition}
SPERL's search objective is inspired by the \textit{intrapersonal equilibria} characterization of time-consistent plans that aim to solve the conflicts between the objectives in the local problem set $\{\mathcal{P}_{t,x}: t \in \mathcal{T}, x \in \mathcal{X}_t\}$ by reformulating them as an SPE search in a sequential subgames played by \textit{self-interested} $\mathcal{T}$-indexed players, which goes as follows:
\begin{quote}
	At each round $t$, only player $t$ is allowed to move by choosing a strategy in the form of a mapping $\pi_t: \mathcal{X}_t \rightarrow \mathcal{U}_t$. Player $t$'s objective is to maximize his/her expected payoff $Q^{\bpi}_{t}(x,\pi_t) \doteq V^{\pi_t \oplus_t\bpi}_{t}(x)$. Player $t$ can observe $(t, X^{\bpi}_{t} = x)$ and has perfect information on the \textit{future players}' strategies ${}{}_{t+1}\bpi$.
\end{quote}
The SPE of the game above can be found by applying backward induction, where $\pi^*_t$ is obtained at each inductive step, resulting in SPE strategies of each player $\{\pi^*_{T-1}, \pi^*_{T-2}, \dotso, \pi^*_0\}$. We can easily verify that this strategy set realizes the condition in (\ref{eq1: SPE-policy}).
\begin{remark}[Markovian assumption on SPE policies]
	By definition, an SPE policy is Markovian, which implies that for each $t \in \mathcal{T}$, the past (including past players' policies ${}{}^{t-1}\bpi$) does not influence how player $t$ acts. \label{remark: Markov-SPE-policy}
\end{remark}
	\subsection{Related Works on TIC Problems in RL} \label{sec:relatedworks}
After reviewing some notions of optimality, it is convenient at this stage to discuss the related works with some technical comments before we end this section. Table \ref{tab:RLtable} compares between the TC and TIC problems in RL and the subclasses of optimality under the TIC problems, which were discussed in Section \ref{sec: intro}. We remark here that the SPE concept is investigated in \cite{Lattimore2014} for general discounting RL problems, which provides a detailed account on the advantages of SPE policies. However, they do not focus on solving the SPE policy search problem.

\begin{table}[!ht]
	\centering
	\caption{Different classes of RL problems and how they are attempted.}
	\label{tab:RLtable}
	\resizebox{\linewidth}{!}{
		\begin{tabular}{p{0.16\linewidth}||p{0.23\linewidth}||p{0.25\linewidth}|p{0.28\linewidth}}
			\hline \hline
			Criterion & \multicolumn{1}{|c||}{\textit{TC} (w/ BPO)} & \multicolumn{2}{c}{\textbf{\textit{TIC}} (w/o BPO)} \\
			\hline \hline
			Optimality & BPO promises dynamic optimality  & Globally optimal plan \newline (Precommitment) & \textbf{Consistent plan} \newline (SPE revives BPO) \\
			\hline
			Update Rule & Policy Improvement (PolImp) & Embed the problem to the one w/ PolImp & \textbf{SPE-improving rule} \newline (Definition \ref{def: SPE-improving rule} below) \\
			\hline
			Convergence guarantee & Monotonicity \newline (w/ PIT) & Monotonicity for the auxiliary one & \textbf{Lex-monotonicity} \newline (Theorem \ref{thm: lex-monotone} below) \\
			\hline
			Task-invariant & Yes & No & \textbf{Yes} \\
			\hline \hline
			{\footnotesize Reference} & {\footnotesize \cite{Sutton2018}} & {\footnotesize \cite{Wang2020}} & {\footnotesize \textbf{This paper}} \\
			\hline \hline
		\end{tabular}
	}
\end{table}

It should also be noted that there are a number of works attempting a TIC-RL problem from the perspective of learning behaviour or efficiency instead of optimality. Hence, each of their algorithms may learn a kind of `optimality' based on a pre-specified behaviour of the agent and it is not clear whether the converged policy falls into any class of optimality in Table \ref{tab:RLtable}. Hence, in the following two paragraphs, we categorize some related works into two streams based on their search of SPE-alike or non-SPE policies.

\paragraph{SPE-alike Policy Search in RL}
We highlight some differences between ours and some related prior works that have mentioned sophisticated, locally optimal, or SPE solutions. \cite{Evans2016} consider the learning of sophisticated behaviour under hyperbolic discounted criterion. The construction of their update rules is based on TC reformulation to learn pre-commitment plans to which sophistication is heuristically encoded afterwards. Though the sophisticated behavior is an exact analogy of the SPE policy, this work does not clarify whether its heuristic encoding sufficiently characterizes sophistication (actually terminates at/converges to SPE policy). Moreover, due to its TC reformulation, their method does not apply to our more general TIC criterion (which includes state-dependency). In \cite{Tamar2013,Tamar2016}, actor-critic and temporal-difference (TD)-based algorithms are proposed to learn a \textit{locally optimal} policy under variance-related criteria, respectively. These works are by far the closest to our approach for some similarities in their derivation of value recursions to the extended DP theory, but the relation of the learnt policies to SPE policy remains obscure for two reasons: (i) SPE policy is defined under deterministic notion of optimality, while both algorithms use stochastic policy representation; (ii) both works adopt gradient-based methods shifting the optimization landscape to that of policy parameters.

\paragraph{Non-SPE Policy Search in RL}
We briefly review some solution alternatives that do not fall into either class of aforementioned solutions. Recently, TIC is often handled from a tool-/task-oriented angle, that is by overcoming the difficulties of applying specific tools (derived under TC criteria) to a specific TIC task through clever maneuvers of task-specific properties. Though such approaches may be effective in handling some specific TIC criteria, their applicability to other TIC tasks remains unknown. Moreover, the tool-oriented handling of difficulties may cause the lost of ``optimality" of the obtained solutions. 
For instance, in the context of RL, Q-learning (\cite{Watkins1992}) is a globally-optimal policy search tool given TC criteria under deterministic notion of optimality. A popular RL algorithm designed for hyperbolic-discounted criterion is \textit{$\mu$Agent} (see \cite{KurthNelson2010} and \cite{Fedus2019}), that learns through a shared representation of Q-learning like agents. However, the policy learnt by such \textit{modified} Q-learning as part of \textit{$\mu$Agent} has an unknown theoretical ``optimality"\footnote{For empirical studies relevant to uncovering the "optimality" property of hyperbolic-discounted RL algorithms, see \cite{KurthNelson2010}.}. For the above reasons, we will pivot on the control/optimization perspective to preserve the main characteristics of TIC, without relying on specifications of tasks or tools.
	
	\section{SPERL Framework} \label{sec: SPERL-foundation}
	In this section, we will lay out some foundations to our proposed framework for training SPE policies structured along the line of policy iteration. We divide our discussion into two parts: policy evaluation and policy improvement. In regards of policy evaluation, we will derive a recursive system satisfied by the TIC Q-function given fixed policy $\bpi$, as defined in (\ref{eq1: Policy-dependent VFs}). At this stage, we have not applied any game-theoretic concepts and instead focus on the validity of the recursive evaluation scheme. We will elaborate on how to use these $\bpi$-dependent Q-functions to structure our search for an SPE policy in our policy improvement. Borrowing the SPE notion of optimality, we will propose a new class of policy iteration algorithms, called backward policy iteration (BPI), which possess desirable analytical properties such as monotonicity (i.e. policy improvement theorem-alike) and convergence to SPE policy in both discrete and continuous state-action spaces.

\subsection{Policy Evaluation (PolEva)} \label{sec:PE}
In this subsection, we set up a recursive evaluation of expected TIC rewards given a fixed policy $\bpi$. Unlike in the case of TC rewards, there is no natural recursive equations between TIC value functions; be it state-value or action-value functions. We apply the techniques, which were used to derive the extended Bellman equations in \cite{Bjoerk2014}, to obtain a backward inductive policy evaluation (PolEva) scheme.

First, we define a few adjustment functions that will help tracking the non-stationary changes in the $Q$-recursion.
\begin{definition}[Adjustment Functions]
	Given any fixed policy $\bpi \in \Pi^{MD}$, we define the following (policy-dependent) adjustment functions.
	\begin{eqnarray}
		f^{\bpi, \tau, y}_t(x, u) & \doteq & \mathbb{E}_{t, x}\left[\mathcal{F}_{\tau}(y, X_T^{u\oplus_t\bpi})\right], \label{eq1: adjustment-functions}\\
		g^{\bpi}_t(x, u) & \doteq & \mathbb{E}_{t, x}\left[X_T^{u\oplus_t\bpi}\right], \label{eq2: adjustment-functions}\\
		r^{\bpi, \tau, m, y}_t(x, u) & \doteq & \mathbb{E}_{t, x}\left[\mathcal{R}_{\tau, m}\left(y, X_m^{u\oplus_t{}{}^{m-1}\bpi}, \pi_m(X_m^{u\oplus_t{}{}^{m-1}\bpi})\right)\right] \label{eq3: adjustment-functions}
	\end{eqnarray}
	for $t\in\mathcal{T}$, $\tau \in {}{}_t\mathcal{T}$, $m \in {}{}_\tau\mathcal{T}$, $y \in \mathcal{X}_\tau$, $x \in \mathcal{X}_t$, and $u \in \mathcal{U}_t$.
	\label{def: adjustment-functions}
\end{definition}
\begin{remark}
	Again, we note that by Definitions \ref{def: Policy-dependent VFs} and \ref{def: adjustment-functions}, for each player $t$, its action-value function and adjustment functions are independent of \textit{past players'} policies ${}{}^{t-1}\bpi$.
\end{remark}

Now, we are ready to present the main result of this subsection, from which we later set up our TIC PolEva scheme.
\begin{proposition}[Policy-dependent TIC $Q$-recursion]	\label{prop: policy-dependent-Q-recursion}
	Let $\bpi\in \Pi^{MD}$ be any fixed policy and $Q, f, g, r$ defined as in Definitions \ref{def: Policy-dependent VFs} and \ref{def: adjustment-functions}. Then, the following holds for every fixed $t\in\mathcal{T}$, $\tau \in {}{}_t\mathcal{T}$, $m \in {}{}_\tau\mathcal{T}$, $y \in \mathcal{X}_\tau$, $x \in \mathcal{X}_t$, and $u \in \mathcal{U}_t$,
	\begin{enumerate}
		\item the adjustment function $r^{\bpi, \tau, m, y}$ satisfy the equations
		\begin{eqnarray}
			r^{\bpi, \tau, m, y}_t(x, u) &=& \mathbb{E}_{t,x}\left[r^{\bpi, \tau, m, y}_{t+1}(X^u_{t+1}, \pi_{t+1}(X^u_{t+1}))\right], \hbox{ for } m\not=t,~t<T-1, \label{PE-r-true-recursion}\\
			r^{\bpi, t, t, y}_t(x, u) &=& \mathbb{E}_{t,x} \left[ \mathcal{R}_{t, t}(y, x, u) \right], \hbox{ for } t<T-1, \label{PE-r-true-boundary0} \\
			r^{\bpi, T-1, T-1, y}_{T-1}(x, u) &=& \mathbb{E}_{T-1,x} \left[ \mathcal{R}_{T-1, T-1}(y, X^u_T, u) \right]; \label{PE-r-true-boundary}
		\end{eqnarray}
		
		\item the adjustment function $f^{\bpi, \tau, y}$ satisfy the equations
		\begin{eqnarray}
			f^{\bpi, \tau,y}_t(x, u) &=& \mathbb{E}_{t,x}\left[f^{\bpi, \tau,y}_{t+1}(X^u_{t+1}, \pi_{t+1}(X^u_{t+1}))\right],\hbox{ for } t < T-1, \label{PE-f-true-recursion}\\
			f^{\bpi, \tau, y}_{T-1}(x, u) &=& \mathbb{E}_{T-1,x}\left[\mathcal{F}_{\tau}(y, X^u_T)\right]; \label{PE-f-true-boundary}
		\end{eqnarray}
		
		\item the adjustment function $g^{\bpi}$ satisfy the equations
		\begin{eqnarray}
			g^{\bpi}_t(x, u) &=& \mathbb{E}_{t,x}\left[g^{\bpi}_{t+1}(X^u_{t+1}, \pi_{t+1}(X^u_{t+1}))\right],\hbox{ for } t < T-1, \label{PE-g-true-recursion}\\
			g^{\bpi}_{T-1}(x, u) &=& \mathbb{E}_{t,x}\left[ X_T^u \right]; \label{PE-g-true-boundary}
		\end{eqnarray}
		
		\item the action-value function $Q^{\bpi}$ satisfies the equations \small
		\begin{eqnarray}
			Q^{\bpi}_t(x, u) &=& r^{\bpi, t, t, x}_t(x, u) + \mathbb{E}_{t,x}\left[Q^{\bpi}_{t+1}(X^u_{t+1}, \pi_{t+1}(X^u_{t+1}))\right] \label{PE-Q-true-recursion} \\
			&&- \left\{ \sum_{m=t+1}^{T-1} \left( \mathbb{E}_{t,x} \left[ r_{t+1}^{\bpi, t+1, m, X^u_{t+1}}(X^u_{t+1}, \pi_{t+1}(X^u_{t+1}))\right] - r_t^{\bpi, t,m,x}(x, u) \right) \right\} \nonumber \\
			&&- \left\{\mathbb{E}_{t,x}\left[f^{\bpi, t+1, X^u_{t+1}}_{t+1}(X^u_{t+1}, \pi_{t+1}(X^u_{t+1}))\right] -
			f^{\bpi, t,x}_t(x, u) \right\} \nonumber \\
			&&- \left\{ \mathbb{E}_{t,x}\left[\mathcal{G}_{t+1}(X^u_{t+1}, g^{\bpi}_{t+1}(X^u_{t+1}, \pi_{t+1}(X^u_{t+1})))\right] - \mathcal{G}_t\left(x, 
			g^{\bpi}_t(x, u)\right)\right\},\hbox{ for } t < T-1, \nonumber \\
			{\textstyle Q^{\bpi}_{T-1}(x, u)} &=& {\textstyle r^{\bpi, T-1, T-1, x}_{T-1}(x, u) + f^{\bpi, T-1, x}_{T-1}(x, u) + \mathcal{G}_{T-1}(x, g^{\bpi}_{T-1}(x, u))}. \label{PE-Q-true-boundary}
		\end{eqnarray}
	\end{enumerate}
\end{proposition}
\begin{proof}
	This proof is similarly developed as in \cite{Bjoerk2014}. The boundary values \eqref{PE-r-true-boundary0}, \eqref{PE-r-true-boundary}, \eqref{PE-f-true-boundary}, \eqref{PE-g-true-boundary}, and \eqref{PE-Q-true-boundary} can be easily computed from Definitions \ref{def: Policy-dependent VFs} and \ref{def: adjustment-functions}. 
	\begin{enumerate}
		\item For the $r$ recursion and $m \not= t$, we have \small
		\begin{align}
			r^{\bpi, \tau, m, y}_t(x, u) &= \mathbb{E}_{t,x}\left[ \mathcal{R}_{\tau, m}(y, X^{u\oplus_t{}{}^{m-1}\bpi}_m, \pi_m(X^{u\oplus_t{}{}^{m-1}\bpi}_m)) \right] \tag*{(by (\ref{eq3: adjustment-functions}))}\\
			&= \mathbb{E}_{t,x}\left[ \mathbb{E}_{t+1, X^u_{t+1}} \left[ \mathcal{R}_{\tau, m}(y, X^{{}{}^{m-1}_{t+1}\bpi}_m, \pi_m(X^{{}{}^{m-1}_{t+1}\bpi}_m)) \right] \right] \tag*{(by the tower rule)} \\
			&= \mathbb{E}_{t,x}\left[ r^{\bpi, \tau, m, y}_{t+1}(X^u_{t+1}, \pi_{t+1}(X^u_{t+1})) \right]. \tag*{(by (\ref{eq3: adjustment-functions}))}
		\end{align} \normalsize
		
		\item For the $f$ recursion and $t < T-1$, by (\ref{eq1: adjustment-functions}) and the tower rule, we similarly have \small
		\begin{align*}
			f^{\bpi, \tau,y}_t(x, u) &=
			\mathbb{E}_{t, x}\left[ \mathcal{F}_{\tau}(y, X_T^{u\oplus_t\bpi})\right]= \mathbb{E}_{t,x}\left[\mathbb{E}_{t+1, X^u_{t+1}}\left[ \mathcal{F}_{\tau}\left(y, X_T^{{}{}_{t+1}\bpi}\right)\right]\right] \\
			&= \mathbb{E}_{t,x}\left[ f^{\bpi, \tau, y}_{t+1}(X^u_{t+1}, \pi_{t+1}(X^u_{t+1})) \right].
		\end{align*} \normalsize
		
		\item For the $g$ recursion and $t < T-1$, by (\ref{eq2: adjustment-functions}) and the tower rule, we similarly have
		$$
		g^{\bpi}_t(x, u)=\mathbb{E}_{t, x}\left[X_T^{u\oplus_t\bpi}\right]=\mathbb{E}_{t,x}\left[\mathbb{E}_{t+1, X^u_{t+1}}\left[X_T^{{}{}_{t+1}\bpi}\right]\right]=\mathbb{E}_{t, x}\left[ g^{\bpi}_{t+1}(X^u_{t+1}, \pi_{t+1}(X^u_{t+1})) \right].
		$$
		
		\item For the $Q$ recursion and $t < T-1$, by (\ref{eq1: Policy-dependent VFs}) and the tower rule, we have \footnotesize
		\begin{align*}
			Q^{\bpi}_t(x, u) = & \; \mathbb{E}_{t,x}[\mathcal{R}_{t,t}(x, x, u)] + \mathbb{E}_{t,x}\left[Q^{\bpi}_{t+1}(X^u_{t+1}, \pi_{t+1}(X^u_{t+1}))\right] - \mathbb{E}_{t,x}\left[Q^{\bpi}_{t+1}(X^u_{t+1}, \pi_{t+1}(X^u_{t+1}))\right] \nonumber \\
			& + \sum_{m = t+1}^{T-1} r^{\bpi,t, m, x}_t(x, u) + f^{\bpi,t, x}_t(x, u) + \mathcal{G}_t\left(x, g^{\bpi}_t(x, u)\right) \\
			= & \; \mathbb{E}_{t,x}[\mathcal{R}_{t,t}(x, x, u)] + \mathbb{E}_{t,x}\left[Q^{\bpi}_{t+1}(X^u_{t+1}, \pi_{t+1}(X^u_{t+1}))\right] \nonumber \\
			& - \sum_{m = t+1}^{T-1}\mathbb{E}_{t,x}\left[ \mathbb{E}_{t+1, X^u_{t+1}}\left[ \mathcal{R}_{t+1, m}(X^u_{t+1}, X^{{}{}_{t+1}^{m-1}\bpi}_m, \pi_m(X^{{}{}_{t+1}^{m-1}\bpi}_m)) \right]\right] + \sum_{m = t+1}^{T-1} r^{\bpi,t, m, x}_t(x, u) \nonumber \\
			& - \mathbb{E}_{t,x}\left[ \mathbb{E}_{t+1, X^u_{t+1}} \left[ \mathcal{F}_{t+1}(X^u_{t+1}, X^{{}{}_{t+1}\bpi}_T) \right] \right] + f^{\bpi,t, x}_t(x, u) \nonumber \\
			& - \mathbb{E}_{t,x}\left[ \mathcal{G}_{t+1}(X^u_{t+1}, g^{\bpi}_{t+1}(X^u_{t+1}, \pi_{t+1}(X^u_{t+1}))) \right] + \mathcal{G}_t\left(x, g^{\bpi}_t(x, u)\right), 
		\end{align*} \normalsize
		which gives the right-hand side (RHS) of \eqref{PE-Q-true-recursion} by noting \eqref{eq1: adjustment-functions}-\eqref{eq3: adjustment-functions}.
	\end{enumerate}
\end{proof}

The SPERL PolEva scheme follows trivially by making the updates of $f, g, r$ and $Q$ flow backward from $t = T-1$ to $0$; see Algorithm \ref{alg: TD-based PE}.

\begin{algorithm}[H]
	\SetAlgoLined
	\SetKwInOut{Input}{Input}
	\SetKwInOut{Output}{Output}
	\SetKwInOut{Loop}{Loop}
	
	\Input{${}{}_{t+1}\bpi$}
	\Output{$Q^{\bpi}_t(x, u), \forall x, u$}
	\For{$k \gets T-1$ \KwTo $t$}{
		\For{$\tau \gets T-1$ \KwTo $k$}{	\label{algline:start-adj}
			\For{$m \gets T-1$ \KwTo $\tau$}{
				Compute 
				$r^{\bpi, \tau, m, y}_k(x, u), \forall x\in \mathcal{X}_k, u\in\mathcal{U}_k, y\in\mathcal{X}_\tau$ by (\ref{PE-r-true-recursion})-(\ref{PE-r-true-boundary})\;
			}
			Compute $f^{\bpi, \tau, y}_k(x,u), \forall x\in \mathcal{X}_k, u\in\mathcal{U}_k, y\in\mathcal{X}_\tau$ by (\ref{PE-f-true-recursion})-(\ref{PE-f-true-boundary})\;
		}
		Compute $g^{\bpi}_k(x,u), \forall x\in \mathcal{X}_k, u\in\mathcal{U}_k$ by (\ref{PE-g-true-recursion})-(\ref{PE-g-true-boundary})\; \label{algline:end-adj}
		Compute $Q^{\bpi}_k(x,u), \forall x\in \mathcal{X}_k, u\in\mathcal{U}_k$ by (\ref{PE-Q-true-recursion})-(\ref{PE-Q-true-boundary})\;
	}
	\caption{TIC-TD Policy Evaluation (PolEva)}
	\label{alg: TD-based PE}
\end{algorithm}
Next, we will justify the validity of Algorithm \ref{alg: TD-based PE} for computing $Q^{\bpi}_t(x, u)$. Firstly, note that the $t$-indexed \textit{adjustment functions} in the RHS of (\ref{PE-Q-true-recursion}) have been computed in the same iteration $k = t$ by lines \ref{algline:start-adj}-\ref{algline:end-adj}. We can verify the validity of the adjustment functions computation scheme by observing that in the non-boundary cases, the $(t+1)$-indexed functions inside the expectations in the RHS of (\ref{PE-r-true-recursion}), (\ref{PE-f-true-recursion}), and (\ref{PE-g-true-recursion}) have all been computed in the previous iteration $k = t+1$ such that an expectation over $X_{t+1}$ can be computed. For the boundary cases, the functions inside the expectations i.e. $\mathcal{R}, \mathcal{F}$ are known such that expectations can be computed, either over the deterministic variable in (\ref{PE-r-true-boundary}) or the random $X_T$ in (\ref{PE-f-true-boundary}) and (\ref{PE-g-true-boundary}). Secondly, we observe that the $(t+1)$-indexed functions inside the expectations in the RHS of (\ref{PE-Q-true-recursion}) have all been computed in the previous iteration $k = t+1$ such that the expectation over $X_{t+1}$ can then be computed. Finally, by Proposition \ref{prop: policy-dependent-Q-recursion}, we have shown that (\ref{eq1: Policy-dependent VFs}) holds.

In the subsequent sections, we will refer to the integrands in the RHS of Proposition \ref{prop: policy-dependent-Q-recursion} as DP targets defined as follows \small
\begin{align}
	\xi^r_t(x, u, \tau, m, y; \, {}{}_{t+1}\bpi) &\doteq
	\begin{cases}
		\mathcal{R}_{T-1,T-1}(y,X^u_T,u), &\text{if $m=\tau=t=T-1$},\\
		\mathcal{R}_{t, t}(y, x, u), &\text{if $m=\tau=t, \forall t<T-1$},\\
		r^{\bpi, \tau, m, y}_{t+1}(X^u_{t+1}, \pi_{t+1}(X^u_{t+1})), \, &\text{if $m\not=t$, $\forall t<T-1$},
	\end{cases} \label{r-DP-target} \\
	\xi^f_t(x, u, \tau, y; \, {}{}_{t+1}\bpi) &\doteq 
	\begin{cases}
		\mathcal{F}_{\tau}(y, X^u_T) &\text{if } t = T-1,\\
		f^{\bpi, \tau, y}_{t+1}(X^u_{t+1}, \pi_{t+1}(X^u_{t+1})) &\text{otherwise},
	\end{cases} \label{f-DP-target} \\
	\xi^g_t(x, u; \, {}{}_{t+1}\bpi) &\doteq
	\begin{cases}
		X^u_T &\text{if $t = T-1$},\\
		g^{\bpi}_{t+1}(X^u_{t+1}, \pi_{t+1}(X^u_{t+1})) &\text{otherwise},
	\end{cases} \label{g-DP-target}\\
	\xi^Q_t(x, u; \, {}{}_{t+1}\bpi) &\doteq 
	\begin{cases}
		r^{\bpi, t, t, x}_t(x, u) + f^{\bpi, t, x}_t(x, u) + \mathcal{G}_t(x, g^{\bpi}_t(x, u))
		&\text{if $t = T-1$},\\
		\scalebox{0.85}{$r^{\bpi, t, t, x}_t(x, u) + Q^{\bpi}_{t+1}(X^u_{t+1}, \pi_{t+1}(X^u_{t+1})) - (\Delta r^{\bpi}_t + \Delta f^{\bpi}_t + \Delta g^{\bpi}_t),$} &\text{otherwise},
	\end{cases} \label{Q-DP-target}
\end{align} \normalsize
where
\begin{align*}
	\Delta r^{\bpi}_t &\doteq \sum_{m = t+1}^{T-1} \left( r^{\bpi, t+1, m, X^u_{t+1}}(X^u_{t+1}, \pi_{t+1}(X^u_{t+1})) - r^{\bpi, t, m, x}_t(x, u) \right), \\
	\Delta f^{\bpi}_t &\doteq f^{\bpi, t+1, X^u_{t+1}}_{t+1}(X^u_{t+1}, \pi_{t+1}(X^u_{t+1})) - f^{\bpi, t, x}_t(x, u), \\
	\Delta g^{\bpi}_t &\doteq \mathcal{G}_{t+1}(X^u_{t+1}, g^{\bpi}_{t+1}(X^u_{t+1}, \pi_{t+1}(X^u_{t+1}))) - \mathcal{G}_t(x, g^{\bpi}_t(x, u)).
\end{align*}


\subsection{Policy Improvement (PolImp)}
This subsection is mainly concerned about adapting the SPE notion of optimality to construct an SPE-improvement scheme. We deliberately deviate from the standard policy improvement (PolImp) scheme due to unprovable PIT under TIC\footnote{This has been shown by \cite{Sobel1982} through counterexample, particularly showing that $\forall s$,
	\begin{align*}
		V^{\bpi'}(s) \geq V^{\bpi}(s) \not\Rightarrow V^{\delta \cdot \bpi'}(s) \geq V^{\delta \cdot \bpi}(s).  
	\end{align*}
	While his work considers infinite-horizon variance-related criterion, finite-horizon RL problems are often rewritten as infinite-horizon problems with time-extended state-space; see \cite{Harada1997}. Moreover, as his argument relies mainly on BPO violation, it also applies to our case.}. Intuitively, we hypothesize that the cause of such PIT failure lies in the definition of `improvement' itself such that the SPE notion of optimality should be accompanied by a corresponding change in the definition of a `better' policy.

To aid our result presentation in the later part of this section, we first define several SPE-improvement relations on the (tail) truncation of policies.
\begin{definition}[SPE-Improving Rules on Truncated Policy Sequences] \label{def: SPE-Improve-Policy-Seq}
	\small
	\begin{align*}
		{}{}_k\bpi' \succeq_{eq} {}{}_k\bpi &\Leftrightarrow \left(\, \forall x \in \mathcal{X}_{k},~Q^{\bpi'}_{k}(x, \pi'_{k}(x)) \geq Q^{\bpi'}_{k}(x, \pi_{k}(x)) \; \land \; {}{}_{k+1}\bpi' \succeq_{eq} {}{}_{k+1}\bpi \, \right),\\
		{}{}_k\bpi' \sim_{eq} {}{}_k\bpi &\Leftrightarrow \left( \, \forall x \in \mathcal{X}_k,~Q^{\bpi'}_k(x, \pi'_k(x)) = Q^{\bpi'}_k(x, \pi_k(x)) \, \land \, {}{}_{k+1}\bpi' \sim_{eq} {}{}_{k+1}\bpi \, \right),\\
		{}{}_k\bpi' \succ_{eq} {}{}_k\bpi &\Leftrightarrow \left(\, \exists x \in \mathcal{X}_{k},~Q^{\bpi'}_{k}(x, \pi'_{k}(x)) > Q^{\bpi'}_{k}(x, \pi_{k}(x)) \; \land \; {}{}_k\bpi' \succeq_{eq} {}{}_k\bpi \, \right).
	\end{align*}
\end{definition}
Then, we define our SPE-improving rules on the set of policies $\Pi^{MD}$.
\begin{definition}[SPE-Improving Rules] \label{def: SPE-improving rule}
	Let us define the following relations on the set of all policies $\Pi^{MD}$ such that for any arbitrary policies $\bpi', \bpi \in \Pi^{MD}$,
	\begin{align*}
		\bpi' \succeq_{eq} \bpi &\Leftrightarrow \left(\forall k, \forall x \in \mathcal{X}_k,~Q^{\bpi'}_k(x, \pi'_k(x)) \geq Q^{\bpi'}_k(x, \pi_k(x)) \, \land \, {}{}_{k+1}\bpi' \succeq_{eq} {}{}_{k+1}\bpi \right),\\
		\bpi' \sim_{eq} \bpi &\Leftrightarrow \left(\forall k, \forall x \in \mathcal{X}_k,~Q^{\bpi'}_k(x, \pi'_k(x)) = Q^{\bpi'}_k(x, \pi_k(x)) \, \land \, {}{}_{k+1}\bpi' \sim_{eq} {}{}_{k+1}\bpi \right),\\
		\bpi' \succ_{eq} \bpi &\Leftrightarrow \left(\exists k, \exists x \in \mathcal{X}_k \text{ s.t. }Q^{\bpi'}_k(x, \pi'_k(x)) > Q^{\bpi'}_k(x, \pi_k(x)) \, \land \, {}{}_k\bpi' \succeq_{eq} {}{}_k\bpi \right).
	\end{align*}
\end{definition}
These rules are inspired by the Nash Equilibrium (NE) concept of game-theory; specifically, we can interpret the relation $\succeq_{eq}$ as a PolImp rule by the following statement:
\begin{center}
	``Player $t$'s strategy $\bpi'_t$ is said to be \textit{better} if playing the strategy $\pi'_t$ improves $t$'s utility, given other players have the same belief about $\bpi' \succeq_{eq} \bpi$ and thus play $\bpi'_{-t}$."
\end{center}
We note that as compared to the NE concept, our SPE-improving rule restricts the set of player $t$'s opponents from $\{0, 1, \dotso, t-1, t+1, \dotso, T-1\}$ to $\{t+1, \dotso, T-1\}$, which is implied by Remark \ref{remark: Markov-SPE-policy} on the concept of SPE solution.

The two SPE-improving rules defined above are related by the following equivalence result.
\begin{proposition} \label{prop: SPE-Improving Equivalence}
	Consider two arbitrary policies $\bpi, \bpi' \in \Pi^{MD}$. Then, the following holds
	\begin{align}
		\bpi' \succeq_{eq} \bpi ~\Leftrightarrow~ \forall k,~{}{}_k\bpi' \succeq_{eq} {}{}_k\bpi, \label{eq1: SPE-Improving Equivalence} \\
		\bpi' \sim_{eq} \bpi ~\Leftrightarrow~ \forall k,~{}{}_k\bpi' \sim_{eq} {}{}_k\bpi. \label{eq2: SPE-Improving Equivalence}
	\end{align}
\end{proposition}
\begin{proof}
	First, we show that (\ref{eq1: SPE-Improving Equivalence}) holds. By Definition \ref{def: SPE-improving rule}, $\bpi' \succeq_{eq} \bpi$ is equivalent to
	$$
	\forall k, \forall x \in \mathcal{X}_k, \quad Q^{\bpi'}_k(x, \pi'_k(x)) \geq Q^{\bpi'}_k(x, \pi_k(x)) \, \land \, ({}{}_{k+1}\bpi' \succeq_{eq} {}{}_{k+1}\bpi),
	$$
	which by Definition \ref{def: SPE-Improve-Policy-Seq}, is equivalent to $\forall k$, ${}{}_{k+1}\bpi' \succeq_{eq} {}{}_{k+1}\bpi$. We can show (\ref{eq2: SPE-Improving Equivalence}) similarly by replacing the relation $\succeq_{eq}$ with $\sim_{eq}$.
\end{proof}

Now, we introduce backward policy iteration (BPI) algorithm that can achieve these SPE-improving rules; see Algorithm \ref{alg: backward-policy-update}.

\begin{algorithm}[H]
	\SetAlgoLined
	\SetKwInOut{Input}{Input}
	\SetKwInOut{Output}{Output}
	\SetKwInOut{Init}{Initialize}
	\Input{$\bpi^{0}\not=\emptyset$}
	\Output{$\bpi^*=\bpi'$}
	\Init{$\bpi' \gets \bpi^{0}, \bpi \gets \emptyset$\;}
	
	\While{\textit{not stable ($\bpi' \neq \bpi$)}}{
		Update $\bpi \gets \bpi'$\;
		\For{$k \gets T-1$ \KwTo $0$}{
			\nonl \textbf{1. Policy Evaluation (PolEva)}\;
			Compute $Q^{\bpi'}_k(x,u), \forall x \in \mathcal{X}_k, u \in \mathcal{U}_k$ by \textit{TIC-TD} (Algorithm \ref{alg: TD-based PE})\;
			
			\nonl~\\
			
			\nonl \textbf{2. Policy Improvement (PolImp)}\;
			\For{$x \in \mathcal{X}_k$}{
				\eIf{$\exists u' \in \mathcal{U}_k \text{ s.t. } Q^{\bpi'}_k(x, u')> Q^{\bpi'}_k(x, \pi_k(x))$}{
					Assign $\pi'_k(x) \gets u'$ (arbitrarily)\;}{
					Assign $\pi'_k(x) \gets \pi_k(x)$\;}}
		}
	}
	\caption{Backward Policy Iteration}
	\label{alg: backward-policy-update}
\end{algorithm}

Subsequently, we refer to different parts of BPI as follows,
\begin{itemize}
	\item \textit{TIC-TD}: the recursive method developed in Section 3.1 for computing Q-values with \textit{adjustments}; see Algorithm \ref{alg: TD-based PE}. Since BPI integrates PolEva and PolImp steps at all $k\in\mathcal{T}$, Algorithm \ref{alg: TD-based PE} is also integrated into this PolEva-PolImp-loop such that $r^{\bpi'}_k, f^{\bpi'}_k, g^{\bpi'}_k, Q^{\bpi'}_k$ at the previous iteration $k=t+1$ can be reused to compute $Q^{\bpi'}_t$.
	\item \textit{PolEva-specs}: every elements in the PolEva block that includes time-extended, $\bpi'$-based evaluation criteria (i.e.  $Q^{\bpi'}_t(x, u)$) and the use of \textit{TIC-TD}-based computation.
	\item \textit{\textit{termination}:} the \textit{while}-condition in the outer-loop, i.e. $\bpi' = \bpi$.
	\item \textit{non-termination}: the \textit{if}-condition inside PolImp block.
	\item \textit{consistent tie-break:} the element in the \textit{else}-block; with \textit{consistent tie-break} rule, \textit{non-termination-condition} characterizes \textit{termination}.
	\item \textit{strictly-improving:} $Q^{\bpi'}_t(x, u') > Q^{\bpi'}_t(x, u)$ for any old and new actions $u, u'$.
	\item \textit{action-specs}: the choice of new action $u'$ that is arbitrary \textit{strictly-improving} one.
\end{itemize}

Next, we probe some analytical properties of the BPI algorithm with an aim to obtain convergence results; specifically, to see whether the algorithm converges and if it does, to determine the property of its converged policy. We start by showing that the BPI algorithm satisfies the (weak) SPE-improving rule in Definition \ref{def: SPE-improving rule}.
\begin{proposition} \label{prop: for FIN 3.3-2}
	Let $\bpi, \bpi'$ be the old and new policies obtained through BPI. Then,
	\begin{equation} \label{eq7: for FIN 3.3-2}
		\forall k\in\mathcal{T},~{}{}_k\bpi' \succeq_{eq} {}{}_k\bpi
	\end{equation}
	and thus, $\bpi' \succeq_{eq} \bpi$.
\end{proposition}
\begin{proof}
	We first note that
	\begin{equation} \label{eq5: for FIN 3.3-2}
		\forall k \in \mathcal{T}, \forall x \in \mathcal{X}_k, \quad Q^{\bpi'}_k(x, \pi'_k(x)) \geq Q^{\bpi'}_k(x, \pi_k(x)),
	\end{equation}
	implied by the PolEva step and action search step in BPI; see Algorithm \ref{alg: backward-policy-update}.
	
	Next, we show that (\ref{eq7: for FIN 3.3-2}) holds. We will use \eqref{eq5: for FIN 3.3-2} and backward induction to show
	\begin{equation} \label{eq4: for FIN 3.3-2}
		\forall x \in \mathcal{X}_k, Q^{\bpi'}_k(x, \pi'_k(x)) \geq Q^{\bpi'}_k(x, \pi_k(x)) \Rightarrow {}{}_k\bpi' \succeq_{eq} {}{}_k\bpi.
	\end{equation}
	
	(\textit{Base step}) For the base case of $k=T-1$, the premise in (\ref{eq4: for FIN 3.3-2}) states
	\begin{align*}
		\forall x \in \mathcal{X}_{T-1}, \quad Q^{\bpi'}_{T-1}(x, \pi'_{T-1}(x)) \geq Q^{\bpi'}_{T-1}(x, \pi_{T-1}(x)),
	\end{align*}
	which is equivalent to ${}{}_{T-1}\pi' \succeq_{eq} {}{}_{T-1}\pi$ by Definition \ref{def: SPE-Improve-Policy-Seq} and thus, the statement (\ref{eq4: for FIN 3.3-2}) holds.
	
	(\textit{Inductive step}) Suppose that the statement (\ref{eq4: for FIN 3.3-2}) holds for $k=t+1$. Hence, by \eqref{eq5: for FIN 3.3-2} for $k=t+1$, we have
	\begin{align}
		{}{}_{t+1}\bpi' \succeq_{eq} {}{}_{t+1}\bpi. \label{eq1: for FIN 3.3-2}
	\end{align}
	Then, at $k = t$, the premise of (\ref{eq4: for FIN 3.3-2}) states that
	\begin{align}
		\forall x \in \mathcal{X}_t, \quad Q^{\bpi'}_t(x, \pi'_t(x)) \geq Q^{\bpi'}_t(x, \pi_t(x)). \label{eq2: for FIN 3.3-2}
	\end{align}
	Combining (\ref{eq1: for FIN 3.3-2}) and (\ref{eq2: for FIN 3.3-2}), we have ${}{}_t\bpi' \succeq_{eq} {}{}_t\bpi$ by Definition \ref{def: SPE-Improve-Policy-Seq} and have thus shown that (\ref{eq4: for FIN 3.3-2}) holds for $k = t$.
	
	Finally, we can combine what we have from \eqref{eq5: for FIN 3.3-2} and the previously shown (\ref{eq4: for FIN 3.3-2}) to show that (\ref{eq7: for FIN 3.3-2}) holds; $\bpi' \succeq_{eq} \bpi$ is directly implied by Proposition \ref{prop: SPE-Improving Equivalence}.
\end{proof}

Note that Proposition \ref{prop: for FIN 3.3-2} is an implication of $\bpi'$\textit{-based PolEva-specs}, which manifests itself through the \textit{backward} update direction in BPI. Unfortunately, being (weakly) SPE-improving is still insufficient to establish monotonicity and this is due to its non-transitivity. To deal with this issue, we will define a lexicographical order on the policy set $\Pi^{MD}$. Before that, we present more properties pertaining to the relation and implication between the old and new policies obtained through BPI.

First, we present two results that characterize BPI's rules of \textit{strictly-improving action-specs} and \textit{consistent tie-break}.
\begin{corollary} \label{cor: implication-action-loop-2}
	Let $\bpi, \bpi'$ be the old and new policies obtained through BPI. Then, $\forall t \in \mathcal{T}, x \in \mathcal{X}_t$,
	\begin{align}
		(\, \exists u \in \mathcal{U}_t \text{ s.t. } Q^{\bpi'}_t(x, u) > Q_t^{\bpi'}(x, \pi_t(x))\,) \Rightarrow \pi'_t(x) \neq \pi_t(x).
	\end{align}
\end{corollary}
\begin{proof}
	Direct implication of \textit{strictly-improving action-specs}.
\end{proof}

\begin{corollary} \label{cor: implication-consistent-tie-break}
	Let $\bpi, \bpi'$ be the old and new policies obtained through BPI. Then, $\forall t \in \mathcal{T}, x \in \mathcal{X}_t$,
	\begin{equation} \label{eq1: implication-consistent-tie-break}
		Q^{\bpi'}_t(x, \pi'_t(x)) = Q^{\bpi'}(t, \pi_t(x)) \Rightarrow \pi'_t(x) = \pi_t(x).
	\end{equation}
\end{corollary}
\begin{proof}
	Direct implication of \textit{consistent tie-break} rule; if the premise of (\ref{eq1: implication-consistent-tie-break}) holds, then the search must have entered the \textit{else}-block and the conclusion follows.
\end{proof}

Then, we establish three lemmas concerning about the relations between two policies obtained through adjoint iterations of the BPI algorithm.
\begin{lemma} \label{lem: BWD tie-breaks}
	Let $\bpi, \bpi'$ be the old and new policies obtained through BPI. Then, for any $k \in \mathcal{T}$,
	\begin{align} \label{eq1: BWD tie-breaks}
		{}{}_k\bpi' \sim_{eq} {}{}_k\bpi \Leftrightarrow {}{}_k\bpi' = {}{}_k\bpi.
	\end{align}
\end{lemma}
\begin{proof}
	$(\Rightarrow)$ We will use mathematical induction to show that for any $k \in \mathcal{T}$,
	\begin{equation} \label{eq10: BWD tie-breaks}
		{}{}_k\pi' \sim_{eq} {}{}_k\pi \Rightarrow {}{}_k\pi'={}{}_k\pi.
	\end{equation}
	
	As our base case, we set $k = T-1$. First, note that by Definition \ref{def: SPE-Improve-Policy-Seq}, the premise of (\ref{eq10: BWD tie-breaks}) is equivalent to $\forall x \in \mathcal{X}_{T-1}$, $Q^{\bpi'}_{T-1}(x, \pi'_{T-1}(x)) = Q^{\bpi'}_{T-1}(x, \pi_{T-1}(x))$, which implies that $\forall x \in \mathcal{X}_{T-1}, \pi'_{T-1}(x) = \pi_{T-1}(x)$ by Corollary \ref{cor: implication-consistent-tie-break} and thus, we have shown ${}{}_{T-1}\bpi'={}{}_{T-1}\bpi$.
	
	Next, we start our inductive argument: if relation (\ref{eq10: BWD tie-breaks}) applies for $k=t+1$, then it also applies for $k = t$. By assumption, we have
	\begin{equation} \label{eq2: BWD tie-breaks}
		{}{}_{t+1}\bpi' \sim_{eq} {}{}_{t+1}\bpi \Rightarrow {}{}_{t+1}\bpi'={}{}_{t+1}\bpi.
	\end{equation}
	At case $k = t$, by Definition \ref{def: SPE-Improve-Policy-Seq}, the premise of (\ref{eq10: BWD tie-breaks}) is equivalent to
	\begin{align}
		& \text{ (i) } \forall x \in \mathcal{X}_t, Q^{\bpi'}_t(x, \pi'_t(x)) = Q^{\bpi'}_t(x, \pi_t(x)); \label{eq4: BWD tie-breaks}\\
		& \text{ (ii) } {}{}_{t+1}\bpi' \sim_{eq} {}{}_{t+1}\bpi. \label{eq3: BWD tie-breaks}
	\end{align}
	By applying the assumption (\ref{eq2: BWD tie-breaks}), condition (\ref{eq3: BWD tie-breaks}) implies
	\begin{equation} \label{eq5: BWD tie-breaks}
		{}{}_{t+1}\bpi'={}{}_{t+1}\bpi.
	\end{equation}
	Then, by Corollary \ref{cor: implication-consistent-tie-break}, condition (\ref{eq4: BWD tie-breaks}) implies
	\begin{equation} \label{eq6: BWD tie-breaks}
		\forall x \in\mathcal{X}_t,~\pi'_t(x) = \pi_t(x).
	\end{equation}
	The conclusion that (\ref{eq10: BWD tie-breaks}) applies for $k = t$ follows from (\ref{eq5: BWD tie-breaks}) and (\ref{eq6: BWD tie-breaks}).
	
	$(\Leftarrow)$ We will next show the converse by induction that for any $k \in \mathcal{T}$ 
	\begin{equation} \label{eq11: BWD tie-breaks}
		{}{}_k\bpi'= {}{}_k\bpi \Rightarrow {}{}_k\bpi' \sim_{eq} {}{}_k\bpi.
	\end{equation}
	
	As our base case at $k = T-1$, we can rewrite the premise as $\forall x \in \mathcal{X}_{T-1}, \pi'_{T-1}(x) = \pi_{T-1}(x)$, which implies $\forall x \in \mathcal{X}_{T-1}$, $Q^{\bpi'}_{T-1}(x, \pi'_{T-1}(x)) = Q^{\bpi'}_{T-1}(x, \pi_{T-1}(x))$ and by Definition \ref{def: SPE-Improve-Policy-Seq}, we have ${}{}_{T-1}\bpi' \sim_{eq} {}{}_{T-1}\bpi$.
	
	Next, we start our inductive argument: if relation (\ref{eq11: BWD tie-breaks}) applies for $k = t+1$, then it also applies for $k = t$. By assumption, we have
	\begin{equation} \label{eq9: BWD tie-breaks}
		{}{}_{t+1}\bpi' = {}{}_{t+1}\bpi \Rightarrow {}{}_{t+1}\bpi' \sim_{eq} {}{}_{t+1}\bpi.
	\end{equation}
	At case $k = t$, the premise ${}{}_t\bpi'={}{}_t\bpi$ can be written as $\forall x \in \mathcal{X}_t, \pi'_t(x) = \pi_t(x)$,
	which implies that
	\begin{equation} \label{eq8: BWD tie-breaks}
		\forall x \in \mathcal{X}_t, \, Q^{\bpi'}_t(x, \pi'_t(x)) = Q^{\bpi'}_t(x, \pi_t(x)).
	\end{equation}
	By Definition \ref{def: SPE-Improve-Policy-Seq}, the conclusions of (\ref{eq9: BWD tie-breaks}) and (\ref{eq8: BWD tie-breaks}) imply ${}{}_t\bpi' \sim_{eq} {}{}_t\bpi$ and thus, we have shown (\ref{eq11: BWD tie-breaks}) holds for $k = t$.
\end{proof}

\begin{lemma} \label{lem: FIN 3.3-1}
	Let $\bpi, \bpi'$ be the old and new policies obtained through BPI. In the event of \textit{non-termination}, then there exists $k\in \mathcal{T}$ such that the following hold
	\begin{enumerate}
		\item ${}{}_k\bpi' \succeq_{eq} {}{}_k\bpi$;
		\item $\exists x \in \mathcal{X}_k$, $Q^{\bpi'}_k\left(x, \pi'_k(x)\right) > Q^{\bpi'}_k\left(x, \pi_k(x)\right)$.
	\end{enumerate}
\end{lemma}
\begin{proof}
	First, the first claim is always true by Proposition \ref{prop: for FIN 3.3-2}, which says
	\begin{equation} \label{eq2: FIN 3.3-1}
		\forall k, {}{}_k\bpi' \succeq_{eq} {}{}_k\bpi.
	\end{equation}
	For the second claim, by Definition \ref{def: SPE-Improve-Policy-Seq}, (\ref{eq2: FIN 3.3-1}) implies 
	\begin{equation} \label{eq3: FIN 3.3-1}
		\forall k\in\mathcal{T}, \forall x \in \mathcal{X}_k, \, Q^{\bpi'}_k(x, \pi'_k(x)) \geq Q^{\bpi'}_k(x, \pi_k(x)).
	\end{equation}
	Since \textit{non-termination} is assumed, we must have $\bpi' \neq \bpi$ such that $\exists k,~x \in \mathcal{X}_k$ and
	\begin{align}
		\pi'_k(x) \neq \pi_k(x) &\Rightarrow Q^{\bpi'}_k(x, \pi'_k(x)) \neq Q^{\bpi'}_k(x, \pi_k(x)) \tag*{(by Corollary \ref{cor: implication-consistent-tie-break})}\\
		&\Rightarrow Q^{\bpi'}_k(x, \pi'_k(x)) > Q^{\bpi'}_k(x, \pi_k(x)). \tag*{(by (\ref{eq3: FIN 3.3-1}))}
	\end{align}
	Hence, the second claim follows.
\end{proof}

\begin{lemma} \label{lem: FIN 3.3-2}
	Let $\bpi, \bpi'$ be the old and new policies obtained through BPI. In the event of \textit{non-termination}, then $\exists k^*\in\mathcal{T}$ s.t. the following holds
	\begin{eqnarray}
		\exists x \in \mathcal{X}_{k^*},~~
		Q^{\bpi'}_{k^*}(x, \bpi'_{k^*}(x)) & > & Q^{\bpi'}_{k^*}(x,\bpi_{k^*}(x)), \label{cond-i} \\
		{}{}_{k^*}\bpi' & \succeq_{eq} & {}{}_{k^*}\bpi, \label{cond-ii} \\
		{}{}_{k^*+1}\bpi' & \sim_{eq} & {}{}_{k^*+1}\bpi. \label{cond-iii}
	\end{eqnarray}
\end{lemma}
\begin{proof}
	Let $k^*$ be the \textit{largest} index in the set of $k$'s realizing the two claims in Lemma \ref{lem: FIN 3.3-1}, i.e.
	\begin{equation} \label{eq:kstar}
		k^*=\max\{k\in \mathcal{T}:~\exists x \in \mathcal{X}_{k},
		Q^{\bpi'}_{k}(x, \bpi'_{k}(x)) > Q^{\bpi'}_{k}(x,\bpi_{k}(x)) \; \land \; {}{}_{k}\bpi' \succeq_{eq} {}{}_{k}\bpi\},
	\end{equation}
	where the set is not empty by Lemma \ref{lem: FIN 3.3-1}. This definition means that for any $k > k^*$,
	$$
	\neg \left(\, \exists x \in \mathcal{X}_{k}, 
	Q^{\bpi'}_{k}(x, \bpi'_{k}(x) > Q^{\bpi'}_{k}(x, \bpi_{k}(x)) \; \land \; {}{}_k\bpi' \succeq_{eq} {}{}_k\bpi \, \right)
	$$
	and by Definition \ref{def: SPE-Improve-Policy-Seq}, the above is equivalent to 
	\begin{equation} \label{eq5: FIN 3.3-2}
		\neg ({}{}_k\bpi' \succ_{eq} {}{}_k\bpi).
	\end{equation}
	Then, it is clear that \eqref{cond-i} and \eqref{cond-ii} hold by the conditions in the set of \eqref{eq:kstar}. Suppose that \eqref{cond-iii} is not true, i.e. $\neg \left({}{}_{k^*+1}\bpi' \sim_{eq} {}{}_{k^*+1}\bpi\right)$. By the second condition in the set of \eqref{eq:kstar}, $({}{}_{k^*+1}\bpi' \succ_{eq} {}{}_{k^*+1}\bpi)$, contradicting (\ref{eq5: FIN 3.3-2}). Therefore, we have shown the existence of $k^*$ by construction and the result follows.
\end{proof}


\subsubsection{Lex-Monotonicity} \label{sec:lexmono}
To characterize the monotonicity of the equilibrium policies, we introduce the lexicographical structure. Under which, we aim to show that two adjoint policies obtained through BPI are strictly monotonic, namely lex-monotonic.

In this section, we study the continuous state-action space setting and define three notations in which, while the subsequent results can be easily reduced to discrete setting by introducing the parallel definitions as in the remark below and thus the proof for the discrete setting is almost identical. In this regard, different settings manifest themselves through different lexicographic representations of a policy.

\begin{definition}[Basis of Policy $\mathcal{B}^{\bpi}$] \label{def: policy-basis}
	For any fixed $\bpi \in \Pi^{MD}$, we represent the basis of a policy $\mathcal{B}^{\bpi}$ by a $T$-dimensional vector function, whose $(k+1)$th entry is defined as $(\mathcal{B}^{\bpi})_k:~\mathcal{X}_k\mapsto \mathbb{R}$ given by $(\mathcal{B}^{\bpi})_k(x)=Q_k^{\bpi}(x, \pi_k(x))$ for $k\in\mathcal{T}$.
\end{definition}

\begin{definition}[Element-wise order $>_{e}$] \label{def: vec-order}
	Let $a, b$ be two functions where dom$(a) = $ dom$(b)$. Then, we say $a>_{e}b$ if $a(x) \geq b(x), \forall x \in \text{dom}(a)$ and $\exists x^* \text{ s.t. } a(x^*) > b(x^*)$.
\end{definition}

\begin{definition}[Lexicographic order $>_{lex}$] \label{def: lex-order}
	For any two policies $\bpi,\bpi'\in\Pi^{MD}$,
	let $k^*$ be the largest index $k \in \mathcal{T}$ such that $(\mathcal{B}^{\bpi})_k \neq (\mathcal{B}^{\bpi'})_k$. Then, we say $\mathcal{B}^{\bpi} >_{lex} \mathcal{B}^{\bpi'}$ if $(\mathcal{B}^{\bpi})_{k^*} >_e (\mathcal{B}^{\bpi'})_{k^*}$.
\end{definition}

\begin{remark}
    When the state space is discrete, the definitions above can be revised accordingly to accommodate the analyses below. Specifically, for each $\bpi \in \Pi^{MD}$, we define $\mathcal{B}^{\bpi}$ as $\mathbb{R}^d$-vector with $d = \sum_k |\mathcal{X}_k|$, whose entries are filled as follows.
	\begin{enumerate}
		\item For $k\in\mathcal{T}$, order state indices for $x \in \mathcal{X}_k$ and fix this order across updates.
		\item Then, fill in entries of the $\mathbb{R}^d$-vector $\mathcal{B}^{\bpi}$ with $(\mathcal{B}^{\bpi})_{k, x}\doteq Q_k^{\bpi}(x, \pi_k(x))$, according to the state order determined above in ascending time order, i.e. the $\big(\sum_{i=0}^{k-1}|\mathcal{X}_i|+j_k(x)\big)$-th entry of $\mathcal{B}^{\bpi}$ is $(\mathcal{B}^{\bpi})_{k, x}$, where $j_k(x)$ is the index of $x$ in $\mathcal{X}_k$.
	\end{enumerate}
	We also denote by $(\mathcal{B}^{\bpi})_k=(Q_k^{\bpi}(x, \pi_k(x)))_{x\in\mathcal{X}_k}$ the $\mathcal{X}_k$-dimensional vector, whose entries are filled according to a pre-fixed state order for any fixed $k$ (aligned with the first point).
	
	Then the element-wise order can be defined for two vectors similarly as follows: we say vectors $a>_e b$ if $a_j\ge b_j$ for any $j$ and there is a $j^*$ such that $a_{j^*}>b_{j^*}$. Subsequently, the definition of lexicographic order in the vector case simply follows Definition \ref{def: lex-order}.
\end{remark}

Note that the lexicographical order in Definition \ref{def: lex-order} is a variant of the conventional lexicographical order by the use of functions/vectors in place of scalars and the corresponding operator $>_{e}$. In the subsequent analyses, we will use the definitions above to show that BPI is lexicographically monotonic, where the SPE-improving property obtained in Proposition \ref{prop: for FIN 3.3-2} becomes a sufficient condition. Such lex-monotonicity result is parallel to the policy improvement theorem (PIT) in standard RL approaches.

We will here onward refer to the $k^*$ defined in Lemma \ref{lem: FIN 3.3-2} as our \textit{lex-index}, indicating a particular time index, after which all later time-state pairs have obtained their SPE policies. We are now ready to establish our lex-monotonicity result.
\begin{theorem}[Lex-monotonicity] \label{thm: lex-monotone}
	Let $\bpi, \bpi'$ be the old and new policies obtained through BPI. In the event of \textit{non-termination},
	$$
	\mathcal{B}^{\bpi'} >_{lex} \mathcal{B}^{\bpi}.
	$$
\end{theorem}
\begin{proof}
	Let $k^*$ be the \textit{lex-index} that satisfies \eqref{cond-i}-\eqref{cond-iii} in Lemma \ref{lem: FIN 3.3-2}. By Lemma \ref{lem: BWD tie-breaks}, (\ref{cond-ii}) implies
	\begin{align}
		{}{}_{k^*+1}\bpi' = {}{}_{k^*+1}\bpi, \label{cond-ii-b}
	\end{align}
	which then implies
	\begin{align}
		\forall k \geq k^*, \quad Q^{\bpi'}_{k}(x, u) = Q^{\bpi}_k(x, u),~\forall (x, u) \in \mathcal{X}_k \times \mathcal{U}_k. \label{eq1: lex-monotone}
	\end{align}
	
	\textit{Claim 1.} $\forall k \geq k^*+1, (\mathcal{B}^{\bpi'})_{k} = (\mathcal{B}^{\bpi})_{k}$ 
	
	Consider any $k \geq k^*+1$. By (\ref{cond-ii-b}), we have $\pi'_k(\cdot) = \pi_k(\cdot)$. By substituting them into (\ref{eq1: lex-monotone}), we have
	$$
	Q^{\bpi'}_{k}(x, \pi'_k(x)) = Q^{\bpi}_k(x, \pi_k(x)),~\forall x \in \mathcal{X}_k,
	$$
	which says that $(\mathcal{B}^{\bpi'})_{k}$ and $(\mathcal{B}^{\bpi})_{k}$ are equal for any $k \geq k^*+1$ that proves Claim 1.
	
	\textit{Claim 2.} At $k = k^*, (\mathcal{B}^{\bpi'})_{k^*} >_e (\mathcal{B}^{\bpi})_{k^*}$
	
	Set $x^*$ to be one $x$ that realizes (\ref{cond-i}). Then, we have
	\begin{align}
		Q^{\bpi'}_{k^*}(x^*,\pi'_{k^*}(x^*)) > Q^{\bpi'}_{k^*}(x^*, \pi_{k^*}(x^*)) = Q^{\bpi}_{k^*}(x^*, \pi_{k^*}(x^*)). \tag*{(by (\ref{eq1: lex-monotone}))}
	\end{align}
	Together with (\ref{cond-iii}), we have $(\mathcal{B}^{\bpi'})_{k^*} >_e (\mathcal{B}^{\bpi})_{k^*}$ by the definition of $>_e$.
	
	Finally, from Claims 1 and 2, we conclude that $\mathcal{B}^{\bpi'} >_{lex} \mathcal{B}^{\bpi}$.
\end{proof}

We note that by the \textit{strict} lex-monotonicity result in Theorem \ref{thm: lex-monotone}, we have shown that the mapping $\mathcal{B}$ is one-to-one, i.e. $\mathcal{B}^{\bpi'} = \mathcal{B}^{\bpi} \Leftrightarrow \bpi' = \bpi$, on the set of policies encountered in BPI's update. This property is central to the analysis in Section \ref{sec:contsa}, where value-based arguments are used to describe the movement of policy across updates.


\subsubsection{Discrete State-Action Space (Policy-based Analysis)} \label{sec:discretesa}
Under a discrete state-action setting, we leverage the lex-monotonicity result above to show that the BPI (Algorithm \ref{alg: backward-policy-update}) terminates in finite steps.
\begin{theorem}[Finite Termination]
	Assuming discrete state-action spaces, the BPI (Algorithm \ref{alg: backward-policy-update}) terminates in finite time.
\end{theorem}
\begin{proof}
	First, by Theorem \ref{thm: lex-monotone}, we have that for any $\bpi, \bpi'$ consecutive policies in BPU,
	\begin{align}
		\mathcal{B}^{\bpi'} >_{lex} \mathcal{B}^{\bpi}.
	\end{align}
	This implies that the BPI visits different basis across updates and specifically, there won't be any cycling of basis by the transitivity of lex-order. 
	
	Secondly, by assumption of discrete state-action spaces, we have a finitely many possible policies to visit. Moreover, since our space of basis is completely spanned by $\mathcal{B}^\Pi$, we also have finitely many basis to visit, i.e. $|\mathcal{B}^\Pi| < \infty$.
	
	From these two observations, finite termination directly follows.
\end{proof}

We now present the result concluding that the converged policy from BPI is a SPE policy. 
\begin{theorem} \label{thm: to-SPE-Policy}
	Let $\bpi, \bpi'$ be the old and new policies obtained through BPI. If $\bpi' = \bpi$, then
	\begin{align}
		\forall t\in\mathcal{T}, \forall x \in \mathcal{X}_t, Q^{\bpi}_t(x, \pi_t(x)) \geq Q^{\bpi}_t(x, u), \forall u \in \mathcal{U}_t. \label{eq4: to-SPE-policy}
	\end{align}
\end{theorem}
\begin{proof}
	Suppose otherwise, then the following must be true 
	\begin{align}
		\exists t\in\mathcal{T}, \exists x \in \mathcal{X}_{t}, \text{ s.t. } \exists u \in \mathcal{U}_t, Q^{\bpi}_{t}(x, u) > Q^{\bpi}_{t}(x, \pi_{t}(x)). \label{eq3: to-SPE-policy}
	\end{align}
	We focus on one such $t \doteq t^*$ and by assumption of this theorem (i.e. $\bpi' = \bpi$), we have
	\begin{align}
		{}{}_{t^*}\bpi' = {}{}_{t^*}\bpi. \label{eq1: to-SPE-policy}
	\end{align}
	By applying (\ref{eq1: to-SPE-policy}) to (\ref{eq3: to-SPE-policy}), we obtain
	\begin{align}
		\exists x \in \mathcal{X}_{t^*}, \text{ s.t. } \exists u \in \mathcal{U}_{t^*}, \quad Q^{\bpi'}_{t^*}(x, u) > Q^{\bpi'}_{t^*}(x, \pi_{t^*}(x)). \label{eq2: to-SPE-policy}
	\end{align}
	Now focus on one such $x \doteq x^*$. By Corollary \ref{cor: implication-action-loop-2}, this implies $\pi'_{t^*}(x^*) \neq \pi_{t^*}(x^*)$ which contradicts the theorem assumption of $\bpi' = \bpi$. Thus, we conclude that our supposition is false and that (\ref{eq4: to-SPE-policy}) must hold.
\end{proof}

Since we have shown that the BPI algorithm will converge in finite time, we have thus shown that this policy at convergence is a SPE policy by Theorem \ref{thm: to-SPE-Policy}. Note that this finite-termination result and all the results derived beforehand hold for \textit{arbitrary action-specs} since the main argument is to permute over the whole policy space $\Pi^{MD}$ that is discrete by assumption. Unfortunately, the proof of Theorem \ref{thm: to-SPE-Policy} does not claim about the convergence rate nor does it extend to more complicated setting such as continuous state-action space. Permuting argument will treat any \textit{action-specs} similarly, concluding the convergence rate to be at most the number of permutation there is. Such analysis is not tight since different \textit{action-specs} would lead to different rates. The limitation of policy-based analysis is even clearer through the continuous state-action case when permuting over infinite-dimensional spaces (i.e. $\left\|\Pi^{MD}\right\| = \infty$) will not conclude anything about finite termination/convergence.


\subsubsection{Continuous State-Action Spaces (Value-based Analysis)} \label{sec:contsa}
In this subsection, we will obtain convergence guarantees for continuous state-action spaces by applying value-based analysis, that is to show convergence by showing $\mathcal{B}' = \mathcal{B}$ rather than $\bpi' = \bpi$. These two termination conditions are interchangable as long as we have one-to-one mapping $\mathcal{B}$, which through Theorem \ref{thm: lex-monotone} has been shown to apply for any algorithm belonging to BPI class. As described in the preceding subsection, the insufficiency of policy-based analysis is caused by its permutative argument that is used to maintain generality on the choice of \textit{action-specs}. In value-based analysis, we are no longer able to keep this generality. In what follows, we will explore three important \textit{action-specs} in RL, verify that each belongs to BPI class, and derive new convergence results for each.
\paragraph{Full-sweep argmax.} We specify our \textit{action-specs} to
\begin{align}
	\pi'_k(x) \gets \argmax_{u \in \mathcal{U}_t} Q^{\bpi'}_t(x, u) \text{ (arbitrarily)} \label{action-specs-case-I}
\end{align}
for any $x\in\mathcal{X}_t$ and retain the \textit{non-termination-condition} of BPI, i.e.
\begin{align}
	\exists u' \in \mathcal{U}_k, \text{ s.t. } Q^{\bpi'}_k(x, u') > Q^{\bpi'}_k(x, \pi_k(x)). \label{non-termination-case-I}
\end{align}
Thus, full-sweep argmax belongs to BPI class and all the preceding analyses apply. In what follows, we will use value-based analysis to derive some convergence results.

To ease notation, we define the mappings $q^{{}{}_{t+1}\bpi}_{t,x}: \mathcal{U}_{t,x} \rightarrow \mathbb{R}$ and $q_{t,x,u}: \, {}{}_{t+1}\Pi^{MD} \rightarrow \mathbb{R}$ by
$$
q^{{}{}_{t+1}\bpi}_{t,x}(u) \doteq q_{t,x,u}\left({{}{}_{t+1}\bpi}\right) \doteq Q^{\bpi}_t(x, u).
$$
Moreover, for each iteration $i\in\mathbb{N}$ and any fixed $(t,x)\in\mathcal{T}\times \mathcal{X}_t$, we define $q^{(i)}_{t,x}(u) \doteq q^{{}{}_{t+1}\bpi^{(i)}}_{t,x}(u)$ and $u^{(i)}_{t,x} \doteq \pi^{(i)}_t(x)$, indicating the current action-values and the current action, respectively. 

Let us now consider the case when there are non-unique local-maximizers ${}{}_{t+1}\bpi^*$ such that given a fixed $t, x$, we will have multiple limiting action-value functions $q^{(\infty)}_{t,x}$. This will pose an issue in the update of $u^{(i)}_{t,x}$, whose convergence requires a well-defined limiting action-value $q^{(\infty)}_{t,x}$. In the subsequent analyses, we address such an issue by showing the existence of an iteration index $i^*_{t+1}$, at which termination to a \textit{unique} local-optima ${}{}_{t+1}\bpi^{(\infty)}$ is guaranteed to happen. Once we have such $i^*_{t+1}$, $\forall i \geq i^*_{t+1}$, we will be dealing with a \textit{unique} limiting action-value $q^{(\infty)}_{t,x} \doteq q^{{}{}_{t+1}\bpi^{(\infty)}}_{t,x}$ and by noting that locally optimal actions with respect to this \textit{unique} action-value are well-defined, the analysis on the update of $u^{(i)}_{t,x}$ can follow naturally. 

\begin{assumption} \label{assum:existsglobalopt}
	At each iteration $i\in\mathbb{N}$, for any pair $(k,x)\in\mathcal{T}\times \mathcal{X}_k$, there exists the global optimum of $q^{(i)}_{k,x}(\cdot)$ over $\mathcal{U}_k$, i.e. there is a map $u^{(i)}(k,x)\in \arg\max_{u\in \mathcal{U}_k}q^{(i)}_{k,x}(u)$.
\end{assumption}
Assumption \ref{assum:existsglobalopt} is related to the compactness of $\mathcal{U}_k$ and the continuity of $q^{(i)}_{k,x}: \mathcal{U}_k \rightarrow \mathbb{R}$, which is linked to the problem specification.

\begin{theorem}[Finite Termination]
	Suppose that Assumption \ref{assum:existsglobalopt} holds. There exists $i^*$ such that for all $i \geq i^*$ and $t \in \mathcal{T}$,
	\begin{align}
		\forall x \in \mathcal{X}_t, \left\|Q^{(i+1)}_t(x, \pi^{(i+1)}_t(x)) - Q^{(i)}_t(x, \pi^{(i)}_t(x))\right\| = 0. \label{termination-cond-case-I}
	\end{align}
	Moreover, we have $i^* = 1$.
	\label{thm: asymp-conv-case-I}
\end{theorem}
\begin{proof}
	Assume any initial policy $\bpi^{(0)}$. We note that showing (\ref{termination-cond-case-I}) is equivalent to showing that given $i^* = 1$, $\forall k \in \mathcal{T}$,
	\begin{align}
		\forall x \in \mathcal{X}_k, \left\| q^{(i^*+1)}_{k,x}(u^{(i^*+1)}_{k,x}) - q^{(i^*)}_{k,x}(u^{(i^*)}_{k,x}) \right\| = 0. \label{eq1: asymp-conv-case-I}
	\end{align}
	
	(\textit{Base step.}) At the base case $k = T-1$, we have
	\begin{align}
		q^{(i+1)}_{T-1,x}(u) = q^{(i)}_{T-1,x}(u) = q^*_{T-1, x}(u), \forall x \in \mathcal{X}_{T-1}, \forall u \in \mathcal{U}_{T-1} \label{eq4: asymp-conv-case-I}
	\end{align}
	By Assumption \ref{assum:existsglobalopt}, a global optimum exists. Since the full-sweep argmax \textit{action-specs} prescribes $u^{(i+1)}_{T-1,x} \gets \argmax_{u \in \mathcal{U}_{T-1}} q^{(i+1)}_{T-1,x}(u)$ at $i = 0$, by \textit{consistent tie-break}, we must have $u_{T-1,x}^{(2)} = u_{T-1,x}^{(1)}, \forall x \in \mathcal{X}_{T-1}$. Combining this with (\ref{eq4: asymp-conv-case-I}), we have shown (\ref{eq1: asymp-conv-case-I}).
	
	(\textit{Inductive step.}) Set $i = 1$. Suppose that (\ref{termination-cond-case-I}) holds for $k = t+1$, we have ${}{}_{t+1}\bpi^{(i+1)} = {}{}_{t+1}\bpi^{(i)}$. By Lemma \ref{lem: BWD tie-breaks}, this implies
	\begin{align}
		\forall x \in \mathcal{X}_t, u\in \mathcal{U}_t,~q^{(i+1)}_{t,x}(u) = q^{(i)}_{t,x}(u). \label{eq2: asymp-conv-case-I}
	\end{align}
	Let us now consider any arbitrary $x \in \mathcal{X}_t$. By Assumption \ref{assum:existsglobalopt}, (\ref{eq2: asymp-conv-case-I}), and our \textit{action-specs},
	$$
	u^{(i)}_{t,x}=\argmax_{u \in \mathcal{U}_t} q^{(i)}_{t,x}(u)=\argmax_{u \in \mathcal{U}_t} q^{(i+1)}_{t,x}(u)=u^{(i+1)}_{t,x}.
	$$
	Therefore,
	\begin{align}
		\left\| q^{(i+1)}_{t,x}\left(u^{(i+1)}_{t,x}\right) - q^{(i)}_{t,x}\left(u^{(i)}_{t,x}\right) \right\| = \left\| q^{(i+1)}_{t,x}\left(u^{(i)}_{t,x}\right) - q^{(i)}_{t,x}\left(u^{(i)}_{t,x}\right) \right\| = 0
	\end{align}
	showing that (\ref{eq1: asymp-conv-case-I}) holds for $k = t$.
\end{proof}
By Theorem \ref{thm: asymp-conv-case-I}, we have $\mathcal{B}^{\bpi^{(2)}} = \mathcal{B}^{\bpi^{(1)}}$. By one-to-one $\mathcal{B}$, we have BPI's termination condition $\bpi^{(2)} = \bpi^{(1)}$. We may then apply Theorem \ref{thm: to-SPE-Policy} to conclude termination to (global) SPE-policy in just two iterations. 

Next, we reveal the fact that full-sweep argmax is infeasible in practice under continuous $\mathcal{U}_k$ and while discretization techniques may be applied, the argmax computation will quickly become intractable as the discretization dimension of $\mathcal{U}_k$ increases. In such situation, local search methods are often desirable to trade-off performance (i.e. allowing termination/convergence to \textit{local}\footnote{The `locality' here refers to the value-optimization landscape as a function of action variable $u$ at a fixed time $t$ and state $x$ and thus, is irrelevant to the `locally optimal' plan terminology of SPE policy, where 'locality' refers to the sequential structure.} SPE-policy) for tractability.
\begin{definition}[$\lambda$-Local SPE-Policy] \label{def: local-SPE-policy}
	Any policy $\bpi \in \Pi^{MD}$ is a local SPE-policy if it satisfies
	\begin{align*}
		\forall k\in\mathcal{T},~x\in \mathcal{X}_k, \quad Q^{\bpi}_k(x, \pi_k(x)) \geq Q^{\bpi}_k(x, u), \forall u \in \mathcal{N}_k(\pi_k(x), \lambda),
	\end{align*}
	where $\mathcal{N}_k(\pi_k(x), \lambda)$ is the neighbourhood of $\pi_k(x)$ with the set radius $\lambda > 0$, i.e. $$\mathcal{N}_k(\pi_k(x), \lambda)\doteq\{u\in\mathcal{U}_k:~|u-\pi_k(x)|<\lambda\}$$.
\end{definition}

\paragraph{Local-sweep argmax.} To capture this localized search aim, we modify the full-sweep argmax \textit{action-specs} as
\begin{align}
	\pi'_k(x) \gets \argmax_{u \in \mathcal{N}(\pi_k(x), \lambda)} Q^{\bpi'}_t(x, u) \text{ (arbitrarily)} \label{action-specs-case-II}
\end{align}
which needs to be accompanied by a modification to BPI's \textit{non-termination-condition} to
\begin{align}
	\exists u' \in \mathcal{N}(\pi_k(x), \lambda) \text{ s.t. } Q^{\bpi'}_k(x, u') > Q^{\bpi'}_k(x, \pi_k(x)) \label{non-termination-case-II}
\end{align}
We note that (\ref{non-termination-case-II}) is necessary to characterize its termination policy. Suppose that we retain (\ref{non-termination-case-I}) and the global optimum $u' \in \mathcal{U}_k$ is not in the current neighborhood $\mathcal{N}(\pi_k(x), \lambda)$, we may encounter the situations where (i) we have non-unique solution to the argmax problem in (\ref{action-specs-case-II}) and no termination happens due to inconsistent choices of solution happening in every consecutive iterations such that \textit{if-condition} is always satisfied, or (ii) we have termination either when there is no non-uniqueness issue or by coincidental outputting of the same solution in consecutive iterations in presence of non-uniqueness issue, which then gives the wrong conclusion that \textit{else-condition} has been satisfied and that the converged policy is a global SPE policy. Both cases are not desirable to our analysis. Moreover, the \textit{non-termination-condition} (\ref{non-termination-case-II}) has an additional advantage of requiring the PolEva computation only up to $\forall u \in \mathcal{N}(\pi_k(x), \lambda)$ which can reduce the computational burden in each iteration. Without (\ref{non-termination-case-II}), Corollaries \ref{cor: implication-action-loop-2} and \ref{cor: implication-consistent-tie-break} are countered by situation (ii) and (i), respectively. By modifying \textit{non-termination-condition} to (\ref{non-termination-case-II}), we can change the premise of Corollary \ref{cor: implication-action-loop-2} as
\begin{align}
	\exists u \in \mathcal{N}(\pi_t(x), \lambda) \text{ s.t. } Q^{\bpi'}_t(x, u) > Q^{\bpi'}_t(x, \pi_t(x)) \label{cor-18-premise-modif}
\end{align}
and recover both corollaries. The results up to Theorem \ref{thm: lex-monotone} then apply as they rely solely on these two corollaries. In what follows, we will show finite termination with value-based analysis, which can be validated once Theorem \ref{thm: lex-monotone} applies by retaining the injectivity of $\mathcal{B}$. However, we need a stronger assumption than Assumption \ref{assum:existsglobalopt}.

\begin{assumption} \label{assum:condqU}
	At each iteration $i\in\mathbb{N}$, for any pair $(k,x)\in \mathcal{T}\times \mathcal{X}_k$, $q^{(i)}_{k,x}(u)$ is continuous and bounded over $u\in\mathcal{U}_k$. Moreover, $\mathcal{U}_k$ is compact.
\end{assumption}

\begin{theorem}[Finite Termination] \label{thm: asymp-conv-case-II}
	Suppose that Assumption \ref{assum:condqU} holds. 

    For any fixed $t \in \mathcal{T}$, if for each $k \in _{t+1}\mathcal{T}$,
    \begin{align}
        \exists i^*_k < \infty \text{ s.t. } \forall i \geq i^*_k, \forall x \in \mathcal{X}_k, \quad \left\|Q^{(i+1)}_k\left(x, \pi^{(i+1)}_k(x)\right) - Q^{(i)}_k\left(x, \pi^{(i)}_k(x)\right)\right\| = 0, \label{eq13: asymp-conv-case-II}
    \end{align}
    then, for any fixed $x \in \mathcal{X}_t$,
    \begin{align}
        \exists i^*_{t,x} < \infty \text{ s.t. } \forall i \geq i^*_{t,x}, \quad \left\| Q_t^{(i+1)}(x, \pi^{(i+1)}_t(x)) - Q_t^{(i)}(x, \pi^{(i)}_t(x)) \right\| = 0. \label{eq14: asymp-conv-case-II}
    \end{align}
    Moreover, if for each $t \in \mathcal{T}$, $i^*_{t,x}$ is bounded in $x \in \mathcal{X}_t$\footnote{To iterate some instances when this assumption is met: (i) discrete $\mathcal{X}_t$, (ii) as in Theorem \ref{thm: asymp-conv-case-I}, and (iii) distance of initialization $u^{(0)}_{t,x}$ to the corresponding local optima $u^{(\infty)}_{t,x}$ is bounded in $x$.}, then the following holds
    \begin{align}
    	\exists i^* < \infty \text{ s.t. } \forall i \geq i^*, \left\|Q^{(i+1)}_t\left(x, \pi^{(i+1)}_t(x)\right) - Q^{(i)}_t\left(x, \pi^{(i)}_t(x)\right)\right\| = 0, \forall t \in \mathcal{T}, x \in \mathcal{X}_t. \label{eq15: asymp-conv-case-II}
    \end{align}
\end{theorem}
\begin{proof}

By Lemma \ref{lem: FIN 3.3-2}, (\ref{eq13: asymp-conv-case-II}) means that at iteration $i \geq i^*_{t+1}$, the lex-index $k^* \leq t$ and it implies ${}{}_{t+1}\bpi^{(i)} = {}{}_{t+1}\bpi^{(\infty)}$ and correspondingly,
\begin{align}
	\forall x \in \mathcal{X}_t, u \in \mathcal{U}_{t}, \quad q^{(i+1)}_{t,x}(u) = q^{(i)}_{t,x}(u) = q^{(\infty)}_{t,x}(u) \label{eq11: asymp-conv-case-II}
\end{align}
Suppose $k^* < t$, we can set $i^*_{t,x} = i^*_{t+1} < \infty, \forall x \in \mathcal{X}_t$, thus showing (\ref{eq14: asymp-conv-case-II}). Otherwise ($k^*=t$), $\forall i \geq i_{t+1}^*$ and for any fixed $x \in \mathcal{X}_t$, we have \footnotesize
\begin{align}
	\left\|q^{(i+1)}_{t,x}\left(u^{(i+1)}_{t,x}\right) - q^{(i)}_{t,x}\left(u^{(i)}_{t,x}\right) \right\| &\leq \left\| q^{(i+1)}_{t,x}\left(u^{(i+1)}_{t,x}\right) - q^{(i+1)}_{t,x}\left(u^{(i)}_{t,x}\right) \right\| + \left\| q^{(i+1)}_{t,x}\left(u^{(i)}_{t,x}\right) - q^{(i)}_{t,x}\left(u^{(i)}_{t,x}\right) \right\| \nonumber \\
	&= \left\| q^{(i+1)}_{t,x}\left(u^{(i+1)}_{t,x}\right)  - q^{(i+1)}_{t,x}\left(u^{(i)}_{t,x}\right) \right\| \nonumber \\
	&= \left\| q^{(\infty)}_{t,x}\left(u^{(i+1)}_{t,x}\right)  - q^{(\infty)}_{t,x}\left(u^{(i)}_{t,x}\right) \right\|, \nonumber
\end{align} \normalsize
where the second and third equations hold by (\ref{eq11: asymp-conv-case-II}). It thus remains to show that $\exists i^*_{t,x} \in [i^*_{t+1}, \infty)$ s.t. $\forall i \geq i^*_{t,x}$,
\begin{align}
    \left\| q^{(\infty)}_{t,x}\left(u^{(i+1)}_{t,x}\right)  - q^{(\infty)}_{t,x}\left(u^{(i)}_{t,x}\right) \right\| = 0.
\end{align}
First, we note that since $k^* = t$, $\exists x \in \mathcal{X}_t$ where the sequence $\left\{q^{(\infty)}_{t, x}\left(u^{(i)}_{t,x}\right): i \geq 0\right\}$ is increasing. By Assumption \ref{assum:condqU}, such sequence is bounded above by $\sup\left\{q^{(\infty)}_{t,x}(u): u \in \mathcal{U}_t\right\}$ and thus, convergent. This then implies
\begin{align}
	\lim_{i \rightarrow \infty} \left\| q^{(\infty)}_{t,x}\left(u^{(i+1)}_{t,x}\right) - q^{(\infty)}_{t,x}\left(u^{(i)}_{t,x}\right) \right\| = 0. \label{eq8: asymp-conv-case-II}
\end{align}
Next, we will show that the limit in (\ref{eq8: asymp-conv-case-II}) is attained at some finite iteration $i^*_{t,x}$. Suppose otherwise, the accumulation point
$$
\sup\left\{q^{(\infty)}_{t,x}\left(u^{(i)}_{t,x}\right): i \geq i^*_{t+1}\right\} = \sup\left\{q^{(\infty)}_{t,x}(u): u \in \bigcup_{i \geq i^*_{t+1}} \mathcal{N}\left(u^{(i)}_{t,x}, \lambda\right)\right\} \label{eq1: asymp-conv-case-II}
$$
is never attained. Since by Assumption \ref{assum:condqU}, $\mathcal{U}_{t,x}$ is bounded,
$$\bigcup_{i \geq i^*_{t+1}} \mathcal{N}\left(u^{(i)}_{t,x}, \lambda \right) = \bigcup_{i \geq i^*_{t+1}} \left(u^{(i)}_{t,x} - \lambda, u^{(i)}_{t,x} + \lambda\right)$$ 
must also be bounded. Moreover, by Assumption \ref{assum:condqU}, $q^{(\infty)}_{t,x}$ is continuous and implies by the extreme value theorem that $\bigcup_{i \geq i^*_{t+1}} \mathcal{N}\left(u^{(i)}_{t,x}, \lambda \right)$ is open.

We first consider the case when the sequence $\left\{u^{(i)}_{t,x} : i \geq i^*_{t+1}\right\}$ is non-monotonic with respect to $u^{(i^*_{t+1})}_{t,x}$. Therefore, $\exists i^* \geq [i^*_{t+1}, \infty)$ s.t. $\left\| u^{(i^* + 1)}_{T-1,x} - u^{(i^*)}_{T-1,x} \right\| > \left\| u^{(i^* + 2)}_{T-1,x} - u^{(i^*)}_{T-1,x} \right\|$ and $u^{(i^* + 2)}_{T-1,x} \in \mathcal{N}\left(u^{(i^*)}, \lambda\right)$. This leads to a contradiction as in iteration $i^* + 1$, $u^{(i^* + 2)}$ should have been chosen then instead of $u^{(i^* + 1)}$.

Consider next bounded and monotonic $\left\{u^{(i)}_{t,x} : i \geq i^*_{t+1}\right\}$. Then, $\lim_{i \rightarrow \infty} u^{(i)}_{t,x}$ exists and
$$
\forall \epsilon > 0, \exists i^* \in [i^*_{t+1}, \infty) \text{ s.t. } \forall i \geq i^*, \left\| u^{(i+1)}_{t,x} - u^{(i)}_{t,x}\right\| < \epsilon.
$$
Set $\epsilon = \frac{\lambda}{2}$. Thus, we must have
$$
\left\| u^{(i^*+2)}_{t,x} - u^{(i^*)}_{t,x}\right\| \leq \left\| u^{(i^*+2)}_{t,x} - u^{(i^* + 1)}_{t,x} \right\| + \left\| u^{(i^* + 1)}_{t,x} - u^{(i^*)}_{t,x} \right\| < \lambda,
$$
such that $u^{(i^*+2)}_{t,x} \in \mathcal{N}\left(u^{(i^*)}_{t,x}, \lambda \right)$. If $u^{(i^*+2)}_{t,x} \neq u^{(i^*+1)}_{t,x}$, we may apply similar argument as in the non-monotonic case to conclude contradiction. Otherwise, we must have $\lim_{i \rightarrow \infty} u^{(i)}_{t,x} = u^{(i^*+1)}_{t,x} \in \bigcup_{i \geq i^*_{t+1}} \mathcal{N}\left(u^{(i)}_{t,x}, \lambda \right)$ which contradicts $\bigcup_{i \geq i^*_{t+1}} \mathcal{N}\left(u^{(i)}_{t,x}, \lambda \right)$ being open. 
Therefore, our supposition must be false: the $0$-limit in (\ref{eq8: asymp-conv-case-II}) must be attained at some finite iteration $i^*_{t,x} \geq i^*_{t+1}$ which concludes that (\ref{eq14: asymp-conv-case-II}) holds.

Showing (\ref{eq15: asymp-conv-case-II}) is equivalent to showing that (\ref{eq13: asymp-conv-case-II}) holds for all $k \in \mathcal{T}$. We prove the latter by induction.

(Base case.) At $k = T-1$, (\ref{eq11: asymp-conv-case-II}) holds for all $i \geq 0$ such that for any fixed $x \in \mathcal{X}_{T-1}$,
\begin{align*}
    \left\|q^{(i+1)}_{T-1,x}\left(u^{(i+1)}_{T-1,x}\right) - q^{(i)}_{T-1,x}\left(u^{(i)}_{T-1,x}\right) \right\| = \left\| q^{(\infty)}_{T-1,x}\left(u^{(i+1)}_{T-1,x}\right)  - q^{(\infty)}_{T-1,x}\left(u^{(i)}_{T-1,x}\right) \right\|.
\end{align*}
By setting $i^*_{T} = 0$, we can apply the proof for (\ref{eq14: asymp-conv-case-II}) (i.e. $k = t$) to show that $i^*_{T-1,x} < \infty$ for any fixed $x \in \mathcal{X}_{T-1}$. Since we have $i^*_{T-1,x}$ bounded in $x \in \mathcal{X}_{T-1}$ by assumption, we can set $i^*_{T-1} = \sup\{i^*_{T-1,x}: x \in \mathcal{X}_{T-1}\} < \infty$ to conclude that (\ref{eq13: asymp-conv-case-II}) holds for $k = T-1$.

(Inductive step.) Suppose (\ref{eq13: asymp-conv-case-II}) holds for $k \in _{t+1}\mathcal{T}$, this is exactly the assumption for (\ref{eq14: asymp-conv-case-II}), to which we have shown the existence of $i^*_{t,x} < \infty$ for any fixed $x \in \mathcal{X}_t$. Similarly by the assumption of $i^*_{t,x}$ bounded in $x \in \mathcal{X}_t$, we can set $i^*_t = \sup \{i^*_{t,x}: x \in \mathcal{X}_t\} < \infty$ thus showing that (\ref{eq13: asymp-conv-case-II}) holds for $k = t$.

Finally, by noting that $i^*_0 \geq i^*_{1} \geq \dotso \geq i^*_{T-1}$, we can set $i^* = i^*_0$ to show (\ref{eq15: asymp-conv-case-II}).
\end{proof}

We will next re-derive a result similar to Theorem \ref{thm: to-SPE-Policy}, whose conclusion becomes unprovable once (\ref{eq2: to-SPE-policy}) is no longer the premise of Corollary \ref{cor: implication-action-loop-2} (i.e. \textit{modified} to (\ref{cor-18-premise-modif})). To rectify this issue, we will modify its conclusion to reflect the $\lambda$-local SPE policy; see Definition \ref{def: local-SPE-policy}.
\begin{theorem}[Converged policy is $\lambda$-local SPE-policy]
	If $\bpi' = \bpi$, then
	\begin{align}
		\forall t\in \mathcal{T}, \forall x \in \mathcal{X}_t, \quad Q^{\bpi}_k(x, \pi_k(x)) \geq Q^{\bpi}_k(x, u), \forall u \in \mathcal{N}(\pi_k(x), \lambda). \label{eq1: to-local-SPE-policy}
	\end{align}
	\label{thm: to-local-SPE-policy}
\end{theorem}
\begin{proof}
	Suppose otherwise,
	\begin{align*}
		\exists k, x \text{ s.t. } \exists u \in \mathcal{N}(\pi_k(x), \lambda), \quad Q^{\bpi}_k(x, u) > Q^{\bpi}_k(x, \pi_k(x)).
	\end{align*}
	Focus on one such $k$. By assumption that $\bpi' = \bpi$, we have ${}{}_k\bpi' = {}{}_k\bpi$ which then implies
	$$
	\exists x \in \mathcal{X}_k \text{ s.t. } \exists u \in \mathcal{N}(\pi_k(x), \lambda), \quad Q^{\bpi'}_k(x, u) > Q^{\bpi'}_k(x, \pi_k(x)).
	$$
	By \textit{modified} Corollary \ref{cor: implication-action-loop-2} with (\ref{cor-18-premise-modif}), this implies that $\pi'_k(x) \neq \pi_k(x)$ which contradicts $\bpi' = \bpi$. Thus, supposition is false and (\ref{eq1: to-local-SPE-policy}) must hold.
\end{proof}


\subsection{Chapter Summary}
Through this section, we have introduced the BPI as a new class of policy iteration algorithms to learn SPE policy under finite-horizon TIC objective specified in Section \ref{sec: prelim}. We obtained monotonicity results for general state-action case that circumvent the use of PIT. In Section \ref{sec:discretesa}, we dealt with discrete state-action and proved the correctness of BPI in converging to a (global) SPE policy by policy-based analysis that is, by permuting over all possible policies in the discrete search space $\Pi^{MD}$. In Section \ref{sec:contsa}, we generalize this result to continuous state-action, where we turn to value-based analysis as permutative argument no longer applies. This necessitates further specification of BPI's \textit{action-specs}, which we exemplified through through cases: (i) full-sweep argmax, and (ii) local-sweep argmax. In each case, we proved convergence to (global/local) SPE policy.

Next, we highlight several defining rules of BPI that have played a central role in our analyses in general and across different \textit{action-specs}. 

\paragraph{$\bpi'$-based PolEva-specs.} This rule captures the game-theoretic nature of BPI and mainly distinguishes BPI from standard RL, in which $\bpi$-based PolEva-specs is used. In particular, it contributes in establishing lex-monotonicity as a sufficient condition through Proposition \ref{prop: for FIN 3.3-2}. Intuitively, lex-monotonicity guarantees that if at indexes in ${}{}_{k^*+1}\mathcal{T}$, the distance to the (global/local) SPE-policy contracts, then at the lex-index $k^*$ it must also contract (by \textit{strictly-improving action-specs}). Note that by definition, the lex-index $k^*$ is determined with $\bpi'$. For instance, in full-sweep argmax, the lex-index $k^*$ for the old-new policy pair $(\bpi^{(0)}, \bpi^{(1)})$ is $0$ and \textit{not} $T-1$. This guarantee is related to the speed of convergence, which in our case consists of two components: (i) how fast is the lex-index $k^*$ moving to $0$, and (ii) how long does it take to finish updating for one particular $k^*$. The faster the update \textit{at} $k^*$, the more advantageous is this rule over $\bpi$-based PolEva. For instance, in the extreme case when full-sweep argmax is used, BPI converges in just two iterations while $\bpi$-based PolEva can only be guaranteed to converge in $T$ iterations. Referring to Theorem \ref{thm: asymp-conv-case-II}, BPI attains $0$-limit at some finite iteration $i^*_0 = i^*_1 = \dotso = i^*_{T-1} = 1$. In contrast, $\bpi$-based PolEva will only reflect a \textit{current} iteration's changes in future policies in the \textit{next} iteration, i.e. $i^*_{t} = i^*_{t+1} + 1, \forall t < T-1$.

\paragraph{\textit{Strictly-improving} action-specs.} This rule imposes \textit{strict} SPE-improving update in each iteration, preventing stagnancy unless it is \textit{SPE-optimal} as described by Corollary \ref{cor: implication-action-loop-2}. This rule is especially important in characterizing  \textit{non-termination} as \textit{strict} lex-monotonicity i.e. $\mathcal{B}' >_{lex} \mathcal{B}$. This allows the use of $\mathcal{B}' = \mathcal{B}$ to characterize \textit{termination} which serves as the basis of value-based analysis in Section 3.2.2. Across different \textit{action-specs}, this rule is tightly connected to the \textit{non-termination}-condition. To illustrate, we may revisit the local-sweep argmax case, where the \textit{strictly-improving} rule alone is insufficient to establish \textit{strict} lex-monotonicity unless accompanied with a corresponding modification of \textit{non-termination}-condition.

\paragraph{Consistent tie-break.} This rule prevents each player $t$'s oscillation of policies when the values $Q^{{}{}_{t+1}\bpi'}_t(x,\pi_t(x))$ are equal and contributes to lex-monotonicity as sufficient conditions in the form of Corollary \ref{cor: implication-consistent-tie-break} and Lemma \ref{lem: BWD tie-breaks}. The need of such rule in SPERL is motivated by the dependence of each players' \textit{SPE-optimality} on the choice of \textit{other} players. Consider the case when at some iteration $i\in\mathbb{N}$, SPE policy $\bpi^{(i)}$ has been found such that by BPI, it remains to go through one more iteration to reach \textit{termination}, i.e. $\bpi^{(i+1)} = \bpi^{(i)}$. Now, suppose that there is a player $t+1$ which from its perspective, the action $u'$ and $u$ have the same values. Without \textit{consistent tie-break} rule, $t+1$ may shift his action choice, i.e. $\pi^{(i)}_{t+1}(x) = u$ and $\pi^{(i+1)}_{t+1}(x) = u'$, which by the \textit{adjustment terms} in (\ref{PE-Q-true-recursion}),
\begin{align}
	{}{}_{t+1}\bpi^{(i+1)} \sim_{eq} {}{}_{t+1}\bpi^{(i)} \not\Rightarrow Q^{{}{}_{t+1}\bpi^{(i+1)}}_{t}(x, u) = Q^{{}{}_{t+1}\bpi^{(i)}}_{t}(x, u), \quad \forall x \in \mathcal{X}_t, u \in \mathcal{U}_t. \label{eq1: consistent-tie-break-importance}
\end{align}
Once we have non-equality (i.e. the conclusion of (\ref{eq1: consistent-tie-break-importance})), such shift will break the \textit{SPE-optimality} of $\bpi^{(i)}$ and will cause \textit{earlier} players in ${}{}^t\mathcal{T}$ to re-adjust to a different SPE-policy. This process can then repeat itself causing the algorithm to never \textit{terminate}. Moreover, by noting that re-adjustment may not happen at once, e.g., local-sweep argmax, force-terminating the algorithm may lead to a non-SPE-optimal policy. In contrast, standard RL approaches do not usually impose such rule since the dependence of $t$-agent's evaluation $Q^{\bpi}_t(x, u)$ to the players in ${}{}_{t+1}\mathcal{T}$ in standard RL problems is fully encoded by $Q^{\bpi}_{t+1}(\cdot, \pi_{t+1}(\cdot)) = \max_u Q^{\bpi}_{t+1}(\cdot, u)$ without \textit{adjustment terms}. Thus, as long as the action choices are (locally/globally) argmax, $Q^{\bpi}_t(x, u)$ is invariant to the choice of $\pi_{t+1}$ and (\ref{eq1: consistent-tie-break-importance}) will never happen. Finally, we note that in SPERL, if we can ensure a \textit{unique} solution to any argmax operation, (\ref{eq1: consistent-tie-break-importance}) can also be prevented; for instance, in local-sweep argmax, when $\lambda$ is sufficiently small or in full-sweep argmax, or when we have \textit{unique} global SPE policy.

\begin{remark}[Performance of the converged SPE policy]
	As illustrated in the paragraph above, each action choice $u'$ matters to which SPE (if non-unique) a search algorithm will converge to. And while BPI's consistent tie-break rule is supported by game-theoretic arguments, in reality, different choices of $u'$ may affect the actual control performance.
	\label{remark: choice-of-u'}
\end{remark}
One drawback of the algorithms covered in this section is the assumed full-sweep over the state-spaces $\{\mathcal{X}_t: t \in \mathcal{T}\}$ that is unrealistic in practice. In the next section, we will propose several SPERL training algorithms that relax such an assumption.
	
	\section{Training Algorithms} \label{sec: training-algo}
	In this section, we will focus on relaxing the full-sweep assumptions on the state-spaces $\{\mathcal{X}_t: t \in \mathcal{T}\}$ by incorporating standard RL simulation methods into BPI. We consider three types of methods, namely (i) tabular Q-learning, (ii) Q-learning with function approximators, and (iii) gradient-based methods. For each method, we will first set up new prediction objectives that build on BPI's PolEva step with particular attention drawn to the training of \textit{adjustment functions}. Then, we specify how to adapt BPI's key rules, specifically $\bpi'$-based PolEva-specs and \textit{consistent tie-break}, while noting that \textit{strictly-improving action-specs} automatically applies by the default setup.


\subsection{Tabular Q-learning}
Here, we derive a SPERL version of the standard finite-horizon Q-learning presented in \cite{Harada1997}. Consider a SPERL agent that consists of $\mathcal{T}$ child agents and define tabular representations $\hat{Q}_t, \hat{f}_t, \hat{g}_t, \hat{r}_t$ for each agent $t$ i.e.
\begin{alignat}{3}
	\Hat{Q}_t(x, u) &\approx && \; Q^{{}{}_{t+1}\bpi'}_t(x, u),
	\label{Q-MC-target}\\
	\Hat{r}_t(x, u, \tau, m, y) &\approx  && \; r^{{}{}_{t+1}\bpi', \tau, m, y}_t(x, u),
	\label{r-MC-target} \\
	\Hat{f}_t(x, u, \tau, y) &\approx && \; f^{{}{}_{t+1}\bpi', \tau, y}_t(x, u),
	\label{f-MC-target}\\
	\Hat{g}_t(x, u) &\approx && \; g^{{}{}_{t+1}\bpi'}_t(x, u). \label{g-MC-target}
\end{alignat}
The superscript $\bpi'$ denotes the SPERL agent's policy obtained after applying BPI with $\bpi$. We can then apply the TIC-TD PolEva derived in Section 3.1 and obtain a bootstrapped version of the DP targets defined in (\ref{r-DP-target})-(\ref{Q-DP-target}) as follows
\begin{align}
	\xi^r_t(x, u, \tau, m, y) \doteq
	\begin{cases}
		\mathcal{R}_{T-1,T-1}(y,X_T,u), &\text{if $m=\tau=t=T-1$},\\
		\mathcal{R}_{t, t}(y, x, u), &\text{if $m=\tau=t, \forall t<T-1$},\\
		\Hat{r}_{t+1}(X_{t+1}, \pi'_{t+1}(X_{t+1}), \tau, m, y), \, &\text{if $m\not=t$, $\forall t<T-1$},
	\end{cases} \label{r-TD-target}
\end{align}
\begin{align}
	\xi^f_t(x, u, \tau, y) \doteq
	\begin{cases}
		\mathcal{F}_{\tau}(y, X_T), &\text{if $t = T-1$},\\
		\Hat{f}_{t+1}(X_{t+1}, \pi'_{t+1}(X_{t+1}), \tau, y), \, &\text{otherwise},
	\end{cases} \label{f-TD-target}
\end{align}
\begin{align}    
	\xi^g_t(x, u) \doteq
	\begin{cases}
		X_T, &\text{if $t = T-1$},\\
		\Hat{g}_{t+1}(X_{t+1}, \pi'_{t+1}(X_{t+1})), \, &\text{otherwise},
	\end{cases} \label{g-TD-target}
\end{align}
\begin{align}
	\xi^Q_t(x, u) \doteq 
	\begin{cases}
		\hat{r}_t(x, u, t, t, x) + \hat{f}_t(x, u, t, x) + \mathcal{G}_{t}(x, \hat{g}_t(x, u)), &\text{if $t = T - 1$},\\
		\hat{r}_t(x, u, t, t, x) + \Hat{Q}_{t+1}(X_{t+1}, \pi'_{t+1}(X_{t+1})) - (\Delta \Hat{r}_t + \Delta \Hat{f}_t + \Delta \Hat{g}_t), \, &\text{otherwise},
	\end{cases} \label{Q-TD-target}
\end{align}
where
\begin{align}
	\Delta \Hat{r}_t &\doteq \sum^{T-1}_{m = t+1} \left(\Hat{r}_{t+1}(X_{t+1}, \pi'_{t+1}(X_{t+1}), t+1, m, X_{t+1}) - \Hat{r}_t(x, u, t, m, x) \right), \label{r-adjust-TD}\\
	\Delta \Hat{f}_t &\doteq \Hat{f}_{t+1}(X_{t+1}, \pi'_{t+1}(X_{t+1}), t+1, X_{t+1}) - \Hat{f}_t(x, u, t, x), \label{f-adjust-TD}\\
	\Delta \Hat{g}_t &\doteq \mathcal{G}_{t+1}(X_{t+1}, \hat{g}_{t+1}(X_{t+1}, \pi'_{t+1}(X_{t+1}))) - \mathcal{G}_t(x, \hat{g}_t(x, u)). \label{g-adjust-TD}
\end{align}
Finally, we follow a generalized policy iteration to perform BPI update i.e. $\forall t,x$,
\begin{align}
	\pi'_t(x) \gets \argmax_{u \in \mathcal{U}_t} \hat{Q}_t(x, u), \text{ (with \textit{consistent tie-break}) }\label{GPI-greedy-improvement}
\end{align}
where $\hat{Q}_t(x, u)$ is used in place of the unknown $Q^{\bpi'}_t(x, u)$. Supposing the use of on-policy training, we note some similarities between (\ref{GPI-greedy-improvement}) and the local-sweep argmax (\ref{action-specs-case-II}) in that for any fixed $t,x$, the values of $\hat{Q}_t(x,u)$ can only be accurate on the actions visited in the set of simulated trajectories which are analogous to $\mathcal{N}(\pi_t(x), \lambda)$. We further note that the \textit{consistent tie-break} rule is imposed explicitly in (\ref{GPI-greedy-improvement}). We summarize the discussion into SPERL Q-learning algorithm above; see Algorithm \ref{alg: SPERL Q-learning} in Appendix \ref{app:general}.

\begin{remark}[Sampling for $\tau, m, y$.]
	To make our approach more scalable, in Algorithm \ref{alg: SPERL Q-learning} in Appendix \ref{app:general}, we identify which $\tau, m, y$ are relevant to the prediction of $\hat{Q}_t(x, u)$ for a \textit{fixed} $t, x, u$. For instance, consider the parameters $\tau, y$ in $\hat{f}_t(x, u; \tau, y)$. Referring to (\ref{Q-TD-target}), we want our estimated $\hat{f}_t(x, u; \tau, y)$ to be accurate at $\tau = t, y = x$. By TIC-adjusted TD-based PolEva-specs for $\hat{f}_t$ prediction, we then need accurate estimates of $\hat{f}_{k}(X_{k}, U_{k}; \tau = t, y = x)$ for $k \geq t+1$. By inverting this observation fixing the $k$ instead, we can then derive the importance region $\tau \leq k - 1$ and correspondingly $y \in \cup_{\tau \leq k - 1} \mathcal{X}_{\tau}$. \label{remark: tabular-input-space-extension}
\end{remark}


\subsection{Q-learning with Function Approximation} \label{sec: q-learning w/ FA}
This subsection focuses on addressing the drawback of tabular representations that are usually limited to small, discrete state-action spaces by adapting the use of function approximators. Here, we adapt the steps used by \cite{Sutton2018} in extending the standard (infinite-horizon) Q-learning (see \cite{Watkins1992}) to handle infinite-dimensional state-action spaces. Consider $\w$-parameterized approximators $Q^{\w}_t, f^{\w}_t, r^{\w}_t, g^{\w}_t$
and set each agent $t$'s prediction objective analogous to Bellman-error minimization i.e. minimizing $J(\varphi_t^{\w}) \doteq \Vert \varphi^{\bpi'}_t(\cdot) - \hat{\varphi}_t(\cdot; \text{w}(t; \varphi)) \Vert_{\mathcal{D}_t^{\varphi}}$ for $\varphi \in \{Q, f, r, g\}$, over the parameter space $\text{W}^{\varphi} \subset \mathbb{R}^{d^{\varphi}}$ where $\text{w}(t; \varphi)$ takes values on. The weighted-norm $\Vert \, \cdot \Vert_{\mathcal{D}^{\varphi}_t}$ is defined on the input space of each $\varphi$ such that $\mathcal{D}^Q_t =\mathcal{D}^g_t \doteq \mathcal{X}_t \times \mathcal{U}_t, \mathcal{D}^f_t \doteq \mathcal{X}_t \times \mathcal{U}_t \times {}{}_t\mathcal{T} \times \mathcal{Y}_t,$ and $\mathcal{D}^r_t \doteq \mathcal{X}_t \times \mathcal{U}_t \times {}{}_t\mathcal{T} \times \mathcal{M}_t \times \mathcal{Y}_t$ where with a little abuse of notation, $\mathcal{X}_t, \mathcal{U}_t, {}{}_t\mathcal{T}, \mathcal{M}_t, \mathcal{Y}_t$ now represents arbitrary density functions defined on the full spaces $\mathcal{X}, \mathcal{U}, \mathcal{T}, \mathcal{T}, \mathcal{X}$ as a measure of approximation quality.

Intuitively, due to $\text{dim}(\text{w}(t; \varphi))$ being much smaller than the dimension of the $\varphi$'s full input space, we want our approximation to be accurate at some important regions at the sacrifice of irrelevant regions. Since $J(Q_t^{\w})$ is an analog of what we have in standard RL, we focus instead on the \textit{adjustment functions}' prediction objective, specifically in dealing with the $\tau, m, y$ in $\hat{f}$ and $\hat{r}$. As in Remark \ref{remark: tabular-input-space-extension}, we can first attempt to set $\mathcal{Y}_t \doteq \rho(\cup_{0 \leq \tau \leq t}\mathcal{X}_{\tau})$, ${}{}_t\mathcal{T} \doteq \rho(\{0, \dotso, t\})$, and $\mathcal{M}_t \doteq \rho(\{t, \dotso, T-1\})$ for some density function $\rho(\cdot)$ that measures the relative importance of any points/regions in the full input space in approximating $\hat{\varphi}_t$ or solving $\text{w}^*(t; \varphi)$ for $\varphi \in \{f, r\}$. However, such aggregation may seem unnatural in some cases; for instance, setting $\rho(\cdot)$ for ${}{}_t\mathcal{T}$ and $\mathcal{M}_t$ to be uniform is a natural choice under such aggregation but that is essentially saying that each element $\tau$ or $m$ contributes uniformly to $\text{w}^*(t; f)$ or $\text{w}^*(t; r)$. 

To address this issue, we introduce weight tables $\text{w}(t, \tau; f)$ and $\text{w}(t, \tau, m; r)$ and modify our approximators such that $\hat{f}_t(x,u,y; \text{w}(t, \tau; f)) \approx f^{\bpi'}_t(x,u,\tau,y)$ and $\hat{r}_t(x,u,y; \text{w}(t, \tau, m; r)) \approx r^{\bpi'}_t(x,u,\tau, m, y)$. According to this setup, our TD-based prediction objectives are as follows:
\begin{align*}
	J(f_t^{\w}) &\doteq \Vert \mathbb{E}[\xi^{f}_t(\cdot, \cdot, \tau, \cdot; \bpi')] - \hat{f}(\cdot; \text{w}(t, \tau; f)) \Vert_{\mathcal{D}^{f}_{t, \tau}} \\
	J(r_t^{\w}) &\doteq \Vert \mathbb{E}[\xi^{r}_t(\cdot, \cdot, \tau, m, \cdot; \bpi')] - \hat{r}(\cdot; \text{w}(t, \tau, m; r)) \Vert_{\mathcal{D}^{r}_{t, \tau}} \\
	J(g_t^{\w}) &\doteq \Vert \mathbb{E}[\xi^{g}_t(\cdot, \cdot; \bpi')] - \hat{g}(\cdot; \text{w}(t; g)) \Vert_{\mathcal{D}^{g}_{t, \tau}} \\
	J(Q_t^{\w}) &\doteq \Vert \mathbb{E}[\xi^{Q}_t(\cdot, \cdot; \bpi')] - \hat{Q}(\cdot; \text{w}(t; Q)) \Vert_{\mathcal{D}^{Q}_{t, \tau}}
\end{align*}
where $\mathcal{D}^Q_{t, \tau} = \mathcal{D}^g_{t, \tau} \doteq \mathcal{X}_t \times \mathcal{U}_t$ and $\mathcal{D}^f_{t, \tau} = \mathcal{D}^r_{t, \tau} \doteq \mathcal{X}_t \times \mathcal{U}_t \times \mathcal{X}_{\tau}$.

The prediction objectives abovecan then be solved using any least-squares solver. Once we have our approximate action-value function $Q^{\w}$, we can then apply BPI's PolImp rules. In the case of discrete action spaces, these rules can be specified similarly as in (\ref{GPI-greedy-improvement}). With continuous action spaces, our choice of approximator needs to be restricted to ensure feasible computation of local-optima. This is possible, e.g., when model-based approximators are available or when domain knowledge allows the identification of twice-differentiable linear features that are amenable to direct argmax solving. These restrictions are however undesirable as they restrict the addressable class of problems. Therefore, in the next subsection, we present gradient-based methods that are applicable to broader problem settings. 


\subsection{Gradient-based Methods: Deterministic Policy Gradient and Actor-Critic} \label{sec: gradient-based methods}
Gradient-based methods are common tools in standard RL to deal with continuous action spaces, which aim to train a parameterized policy separate from the action-value estimates. Here, we will derive a SPERL version of deterministic\footnote{The choice to present deterministic instead of a stochastic version is made to avoid much deviation from the control-theoretic definition of SPE policy in Section \ref{sec: prelim}.} policy gradient along the line of \cite{Silver2014}. We consider a finite-horizon $\boldsymbol\theta$-parameterized policy $\bpi^{\boldsymbol\theta}$ where we assume separate policy representation $\hat{\pi}_t(x ; \theta(t))$ for each agent $t$. Such a separation is consistent with $\bpi'$-based PolEva-specs in BPI, where each agent $t$ is only allowed to vary its policy $\pi_t$ while assuming fixed \textit{future players'} policies at ${}{}_{t+1}\bpi'$. For each agent $t$, we incorporate BPI's PolImp by applying simple chain rules to $\nabla_{\theta(t)} Q^{{}{}_{t+1}\bpi'}_t(x, \pi_t(x; \theta(t)))$ and obtain the corresponding deterministic gradient-ascent rule, i.e.
\begin{align}
	\theta^{l+1}(t) &=  \theta^{l}(t) + \alpha \nabla_{\theta(t)} Q^{\bpi^{l+1}}_t(x_t, \hat{\pi}_t(x_t; \theta^{l}(t))) \nonumber \\
	&= \theta^l(t) + \alpha \nabla_{\theta}\hat{\pi}_t(x_t; \theta)\vert_{\theta = \theta^l(t)} \nabla_u Q^{\bpi^{l+1}}_t(x_t,u)\vert_{u = \hat{\pi}_t(x_t; \theta^l(t))}. \label{action-value-grad-ascent-4}
\end{align}
We note the similarities of (\ref{action-value-grad-ascent-4}) to local-sweep argmax rule (\ref{action-specs-case-II}), except for here, optimization is done over $\Theta_t$ instead of $\mathcal{U}_t$ and $\lambda \downarrow 0$. We can also observe that while consistent tie-break rule does not explicitly appear anywhere in (\ref{action-value-grad-ascent-4}), it is implicitly imposed by letting $\lambda \downarrow 0$. For the gradient-ascent rule (\ref{action-value-grad-ascent-4}) to be implementable in practice, the \textit{true} action-value gradient $\nabla_u Q^{{}{}_{t+1}\bpi'}_t(x, u)$, given continuation policy ${}{}_{t+1}\bpi' \doteq \{\hat{\pi}_t(\cdot; \theta'(t)): \forall t \geq t+1\}$, must be approximated. We follow \cite{Silver2014} to instead approximate $Q^{{}{}_{t+1}\bpi'}_t(x, u)$ to which the results from Section \ref{sec: q-learning w/ FA} can be applied and SPERL Deterministic Actor-Critic algorithm can be obtained; see Algorithms \ref{alg: SPERL Deterministic Actor-Critic}--\ref{alg: update-Q} in Appendix \ref{app:general}, where we also discuss about the choice of critic approximator, the critic parameter update, and the use of replay buffer.


\subsection{Chapter Summary}
In this section, we have adapted standard RL simulation methods into BPI, addressing the main drawback of the version presented in Section \ref{sec: SPERL-foundation}. Two SPERL training algorithms were derived for two different model assumptions, discrete and continuous state-action spaces. The adaptation of TIC-adjusted TD-based methods to evaluate policy and some training procedures were discussed. We emphasize that the key to adapting BPI's rules is to realize the $\bpi'$-based PolEva-specs. In all three subsections, we have demonstrated that once we have a prediction framework, such rule can be integrated seamlessly into all methods we consider by simply imposing a \textit{backward} policy update direction; see for instance, Algorithm \ref{alg: SPERL Q-learning} in Appendix \ref{app:general}. While a thorough investigation on the training algorithms is not the focus of this paper, we exemplify our insights into the training under the SPERL framework with a financial example in the next section.
	
	\section{An Illustrative Example: Dynamic Mean-Variance Portfolio Selection} \label{sec: illustrative-example}
	This section focuses on illustrating an end-to-end derivation of training algorithm under the SPERL framework with an application of dynamic mean-variance (MV) portfolio selection. 


\subsection{Problem Formulation}
We consider a portfolio management problem with a fixed investment horizon $T_{\text{inv}} < \infty$ that can be discretized into $T > 0$ decision periods in $\mathcal{T} \doteq \{0, 1, \dotso, T-1\}$. We denote by $\Delta t$ the timestep or the length of each period such that $T_{\text{inv}} = T \Delta t$. For simplicity, we assume a market environment consisting of one risky and one riskless asset. Given a standard one-dimensional Brownian motion $\{W_{t}: 0 \leq t \leq T\}$, our risky asset price follows 
\begin{align} \label{eq:stockpricemodel}
	S_{t+\Delta t} - S_t = S_t (\mu \Delta t + \sigma \sqrt{\Delta W_t}), \quad \forall t \in \mathcal{T}
\end{align}
with $S_0 = s_0 > 0$, $\mu \in \mathbb{R}$, and $\sigma > 0$ denoting the initial price at $t = 0$, annualized mean return, and annualized stock volatility, respectively. The riskless asset has a constant annualized interest rate $r_{\text{ann.}}>0$.

An agent's state $X^u_t \in \mathbb{R}$ defines her wealth at time $t$ and agent's action $u_t \in \mathbb{R}$ signifies how much wealth she puts into the risky asset with the remaining wealth $X^u_t-u_t$ being invested into riskless asset. Given the above market environment model, the wealth process can then be described by the following stationary linear dynamics
\begin{align}
	X^{u}_{t+1} = (1 + r) X^u_t + u_t (Y_{t+1} - r), \quad \forall t \in \mathcal{T} \label{dyn: MV-const-risk}
\end{align}
with normalized wealth at time $0$, i.e. $X_0 = 1$, period rate $r = r_{\text{ann.}}\Delta t$, and $\{Y_t\}_{t \in \mathcal{T}}$ a sequence of i.i.d random variables with the following attributes
\begin{align}
	\mathbb{E}[Y_t] = \mu \Delta t, \quad \hbox{Var}(Y_t) = \sigma^2\Delta t, \quad \forall t \in \mathcal{T}. \label{unknowns: MV-const-risk}
\end{align}

The agent's objective is to select a dynamic portfolio $\bpi=\{\pi_0, \pi_1, \ldots, \pi_{T-1}\}=\{u_0, u_1, \ldots, u_{T-1}\}$ that strikes the best balance between the expected value (reward) of the terminal wealth $\mathcal{E}[X_T]$ and the variance (risk) of the terminal wealth $\hbox{Var}(X_T)$. Hence, the performance criterion at time $t\in \mathcal{T}$ takes the form
\begin{equation} \label{obj: MV-const-risk}
	V^{\bpi}_t(x)\doteq \mathbb{E}_{t,x}[X_T^{\bpi}]-\frac{\gamma}{2}\hbox{Var}_{t,x}(X_T^{\bpi}) = \mathbb{E}_{t, x} \left[X_T^{\bpi} - \frac{\gamma}{2}\left(X_T^{\bpi}\right)^2\right] + \frac{\gamma}{2}\left(\mathbb{E}_{t, x}[X_T^{\bpi}]\right)^2,
\end{equation}
which by the general form in (\ref{TIC-objective}), we have
\begin{align}
	\mathcal{G}_{\tau}(y, x) = \frac{\gamma}{2}x^2, \quad \mathcal{F}_{\tau}(y, x) = x - \frac{\gamma}{2}x^2 \label{TIC-sources: MV-const-risk}
\end{align}
implying the existence of $\mathcal{G}$-type of TIC. Since we have continuous state-action spaces, we will apply the SPERL Deterministic Actor-Critic algorithm proposed in Section \ref{sec: gradient-based methods} to train an SPE policy that solves (\ref{obj: MV-const-risk}). We remark here that in this example, our only unknowns are the transition model parameters in (\ref{unknowns: MV-const-risk}), which also defines our risky asset model \eqref{eq:stockpricemodel}.


\subsection{Model-based Function Approximators}
In the next subsection, we describe both policy and critic approximators that we use to train our algorithm.

\subsubsection{Critic Approximators}
To address the estimation of $Q^{\bpi'}_t(x, u)$ in the gradient-ascent rule (\ref{action-value-grad-ascent-4}), we derive model-based linear representations for both $\hat{Q}_t(x, u)$ and $\hat{g}_t(x, u)$. Referring to the boundary conditions at $t = T-1$, \footnotesize
\begin{align}
	Q^{\bpi}_{T-1}(x, u)
	&\doteq \mathbb{E}_{T-1,x}[X^u_T] - \frac{\gamma}{2} \hbox{Var}_{T-1,x}[X^u_T] \nonumber \\
	&= \mathbb{E}_{T-1,x}\left[ (1+r)x + (Y_T - r) u\right] - \frac{\gamma}{2} \hbox{Var}_{T-1,x}\left[ (1+r)x + (Y_T - r) u \right] \tag*{(by (\ref{dyn: MV-const-risk}))} \\
	&= (1+r)x + (\mathbb{E}[Y_T] - r) u - \frac{\gamma}{2} \hbox{Var}[Y_T] u^2 \label{q-ansatz: T-1}\\
	g^{\bpi}_{T-1}(x, u)
	&\doteq \mathbb{E}_{T-1,x}[X^u_T]=\mathbb{E}_{T-1,x}\left[ (1+r)x + (Y_T - r) u\right] \tag*{(by (\ref{dyn: MV-const-risk}))} \\
	&= (1+r)x + (\mathbb{E}[Y_T] - r) u \label{g-ansatz: T-1}
\end{align} \normalsize
and noting the linear-quadratic setting of this example, for any arbitrary policy $\bpi$, $Q^{\bpi}$ and $g^{\bpi}$ will have the following form
\begin{align}
	Q^{\bpi}_t(x, u) &= A_t u^2 + B_t u + C_t x + D_t \label{Q-ansatz: MV-const-risk}\\
	g^{\bpi}_t(x, u) &= a_t u + b_t x + c_t \label{g-ansatz: MV-const-risk}
\end{align}
We can thus set our critic approximators according to (\ref{Q-ansatz: MV-const-risk})-(\ref{g-ansatz: MV-const-risk}) i.e.
\begin{align}
	\hat{Q}_t(x, u; \text{w}(t)) &\doteq \text{w}_3(t; Q) u^2 + \text{w}_2(t; Q) u + \text{w}_{1}(t; Q) x + \text{w}_{0}(t; Q) \label{Q-lin-FA: MV-const-risk} \\
	\hat{g}_t(x, u; \text{w}(t)) &\doteq \text{w}_2(t; g) u + \text{w}_{1}(t; g) x + \text{w}_{0}(t; g) \label{g-lin-FA: MV-const-risk}
\end{align}

\subsubsection{Policy Approximators}
The obtained forms in (\ref{Q-ansatz: MV-const-risk})-(\ref{g-ansatz: MV-const-risk}) can further give us clues to set up model-based policy approximators; in particular, we observe that the action-value gradient $\nabla_u Q^{\bpi}_t(x, u)$ is independent of the state $x$, i.e. 
\begin{align}
	\nabla_u Q^{\bpi'}_t(x_t, u) = \nabla_u Q^{\bpi'}_t(\tilde{x}_t, u), \forall x_t, \tilde{x}_t \in \mathcal{X}_t, x_t \neq \tilde{x}_t, \label{eq2: state-invariant-policy}
\end{align}
which means that no matter what state agent $t$ is in, any state $x_t$ will give the same signal about what direction of improvement (towards the optimal action) to take. Such indifference to $x_t$ can then be exploited to set state-invariant policy approximators, i.e. $\forall t \in \mathcal{T}$,
\begin{align}
	\hat{\pi}_t(x) \doteq \theta(t),\quad \forall x \in \mathcal{X}_t \label{eq1: state-invariant-policy}
\end{align}


\subsection{Training Procedures}
In this subsection, we specify in detail how we train the approximators above as outlined in Algorithm \ref{alg: SPERL-MV} below.

\small
\begin{algorithm}
	\SetAlgoLined
	\SetKwInOut{Input}{Input}
	\SetKwInOut{Output}{Output}
	\Input{MarketEnv($\mu, \sigma, r, x_0, \Delta t, T, \gamma$), Hyperparameters($L, B, \lambda, \kappa, \alpha_{\text{w}}, \alpha_{\theta}, {\color{red}\dotso}$)}
	\Output{Approximate SPE-policy $\bpi^{\boldsymbol\theta}$}
	Initialize critic parameters $\w$, actor parameters $\boldsymbol\theta$, replay memory $\mathcal{D} \leftarrow \emptyset$\;
	
	\For{$l \gets 0$ \KwTo $L$}{
		\For{$b \gets 1$ \KwTo $B$}{ \label{alg:mv-startgentracj}
			Set $X_0 \gets 1$\;
			Generate trajectory $X_0, U_0, X_1, U_1, \dotso, X_{T-1}, U_{T-1}, X_{T} \sim \bpi^{(l)}_{\lambda\text{-unif}}$\;
			\For{$t \gets 0$ \KwTo $T-1$}{
				\For{$\tau \gets t$ \KwTo $0$}{
					$\mathcal{D} \gets$ $\mathcal{D}\ \cup \left\{\left(t, \tau, X_t, U_t, X_{\tau}, X_{t+1}\right)\right\}$
				}
			}
		}	\label{alg:mv-endgentracj}
		\For{$t \gets T-1$ \KwTo $0$}{
			\eIf{$t = T-1$}{
				Initialize $\Xi^g_{\cdot, \cdot}, \Xi^Q_{\cdot, \cdot} \gets \emptyset$\;
				Sample mini-batch $\tilde{\mathcal{D}}_{\cdot, \cdot} \sim$ Replay$\left(\cdot, \cdot, \mathcal{D}, \kappa \right)$ \; \label{alg:mv-minibatch}
				\For{$(t, \tau, x, u, y, X^{x, u}) \in \tilde{\mathcal{D}}_{\cdot, \cdot}$}{
					Transform $(x, u, X^{x,u})$ to $(1, u, X^{1, u})$ \; \label{alg:mv-starttransform}
					$\xi_t^g \gets X^{1,u}$\;
					Set $\Xi^g_{\cdot, \cdot} \gets \Xi^g_{\cdot, \cdot} \cup (\cdot, 1, u, \xi^g_t)$\; \label{alg:mv-endtransform}
				}
				Solve $\text{w}^* \gets \argmin_{\text{w}} \sum_{\Xi^g_{\cdot, \cdot}} \left(\xi_t^g - \hat{g}_t(x, u; \text{w})\right)^2$ (with ALS)\; \label{alg:mv-argming}
				Update $\text{w}_{2}(t; g) \gets \text{w}_{2}(t; g) + \alpha_{\text{w}} \left( \text{w}^*_2 - \text{w}_2(t; g)\right)$ (with EMA) \; \label{alg:mv-updatew2g}
				Update $\text{w}_1(t; g) \gets (1+r)$ \; \label{alg:mv-updatew1g}
				
				\For{$(t, \tau, x, u, y, X^{x, u}) \in \tilde{\mathcal{D}}_{\cdot, \cdot}$}{
					Transform $(x, u, X^{x,u})$ to $(1, u, X^{1, u})$ \;
					$\xi_t^Q \gets X^{1,u} - \frac{\gamma}{2} (X^{1,u})^2 + \frac{\gamma}{2} \hat{g}^2_t(1, u; \text{w}(t; g))$\;
					Set $\Xi^Q_{\cdot, \cdot} \gets \Xi^Q_{\cdot, \cdot} \cup (\cdot, 1, u, \xi^Q_t)$\;
				}
				
				Solve $\text{w}^* \gets \argmin_{\text{w}} \sum_{\Xi^Q_{\cdot, \cdot}} \left(\xi_t^Q - \hat{Q}_t(x, u; \text{w})\right)^2$ (with ALS)\; \label{alg:mv-argminQ}
				
				Update $\text{w}_{2,3}(t; Q) \gets \text{w}_{2,3}(t; Q) + \alpha_{\text{w}} \left( \text{w}^*_{2,3} - \text{w}_{2,3}(t; Q)\right)$  (with EMA)\; \label{alg:mv-updatew2Q}
				
				Update $\text{w}_1(t; Q) \gets (1+r)$\; \label{alg:mv-updatew1Q}
				
			}{
				Update $\text{w}_{1,2,3}(t; g), \text{w}_{1,2,3}(t; Q)$ following (\ref{eq1-param-recursion: MV-const-risk})-(\ref{eq2-param-recursion: MV-const-risk}) \; \label{alg:mv-updatewg}
			}
			$\theta(t) \gets \theta(t) + \alpha_{\theta}  \nabla_u \hat{Q}_t(1, u; \text{w}(t; Q))|_{u = \theta(t)} $\;
		}
	}
	\caption{SPERL Dynamic MV Portfolio Selection}
	\label{alg: SPERL-MV}
\end{algorithm}
\normalsize

\subsubsection{Trajectory Generation}
We refer to lines \ref{alg:mv-startgentracj}--\ref{alg:mv-endgentracj} in Algorithm \ref{alg: SPERL-MV} as experience collection step. Experiences here are collected at every iteration $l$ in the form of wealth trajectories of length $T$ with initial state $X_0$ normalized to $1$. At each time $t$ and given the corresponding wealth $X_t$, we sample the next state $X_{t+1}$ from a MarketEnv simulator under a uniform exploratory policy $\pi^{(l)}_{t, \lambda\text{-unif}}$ that samples action $U_t \sim \text{Unif}(\pi_t(X_t; \theta^{(l)}(t)) - \lambda, \pi_t(X_t; \theta^{(l)}(t)) + \lambda)$. We note that such exploration schedule is possible since our training is offline (i.e. the environment that we interact with is a simulator and \textit{not} the real market). All $B$ generated trajectories are then stored into the replay buffer $\mathcal{D}$ in a tupled form as specified in Section \ref{sec: gradient-based methods}.

\subsubsection{Critic Training} \label{sec: critic-training}
At $t = T-1$, we follow the steps in Algorithm \ref{alg: SPERL Deterministic Actor-Critic} in Appendix \ref{app:general} to solve the prediction problems
\begin{align*}
	\hat{Q}_{T-1}(x, u; \text{w}(T-1))
	&\approx Q^{\bpi'}_{T-1}(x, u) \doteq \mathbb{E}_{T-1,x}[X^u_T] - \frac{\gamma}{2} \hbox{Var}_{T-1,x}[X^u_T],\\
	\hat{g}_{T-1}(x, u; \text{w}(T-1))
	&\approx g^{\bpi'}_{T-1}(x, u) \doteq \mathbb{E}_{T-1,x}[X^u_T].
\end{align*}
Note that since we only have two unknown model parameters by (\ref{unknowns: MV-const-risk}), some critic parameters can be fixed by (\ref{q-ansatz: T-1})-(\ref{g-ansatz: T-1}) at
\begin{align}
	\text{w}_0(T-1; Q) =  \text{w}_0(T-1; g) &= 0, \cr
	\text{w}_1(T-1; Q) =  \text{w}_1(T-1; g) &= (1+r) \label{eq2-knowns: MV-const-risk}
\end{align}
to let the critic training focus on the remaining unknown parameters,
\begin{align}
	\text{w}_2(T-1; g) &\approx (\mathbb{E}[Y_T] - r) = \mu \Delta t - r, \label{eq1: critic-training} \\
	\text{w}_2(T-1; Q) &\approx (\mathbb{E}[Y_T] - r) = \mu \Delta t - r, \label{eq2: critic-training}\\
	\text{w}_3(T-1; Q) &\approx  - \frac{\gamma}{2} \hbox{Var}[Y_T] = -\frac{\gamma}{2} \sigma^2\Delta t. \label{eq3: critic-training}
\end{align}
\paragraph{Parametric Recursions} 
Moreover, by noting that the approximation scheme for (\ref{eq1: critic-training})-(\ref{eq3: critic-training}) has learnt all our unknowns, we can exploit the same model knowledge and assumptions to perform parametric recursions. This technique reduces the problem of estimating $\{Q^{\bpi'}_t: t \in \mathcal{T}\}$ to only $Q^{\bpi'}_{T-1}$ by converting the remaining $T-1$ estimations to simple computation problems.

Let us first recover some statistics about our state dynamics from the estimated critics in (\ref{eq1: critic-training})-(\ref{eq3: critic-training}) that is to be used in the subsequent parametric recursion derivation,
\begin{align}
	\hbox{Var}_{T-1,x}[X^u_T] &\approx \frac{2}{\gamma}( \hat{g}_{T-1}(x, u; \text{w}(T-1)) - \hat{Q}_{T-1}(x, u; \text{w}(T-1))), \label{var-infer: MV-const-risk}\\
	\mathbb{E}_{T-1,x}[X^u_T] &\approx \hat{g}_{T-1}(x, u; \text{w}(T-1)). \label{mu-infer: MV-const-risk}
\end{align}
Next, we rewrite the Q-recursion from Proposition \ref{prop: policy-dependent-Q-recursion}, after keeping only the $g$-term by MV TIC-source specifications with $Q^{\bpi'}$ replaced by their approximators $\hat{Q}$ as follows \footnotesize
\begin{alignat}{3}
	\hat{Q}_t(x, u; \text{w}'(t)) &= &&\; \mathbb{E}_{t,x}\left[\hat{Q}_{t+1}(X^u_{t+1}, \theta'(t+1); \text{w}'(t+1))\right] - \frac{\gamma}{2} \hbox{Var}_{t,x}\left[\hat{g}_{t+1}(X^u_{t+1}, \theta'(t+1); \text{w}'(t+1))\right] \nonumber \\
	&= &&\; \mathbb{E}_{t,x}\left[\text{w}'_3(t+1; Q) \theta'(t+1)^2 + \text{w}'_2(t+1;Q) \theta(t+1) + \text{w}'_1(t+1; Q) X^u_{t+1} + \text{w}'_0(t+1; Q)\right] \nonumber \\
	& &&- \frac{\gamma}{2} \hbox{Var}_{t,x}\left[ \text{w}'_2(t+1; g) \theta'(t+1) + \text{w}'_1(t+1; g) X^u_{t+1} + \text{w}'_0(t+1; g) \right] \tag*{(by (\ref{Q-ansatz: MV-const-risk})-(\ref{g-ansatz: MV-const-risk}))} \\
	&= && \; \text{w}'_3(t+1; Q) \theta'(t+1)^2 + \text{w}'_2(t+1;Q) \theta'(t+1) + \text{w}'_1(t+1; Q) \mathbb{E}_{t,x} \left[X^u_{t+1}\right] + \text{w}'_0(t+1; Q) \nonumber \\
	& &&- \frac{\gamma}{2} (\text{w}'_1(t+1; g))^2 \hbox{Var}_{t,x}\left[ X^u_{t+1} \right] \tag*{(by deterministic coefficients)} \\
	&= && \; \text{w}'_3(t+1; Q) \theta'(t+1)^2 + \text{w}'_2(t+1;Q) \theta'(t+1) + \text{w}'_1(t+1; Q) \mathbb{E}_{T-1,x} \left[X^u_{T}\right] + \text{w}'_0(t+1; Q) \nonumber \\
	& &&- \frac{\gamma}{2} \text{w}'_1(t+1; g)^2 \hbox{Var}_{T-1,x}\left[ X^u_T \right] \tag*{(by stationary transitions (\ref{dyn: MV-const-risk}))} \\
	&= && \; \left( \text{w}'_1(t+1; Q) - \text{w}'_1(t+1; g)^2 \right) \hat{g}_{T-1}(x, u; \text{w}'(T-1)) + \text{w}'_1(t+1; g)^2 \hat{Q}_{T-1}(x, u; \text{w}'(T-1)) \nonumber \\
	& && + \text{w}'_3(t+1; Q) \theta'(t+1)^2 + \text{w}'_2(t+1;Q) \theta'(t+1) + \text{w}'_0(t+1; Q). \tag*{(by (\ref{var-infer: MV-const-risk})-(\ref{mu-infer: MV-const-risk}))}
\end{alignat}\normalsize
Applying similar steps as the above, we obtain the following \footnotesize
\begin{align}
	\hat{g}_t(x, u; \text{w}'(t)) &= \mathbb{E}_{t,x}\left[ \hat{g}_{t+1}(X^u_{t+1}, \theta'(t+1); \text{w}'(t+1)) \right] \nonumber \\
	&= \mathbb{E}_{t,x}\left[ \text{w}'_2(t+1; g) \theta'(t+1) + \text{w}'_1(t+1;g) X^u_{t+1} + \text{w}'_0(t+1; g) \right] \tag*{(by (\ref{g-ansatz: MV-const-risk}))} \\
	&= \text{w}'_2(t+1; g) \theta'(t+1) + \text{w}'_1(t+1;g)  \mathbb{E}_{t,x}\left[X^u_{t+1}\right] + \text{w}'_0(t+1; g)  \tag*{(by deterministic coefficients)} \\
	&= \text{w}'_2(t+1; g) \theta'(t+1) + \text{w}'_1(t+1;g)  \mathbb{E}_{T-1,x}\left[X^u_T\right] + \text{w}'_0(t+1; g) \tag*{(by stationary transitions (\ref{dyn: MV-const-risk}))} \\
	&= \text{w}'_2(t+1; g) \theta'(t+1) + \text{w}'_0(t+1; g) + \text{w}'_1(t+1;g)  \hat{g}_{T-1}(x, u; \text{w}'(T-1)). \tag*{(by (\ref{mu-infer: MV-const-risk}))}
\end{align} \normalsize
By matching coefficients on the LHS and RHS in the last line in each parametric recursion derivation, we obtain a formula\footnote{The parameters $\text{w}'_0(t; g), \text{w}'_0(t; Q)$ are irrelevant to the value of $\nabla_u \hat{Q}_t(x, u)$ and have thus been omitted.} to replace lines \ref{alg:dac-startupdatew}--\ref{alg:dac-endupdatew} in Algorithm \ref{alg: SPERL Deterministic Actor-Critic} for $t < T-1$, 
\begin{alignat}{3}
	\text{w}_1(t; Q) &\gets &&\; ( \text{w}_1(t+1; Q) - \text{w}^2_1(t+1; g) ) \text{w}_1(T-1; g) + \text{w}^2_1(t+1; g) \text{w}_1(T-1; Q), \label{eq1-param-recursion: MV-const-risk}\\
	\text{w}_2(t; Q) &\gets &&\; \text{w}_1(t+1; Q) \text{w}_2(T-1; g) + \text{w}^2_1(t+1; g) \left( \text{w}_2(T-1; Q) - \text{w}_2(T-1; g) \right),\\
	\text{w}_3(t; Q) &\gets &&\; \text{w}^2_1(t+1; g) \text{w}_3(T-1; Q),\\
	\text{w}_1(t; g) &\gets &&\; \text{w}_1(t+1;g) \text{w}_1(T-1; g),\\
	\text{w}_2(t; g) &\gets &&\; \text{w}_1(t+1;g) \text{w}_2(T-1; g). \label{eq2-param-recursion: MV-const-risk}
\end{alignat}
To see the use of the formulas above, please refer to line \ref{alg:mv-updatewg} in Algorithm \ref{alg: SPERL-MV}.

\subsubsection{Actor Training} \label{sec: actor-training}
To train our policy parameters $\{\theta(t): t \in \mathcal{T}\}$, we will adopt the gradient-ascent rule in (\ref{action-value-grad-ascent-4}). By substituting the state-invariant approximator in (\ref{eq1: state-invariant-policy}) and replacing the true $Q^{\bpi'}_t(x, u)$ with its current estimate $\hat{Q}_t(x, u)$, we obtain
\begin{align}
	\theta(t) \gets \theta(t) + \alpha_{\theta} \sum_{x \in \mathcal{\tilde{D}}_{t, \cdot}} \nabla_u \hat{Q}_t(x, u; \text{w}'(t)) \vert_{u = \theta(t)}, \quad \forall t \in \mathcal{T}. \label{GA-rule1: MV-const-risk}
\end{align}

\paragraph{State-invariant Policy} By exploiting the state-invariance property in (\ref{eq2: state-invariant-policy}), we can simplify the rule (\ref{GA-rule1: MV-const-risk}) while improving the accuracy of the update direction.

Let us revisit the rule before substitution of critic estimates,
\begin{align}
	\theta(t) \gets \theta(t) + \alpha_{\theta} \sum_{x \in \mathcal{\tilde{D}}_{t, \cdot}} \nabla_u Q^{\bpi'}_t(x, u) \vert_{u = \theta(t)}, \quad \forall t \in \mathcal{T}. \label{GA-rule: MV-const-risk}
\end{align}
By the independence of action-value gradient to $x$, we have
$$
\sum_{x \in \mathcal{\tilde{D}}_{t, \cdot}} \nabla_u Q^{\bpi'}_t(x, u) \propto \nabla_u Q^{\bpi'}_t(x, u),
$$
which allows us to arbitrarily choose any $x \in \mathcal{\tilde{D}}_{t, \cdot}$ and substitute the latter into (\ref{GA-rule: MV-const-risk}). However, such an arbitrary substitution may no longer apply when $\hat{Q}_t(x, u)$ is used in place of $Q^{\bpi'}_t(x, u)$ as the deterministic actor-critic prescribes due to the possible discrepancy in the accuracy of $\hat{Q}_t(x, u)$ at different $x \in \mathcal{\tilde{D}}_{t, \cdot}$, making the choice of $x$ matters. We deal with this issue by focusing our critic estimation to one particular $x$ that we simply set to $1$.

In what follows, we discuss how focusing critic approximation to $x = 1$ necessitates modifications to Algorithms \ref{alg: SPERL Deterministic Actor-Critic}-\ref{alg: update-Q} or the methods detailed in Section \ref{sec: critic-training}.
\begin{itemize}
	\item At $t = T-1$, our modification mainly happens inside Algorithms \ref{alg: update-g}-\ref{alg: update-Q} concerning how to avoid ``throwing away" the samples collected with $x \neq 1$ in estimating $\hat{Q}_t(1, u)$. This can be done by transforming each collected experience tuple $(x, u, X^{x,u})$ to $(1, u, X^{1,u})$ according to (\ref{dyn: MV-const-risk}), i.e. $X^{1,u} \doteq X^{x,u} - (1+r) (x - 1)$. We then adjust our TD-targets $\xi_t^g, \xi_t^Q$ by substituting any $x$ with $1$. This discussion is summarized into lines \ref{alg:mv-starttransform}--\ref{alg:mv-endtransform} in Algorithm \ref{alg: SPERL-MV}. Moreover, once we no longer care about $x$, keeping an estimate importance distribution of $\mathcal{X}_{T-1}$ is no longer necessary. This warrants the use of all 3-tuples $(x, u, X^{x, u})$ from any period $t$ in estimating $\hat{Q}_{T-1}(1, u)$ that we indicate by dropping the subscripts $t$ from all mini-batches notation; for instance, compare between line \ref{alg:mv-minibatch} in Algorithm \ref{alg: SPERL-MV} and line \ref{alg:updateQ-minibatch} in Algorithm \ref{alg: update-Q}.
	
	\item Next, still at $t = T-1$, we record some changes in the number of trainable parameters for both $\hat{Q}_{T-1}(1, u)$ and $\hat{g}_{T-1}(1,u)$ as we collapse the coefficients of $x$, i.e. $\text{w}_1(T-1; g)$ and $\text{w}_1(T-1; Q)$ into intercepts; see lines \ref{alg:mv-updatew2g} and \ref{alg:mv-updatew2Q} in Algorithm \ref{alg: SPERL-MV}. We can then disentangle the merged parameters $\text{w}_0(T-1; \cdot)$ and $\text{w}_1(T-1; \cdot)$ by applying (\ref{eq2-knowns: MV-const-risk}) noting that $r$ is a known parameter; see lines \ref{alg:mv-updatew1g} and \ref{alg:mv-updatew1Q} in Algorithm \ref{alg: SPERL-MV}.
	
	\item At $t < T-1$, our trainable parameters stay the same: $\text{w}_1(t; Q), \text{w}_2(t; Q), \text{w}_3(t; Q)$ for $Q$ and $\text{w}_1(t; g), \text{w}_2(t; g)$ for $g$. Since we have recovered all the necessary parameters at $t = T-1$ to perform parametric recursion, no modifications to the derived formula (\ref{eq1-param-recursion: MV-const-risk})-(\ref{eq2-param-recursion: MV-const-risk}) are required; see line \ref{alg:mv-updatewg} in Algorithm \ref{alg: SPERL-MV}. 
\end{itemize}

\subsubsection{Improving Training Stability}
In this subsection, we group the training components in Algorithm \ref{alg: SPERL-MV} that deal particularly with in-training stability issues.

\paragraph{Replay Specifications.} Here, we specify a replay technique to regulate the mini-batch sampling in line \ref{alg:mv-minibatch} of Algorithm \ref{alg: SPERL-MV} that will be used in solving the least-squares problems in lines \ref{alg:mv-argming} and \ref{alg:mv-argminQ}. At each iteration $l$, we separate \textit{current} experiences (referring to the new batch generated by lines \ref{alg:mv-startgentracj}--\ref{alg:mv-endgentracj}) from \textit{past} experiences. We include all \textit{current} experiences into the mini-batch $\tilde{\mathcal{D}}^{(l)}_{\cdot, \cdot}$ and then, sample randomly without replacement from \textit{past} experiences in $\kappa:1$ proportion to the size of the \textit{current} experiences. Hyperparameter involved in this replay technique is the \textit{resample constant} $\kappa$.

\paragraph{Least-squares Solver.} To solve the argmin functions in lines \ref{alg:mv-argming} and \ref{alg:mv-argminQ}, we will use a type of regression that has been modified to account for the special attributes of noise model (\ref{dyn: MV-const-risk}) that breach the assumption of residuals independence to input variables in ordinary least squares (OLS), causing severe instability issues in critic parameter estimation.

For each $\varphi \in \{g, Q\}$, we consider the corresponding OLS regression model for $\varphi_t(x, u; \text{w})$
\begin{align}
	\xi^{\varphi} = \boldsymbol\phi^{\varphi} \cdot \text{w}^{\varphi} + e^{\varphi} \label{OLS-model}
\end{align}
with $\xi^{\varphi}, \boldsymbol\phi^{\varphi},$ and $e^{\varphi}$ representing \textit{target} variable, \textit{input} variables, and residuals, respectively. To illustrate the aforementioned noise attributes, we focus on $\varphi = g$, where $\xi^g (x, u) = \tilde{\mathbb{E}}_{T-1,x}[X^u_T]$. Note that in the above, we have re-defined the \textit{target} variable definitions from the \textit{one-sample} estimator $X^{x,u}$ to \textit{mini-batch} estimator $\tilde{\mathbb{E}}_{T-1,x}[X^u_T]$ to reflect the actual implementation of line \ref{alg:mv-argming} in Algorithm \ref{alg: SPERL-MV}. We then compute the following
\begin{align}
	(e^g)^2(x, u) &= \hbox{Var}[\xi^g(x, u)]=\hbox{Var}[\tilde{\mathbb{E}}_{T-1,x}[X^u_T]]=\hbox{Var}[\tilde{\mathbb{E}}[(1+r)x + u(Y - r)]] \nonumber\\
	&=\hbox{Var}[u\tilde{\mathbb{E}}[(Y - r)]]=u^2 \frac{\sigma^2 \Delta W_t}{N_{x,u}}, \label{model-based-res-end}
\end{align}
where $N_{x,u}$ represents the number of sampled tuples $(x, u, X^{x,u})$ used in estimating $\text{w}^g$. Referring back to (\ref{g-lin-FA: MV-const-risk}), we have $\boldsymbol\phi^g = (x, u, 1)$ and thus, the \textit{homoscedasticity} requirement on the residuals $e^g$ is only met when the number of samples in the mini-batch $N_{x,u} \approx \infty$; this is unrealistic in practice.

To mitigate the aforementioned heteroscedasticity's effect on training stability, we propose an adaptive correction to our OLS regression model, namely \textit{adaptive least squares (ALS)}, by performing the following steps; see \cite{Sterchi2017} for empirical evidence.
\begin{enumerate}
	\item For each $\varphi \in \{g, Q\}$, rewrite the original OLS model in (\ref{OLS-model}) as
	\begin{align}
		\xi^{\varphi}_{\text{ols}} = \boldsymbol\phi^{\varphi}_{\text{ols}} \cdot \text{w}^{\varphi}_t + e^{\varphi}_{\text{ols}} \label{OLS-model-2}
	\end{align}
	and denote by $\hat{\xi}^{\varphi}_{\text{ols}}$ the fitted solution.
	\item Derive model-based features $\boldsymbol\phi^{e}$ for the OLS squared residuals $(e^{\varphi}_{\text{ols}})^2$ as exemplified in (\ref{model-based-res-end}); thus, $\boldsymbol\phi^{e^{g}} = (u^2)$ and $\boldsymbol\phi^{e^{Q}} = (u^2, u^4)$. 
	\item Perform OLS regression on the \textit{target-input} variables $((\xi^{\varphi}_{\text{ols}} - \hat{\xi}^{\varphi}_{\text{ols}})^2, \boldsymbol\phi^{e^{\varphi}})$ without fitting any intercepts and denote by $(\hat{e}^{\varphi}_{\text{ols}})^2$ the fitted residual values. As we may get $0$ or negative fitted values due to some noisy estimates, we proceed by keeping only the \textit{positive} fitted values.
	\item Transform the original \textit{target-input} variable in (\ref{OLS-model-2}) by
	\begin{align*}
		(\xi^{\varphi}_{\text{als}}, \boldsymbol\phi^{\varphi}_{\text{als}}) \gets \left(\frac{\xi^{\varphi}_{\text{ols}}}{\sqrt{(\hat{e}^{\varphi}_{\text{ols}})^2}}, \frac{\boldsymbol\phi^{\varphi}_{\text{ols}}}{\sqrt{(\hat{e}^{\varphi}_{\text{ols}})^2}}\right).
	\end{align*}
	\item Perform OLS regression with the \textit{transformed} \textit{target-input} variables $ (\xi^{\varphi}_{\text{als}}, \boldsymbol\phi^{\varphi}_{\text{als}})$ without fitting any intercepts.
\end{enumerate}

\paragraph{Smoothing Regularization.} Finally, to tame the variance\footnote{This technique of slowing the update of parameters is commonly used in standard RL with function approximation to ensure TD-error remains small across iterations; see \cite{Fujimoto2018} for instance.} of (mini-batch) critic estimation at each iteration $l$, we apply exponential moving average (EMA) by setting the critic learning rate $\alpha^{(l)}_{\text{w}} = 2/(l+1)$; see lines \ref{alg:mv-updatew2g} and \ref{alg:mv-updatew2Q} in Algorithm \ref{alg: SPERL-MV}.


\subsection{Experiments}
In this subsection, we perform simulation study, where we deploy our algorithm in two different types of MarketEnv with annualized mean $\mu \in \{20\%, -20\%\}$, annualized volatility $\sigma = 30\%$, and annualized risk-free rate $r_{\text{ann.}} = 2\%$. We normalize initial wealth $x_0$ to $1$, set the investment horizon $T_{\text{inv.}}$ to $1$ year with timestep $\Delta t = 1/100$, and fix the mean-variance criterion parameter $\gamma = 1.2$. In each MarketEnv, we evaluate our algorithm by its financial performance and learning curves of both critic parameters $\w$ and policy parameter $\boldsymbol\theta$.

\subsubsection{Training Setup}
For both experiments, we set the total training episodes $L = 5000$, trajectory generation size $B = 5$, exploratory policy parameter $\lambda = 1.5$, and resample constant $\kappa = 1$. We note that such setup of $B$ and $\kappa$ then implies a mini-batch size $|\tilde{\mathcal{D}}_{\cdot, \cdot}| = 1000$ after appending \textit{past} experiences and including samples from all time periods $t < T-1$ as specified in Section \ref{sec: actor-training}. We initialize our critic parameters $\w$ near the true analytical parameters and actor parameters $\boldsymbol\theta$ to $0$. We fix the learning rate for actor parameter update $\alpha_{\theta} = 2$ and use EMA learning rate $\alpha^{(l)}_{\text{w}} = 2/(l + 1)$ for our critic parameter update. 

\subsubsection{Results and Discussions}
\paragraph{Financial Performance} For evaluation purpose, at each iteration $l$, a new price trajectory (different from the one used in training our actor-critic parameters) is generated from MarketEnv. A non-randomized policy $\bpi^{(l)}$ is then used to generate a wealth trajectory from which the terminal wealth $X^{(l)}_T$ is recorded. In Figures \ref{fig: finperf-pos} and \ref{fig: finperf-neg}, we plot the learning curves of sample mean and sample standard deviation (stdev) of terminal wealth $X_T$ that are computed by aggregating $X^{(l)}_T$ over 50 non-overlapping episodes. From these two figures, we can observe that our algorithm converges in $\approx 20$ aggregated episodes in both MarketEnv setups. Moreover, the mean and stdev of return at convergence, i.e. $(35\%, 50\%)$ in MarketEnv$(\mu = 20\%)$ and $(45\%, 60\%)$ in MarketEnv$(\mu = -20\%)$, are within a reasonable range of Sharpe ratio.

\begin{figure}[!ht]
	\centering
	\includegraphics[width=15.2cm, trim={1cm 0 1cm 1cm}, clip]{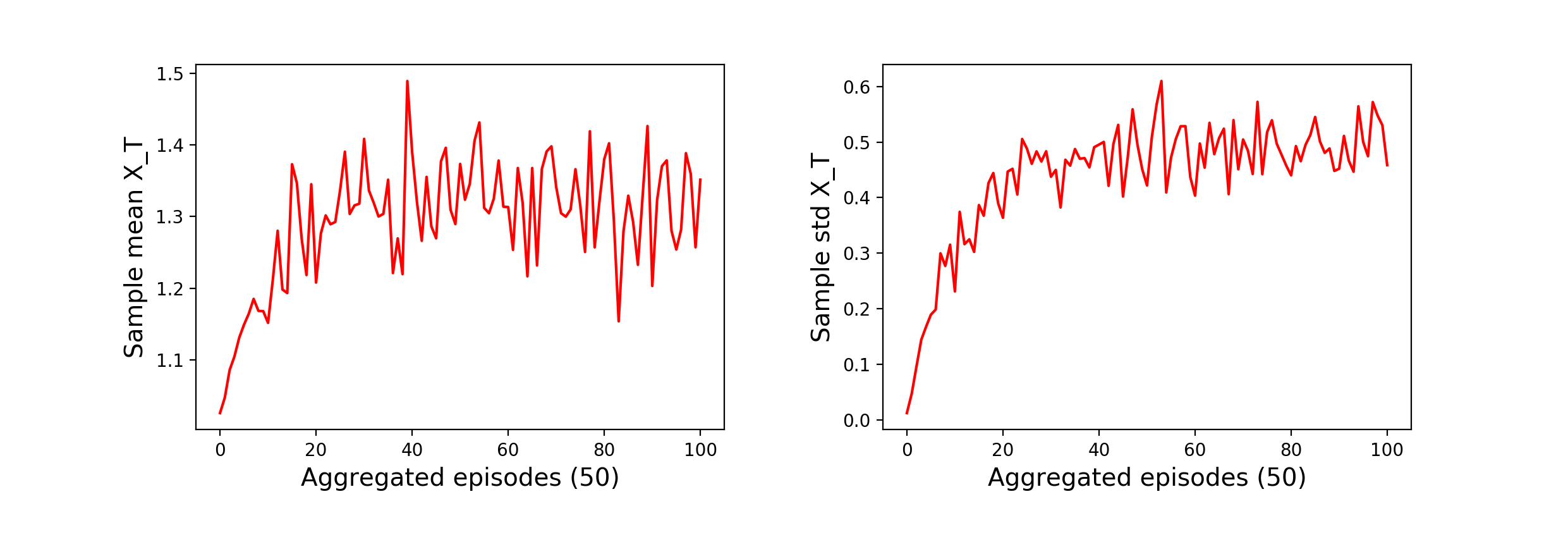}
	\caption{Sample mean and stdev of terminal wealth $(\mu = 20\%)$}
	\label{fig: finperf-pos}
\end{figure}
\begin{figure}[!ht]
	\centering
	\includegraphics[width=15.2cm, trim={1cm 0 1cm 1cm}, clip]{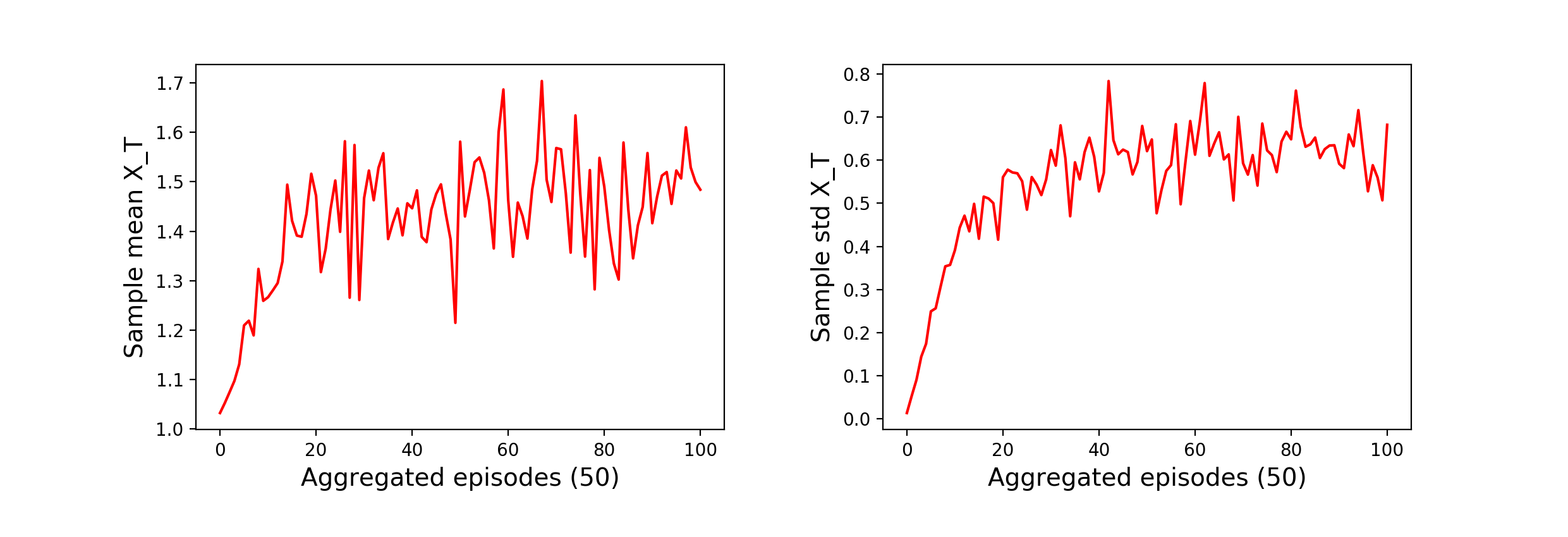}
	\caption{Sample mean and stdev of terminal wealth $(\mu = -20\%)$}
	\label{fig: finperf-neg}
\end{figure}

\paragraph{Parameter Learning Curves} In the interest of model parameter identification, we record the learning curves of critic parameters $\text{w}(T-1)$; see Figures \ref{fig: wout-pos} and \ref{fig: wout-neg}. First, we clarify that we only present the learning curves at $t = T-1$ because with the use of parametric recursion technique, one directly links the learning performance for earlier time periods to the last time period $t = T-1$. We compare the learning curve of our proposed algorithm (`EMA') to the ground truth (`TRUE') that is computed by substituting the MarketEnv parameters $\mu, \sigma$ to (\ref{eq1: critic-training})-(\ref{eq3: critic-training}). Similarly as in the terminal wealth curve, we observe that convergence happens in about 20 aggregated episodes. Moreover, to illustrate how our smoothing choice stabilizes the noisiness of $\text{w}(T-1)$ updates, we also include the parameter learning curve of a contending stabilization moving average technique over past 20 periods, MA($q = 20$), which clearly fall short of EMA.

\begin{figure}[!ht]
	\centering
	\includegraphics[width=15.2cm, trim={1cm 1cm 1cm 1cm}, clip]{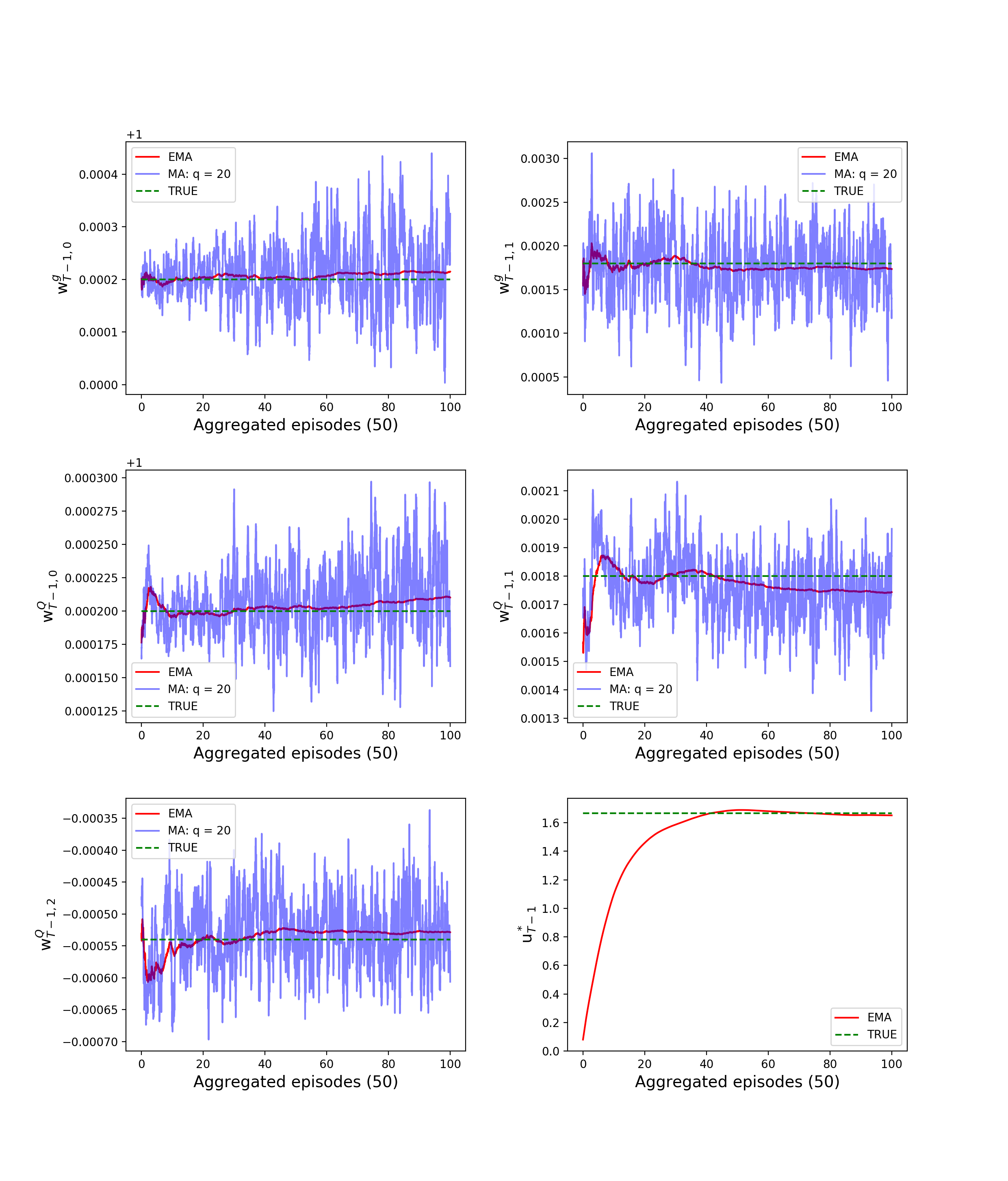}
	\caption{Critic and actor parameter learning curves at $t = T-1$ under MarketEnv$(\mu = 20\%)$}
	\label{fig: wout-pos}
\end{figure}
\begin{figure}[!ht]
	\centering
	\includegraphics[width=15.2cm, trim={1cm 1cm 1cm 1cm}, clip]{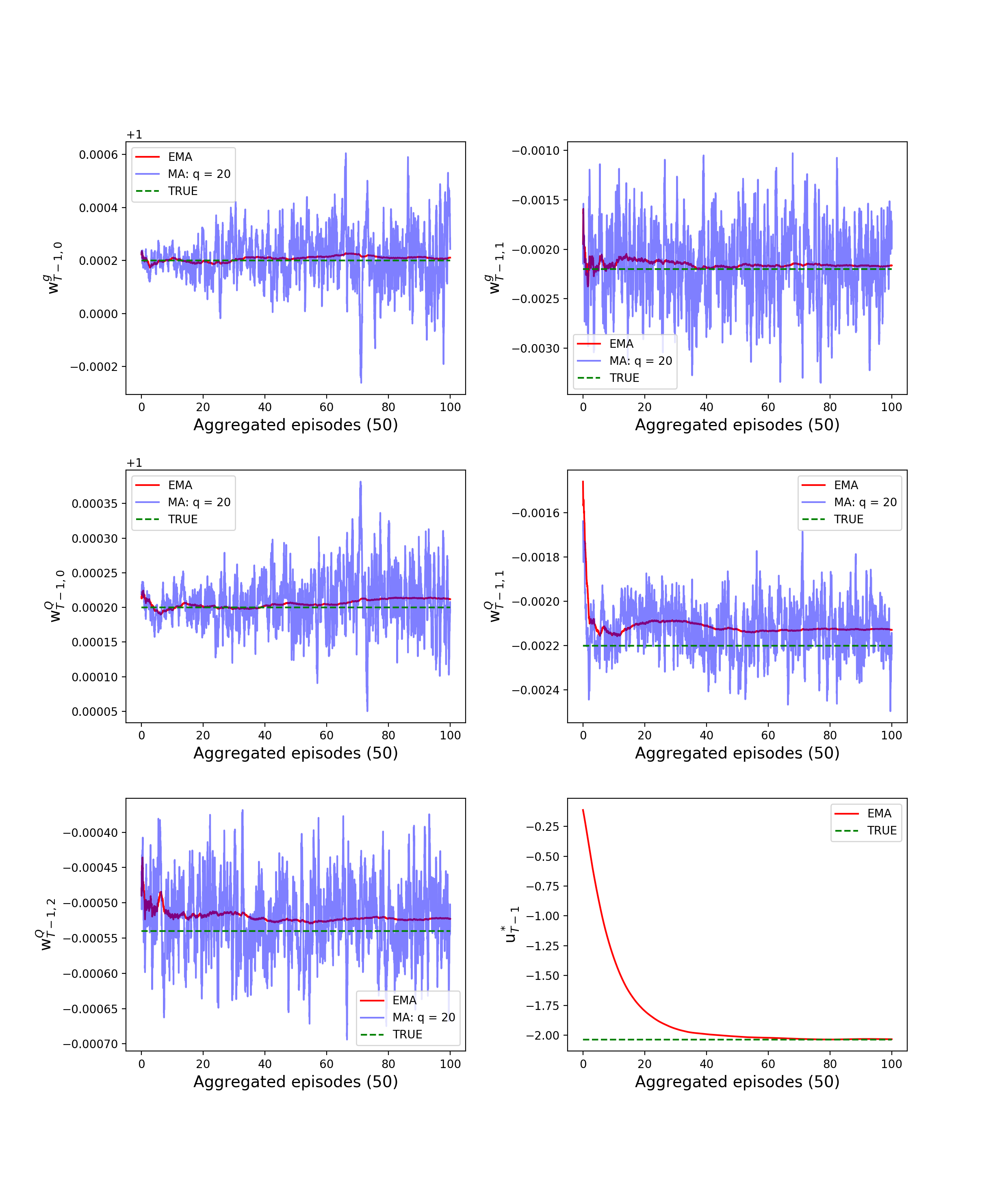}
	\caption{Critic and actor parameter learning curves at $t = T-1$ under MarketEnv$(\mu = -20\%)$}
	\label{fig: wout-neg}
\end{figure}

\paragraph{Optimal Strategy} Finally, in the last plot of both Figures \ref{fig: wout-pos} and \ref{fig: wout-neg}, we present the learning curves of policy $\pi_{T-1}(\cdot)$ that concurrently represents our actor parameter $\theta(T-1)$ by the state-invariant approximator \ref{eq2: state-invariant-policy}. We compare the policy at convergence with the ground truth (`TRUE') where as derived in \cite{Bjoerk2014}, 
\begin{align*}
	u^*_{t} = \frac{(\mu\Delta t - r)^2}{\gamma \sigma^2 \Delta t}(1+r)^{T - (t+1)}, \quad \forall t \in \mathcal{T}.  
\end{align*}
Thus, we conclude that the gradient-based update will converge to the ground truth in about 40 aggregated episodes. To further illustrate how this result translates to other time periods, we provide the learning curves of $u^*_t$ at $t = 0, \frac{T-1}{2}$ in Appendix \ref{app:mv}.


\subsection{Chapter Summary}
In this section, we applied SPERL deterministic actor-critic training framework and specified training procedures that suit the problem specifications. In particular, we have used model-based function approximators and perform model-based reductions to both actor and critic training problems. We note some connections to control-based approaches in the derivations of parametric recursion formulas (\ref{eq1-param-recursion: MV-const-risk})-(\ref{eq2-param-recursion: MV-const-risk}), that are analogous to the derivations of analytic equilibrium control except for our use of action-value function $Q$ in place of value function $V$. This is natural under the restrictive problem specifications that we made at the beginning, i.e. (\ref{unknowns: MV-const-risk}) being our only unknowns, which implies that all the remaining model assumptions are at our disposal. The usage of these assumptions can be observed at each step of converting $Q$-recursion to parametric recursion. While such reduction can help us achieve sample and training efficiency, using too many model assumptions may be undesirable in practice, especially in more complex domain, where model is generally unavailable. In such an event, we have no choice but to get stuck in the first step of conversion and use model-free training as presented in Section \ref{sec: training-algo}. This observation illustrates SPERL advantage over the current analytic equilibrium control approaches.
	
	\section{Conclusions and Future Works}\label{sec: conclusion-future}
	In this paper, we studied the search of SPE policy in finite-horizon TIC problems as a RL problem, which forms the proposed SPERL framework. By drawing insights from the extended DP theory, we proposes a new class of policy iteration algorithm, which we refer to as BPI, as a SPERL solver. We further conducted detailed analyses on BPI's update rules and correspondingly showed some desirable properties, such as update (lex-)monotonicity and convergence to SPE policy, which in turn address some existing challenges in TIC-RL domain. To demonstrate how BPI can be used in practice, we discussed several ways of pairing BPI with standard RL simulation methods, resulting in two main training frameworks: SPERL Q-learning and SPERL deterministic actor-critic. We then illustrate a full training algorithm derivation under the SPERL deterministic actor-critic framework through a mean-variance analysis example. The experimental results are plausible and show the efficiency of the proposed algorithms.

Noting that some training and implementation details are still left to generalities, promising future research directions are to investigate on these training matters, examining on more complex domain problems, and benchmarking with other TIC-RL algorithms, especially those that do not belong to either globally optimal or SPE policy class. Moreover, the learning algorithms in this paper provide a practical solution towards the search of SPE policy other than the analytical solution, especially when the latter is not available in practice or a model-free environment. Since TIC is considered as a key feature to better revealing human's preferences, it will be interesting to explore applications of this SPERL framework.
	
	\newpage
	
	\appendix
	\section{SPERL Training Algorithms for General TIC Problems} \label{app:general}
	\subsection{SPERL Q-learning}

\begin{algorithm}[H]
	\SetAlgoLined
	\SetKwInOut{Input}{Input}
	\SetKwInOut{Output}{Output}
	\SetKwInOut{Init}{Initialize}
	\SetKwInOut{Loop}{Loop}
	
	\Input{ Env, Hyperparameters$(\alpha, \epsilon)$}
	\Output{ Approximate SPE Q-function $\hat{Q}_t(x, u), \forall t,x,u$}
	\Init{ $\hat{Q}, \hat{r}, \hat{g}, \hat{f}$; $\bpi' \gets \emptyset$; $\pi_t(x) \gets \arg\max_u \hat{Q}(x,u), \forall t, x$}
	
	\While{$\bpi' \neq \bpi$}{
		Update $\bpi \gets \bpi'$\;
		Choose $X_0$ randomly\;
		Generate trajectory $X_0, U_0, X_1, U_1, \dotso, X_{T-1}, U_{T-1}, X_T \sim \bpi_{\epsilon\text{\textit{-greedy}}}$\; \label{alg:Q-gentraj}
		\For{$t \gets T-1$ \KwTo $0$}{
			\For{$\tau \in \{t, t-1, \dotso, 0\}$, $y \in \{X_t, X_{t-1}, \dotso, X_0\}$}{
				\For{$m \gets t$ \KwTo $T-1$}{
					Compute $\xi_t^r \gets \xi_t^r(X_t, U_t, \tau, m, y)$ by (\ref{r-TD-target})\; \label{alg:Q-computexir}
					$\hat{r}_t(X_t, U_t, \tau, m, y) \gets \hat{r}_t(X_t, U_t, \tau, m, y) + \alpha(\xi_t^r - \hat{r}_t(X_t, U_t, \tau, m, y))$\;
				}
				Compute $\xi_t^f \gets\xi_t^f(X_t, U_t, \tau, y)$ by (\ref{f-TD-target})\;
				$\hat{f}_t(X_t, U_t, \tau, y) \gets \hat{f}_t(X_t, U_t, \tau, y) + \alpha(\xi_t^f - \hat{f}_t(X_t, U_t, \tau, y))$\;
			}
			Compute $\xi_t^g \gets \xi_t^g(X_t, U_t)$ by (\ref{g-TD-target})\;
			$\hat{g}_t(X_t, U_t) \gets \hat{g}_t(X_t, U_t) + \alpha(\xi_t^g - \hat{g}_t(X_t, U_t))$\;
			Compute $\xi_t^Q \gets \xi_t^Q(X_t, U_t)$ by (\ref{Q-TD-target})\;
			$\hat{Q}_t(X_t, U_t) \gets \hat{Q}_t(X_t, U_t) + \alpha(\xi_t^Q - \hat{Q}_t(X_t, U_t))$\;
			Compute $u' \gets \arg\max_u \hat{Q}_t(X_t, u)$\;
			\eIf{$\hat{Q}_t(X_t, u') > Q_t(X_t, \pi_t(X_t))$}{
				$\pi'_t(X_t) \gets u'$}{
				$\pi'_t(X_t) \gets \pi_t(X_t)$\;}}}
	\caption{SPERL $Q$-learning}
	\label{alg: SPERL Q-learning}
\end{algorithm}
\paragraph{The case of random rewards.} In the presence of random rewards as shortly remarked in Remark \ref{remark: random-rewards}, we can sample $R_t$ and modify our \textit{trajectory generation} in line \ref{alg:Q-gentraj} to 
\begin{align*}
	X_0, U_0, R_1, X_1, U_1, \dotso, X_{T-1}, U_{T-1}, R_T, X_T.    
\end{align*}
Since these random rewards are the strict attributes of adjustment function $r$, it remains to modify its target computation in line \ref{alg:Q-computexir} to account for $R_t$ that is, by assigning 
\begin{align*}
	\xi_t^r \gets \mathcal{H}(\tau, y, R_t)
\end{align*}

\subsection{SPERL Deterministic Actor-Critic}
\begin{algorithm}[H]
	\SetAlgoLined
	\SetKwInOut{Input}{Input}
	\SetKwInOut{Output}{Output}
	\Input{Env, Hyperparameters($\alpha, \epsilon$)}
	\Output{Approximate SPE-policy $\bpi^{\boldsymbol\theta}$}
	Initialize critic parameters $\w$, actor parameters $\boldsymbol\theta$, replay memory $\mathcal{D} \leftarrow \emptyset$\;
	
	\For{$l \gets 0$ \KwTo $L$}{
		\For{$b \gets 1$ \KwTo $B$}{
			Choose $X_0$ randomly\;
			Generate trajectory $X_0, U_0, X_1, U_1, \dotso, X_{T-1}, U_{T-1}, X_{T} \sim \bpi^{\boldsymbol\theta}_{\epsilon\text{\textit{-greedy}}}$\;
			\For{$t \gets 0$ \KwTo $T-1$}{
				\For{$\tau \gets t$ \KwTo $0$}{
					$\mathcal{D} \gets$ $\mathcal{D}\ \cup \left\{\left(t, \tau, X_t, U_t, X_{\tau}, X_{t+1}\right)\right\}$
				}
			}
		}
		\For{$t \gets T-1$ \KwTo $0$}{
			\For{$\tau \gets t$ \KwTo $0$}{ \label{alg:dac-startupdatew}
				\For{$m \gets t$ \KwTo $T-1$}{
					$\text{w}(t, \tau, m; r) \gets \text{UPDATE-}r(\w, \alpha, \boldsymbol\theta, t, \tau, m, \mathcal{D})$\;}
				$\text{w}(t, \tau; f) \gets \text{UPDATE-}f(\w, \alpha, \boldsymbol\theta, t, \tau, \mathcal{D})$\;
			}
			$\text{w}(t; g) \gets \text{UPDATE-}g(\w, \alpha, \boldsymbol\theta, t, \mathcal{D})$\;
			$\text{w}(t; Q) \gets \text{UPDATE-}Q(\w, \alpha, \boldsymbol\theta, t, \mathcal{D})$\; \label{alg:dac-endupdatew}
			
			$\theta(t) \gets \theta(t) + \alpha_{\theta} \sum_{\tilde{\mathcal{D}}_{t, \cdot}} \left( \nabla_{\theta} \hat{\pi}_t(x; \theta)\vert_{\theta = \theta(t)} \nabla_u \hat{Q}(x, u; \text{w}(t; Q))|_{u = \hat{\pi}_t(x; \theta(t))} \right)$\;
		}
	}
	\caption{SPERL Deterministic Actor-Critic}
	\label{alg: SPERL Deterministic Actor-Critic}
\end{algorithm}

\begin{algorithm}
	\SetAlgoLined
	\SetKwInOut{Input}{Input}
	\SetKwInOut{Output}{Output}
	\Input{$\w, \alpha, \boldsymbol\theta, t, \tau, m, \mathcal{D}$}
	\Output{$\text{w}'(t, \tau, m; r)$}
	Initialize $\Xi_{t, \tau} \gets \emptyset$\;
	
	Sample mini-batch $\tilde{\mathcal{D}}_{t, \tau} \sim$ Replay$\left(t, \tau, \mathcal{D}\right)$\;
	\For{$(t, \tau, x, u, y, X^{x, u}) \in \tilde{\mathcal{D}}_{t, \tau}$}{
		\eIf{m = t}{
			$\xi^r_t \gets \mathcal{R}_{\tau, t}(y, x, u)$}{
			$\xi^r_t \gets \hat{r}_{t+1}(X^{x, u}, \hat{\pi}_{t+1}(X^{x, u}; \theta(t+1)), y; \text{w}(t+1, \tau, m; r))$}
		Set $\Xi_{t, \tau} \gets \Xi_{t, \tau} \cup (t, \tau, x, u, y, \xi^r_t)$\;
	}
	Solve $\text{w}^* \gets \argmin_{\text{w}} \sum_{\Xi_{t, \tau}} \left(\xi_t^r - \hat{r}_t(x, u, y; \text{w})\right)^2$\;
	$\text{w}'(t, \tau, m; r) \gets \text{w}(t, \tau, m; r) + \alpha \left( \text{w}^* - \text{w}(t, \tau, m; r)\right)$\;
	\caption{UPDATE-$r$}
	\label{alg: update-r}
\end{algorithm}

\begin{algorithm}
	\SetAlgoLined
	\SetKwInOut{Input}{Input}
	\SetKwInOut{Output}{Output}
	\Input{$\w, \alpha, \boldsymbol\theta, t, \tau, \mathcal{D}$}
	\Output{$\text{w}'(t, \tau; f)$}
	
	Initialize $\Xi_{t, \tau} \gets \emptyset$\;
	
	Sample mini-batch $\tilde{\mathcal{D}}_{t, \tau} \sim$ Replay$\left(t, \tau, \mathcal{D}\right)$\;
	\For{$(t, \tau, x, u, y, X^{x, u}) \in \tilde{\mathcal{D}}_{t, \tau}$}{
		\eIf{$t = T-1$}{
			$\xi^f_t \gets \mathcal{F}_{\tau}(y, X^{x,u})$}{
			$\xi^f_t \gets \hat{f}_{t+1}(X^{x,u}, \hat{\pi}_{t+1}(X^{x,u}; \theta(t+1)), y; \text{w}(t+1, \tau; f))$}
		Set $\Xi_{t, \tau} \gets \Xi_{t, \tau} \cup (t, \tau, x, u, y, \xi^f_t)$\;
	}
	Solve $\text{w}^* \gets \argmin_{\text{w}} \sum_{\Xi_{t, \tau}} \left(\xi_t^f - \hat{f}_t(x, u, y; \text{w})\right)^2$\;
	$\text{w}'(t, \tau; f) \gets \text{w}(t, \tau; f) + \alpha \left( \text{w}^* - \text{w}(t, \tau; f)\right)$\;
	\caption{UPDATE-$f$}
	\label{alg: update-f}
\end{algorithm}

\begin{algorithm}
	\SetAlgoLined
	\SetKwInOut{Input}{Input}
	\SetKwInOut{Output}{Output}
	\Input{$\w, \alpha, \boldsymbol\theta, t, \mathcal{D}$}
	\Output{$\text{w}'(t; g)$}
	
	Initialize $\Xi_{t, \cdot} \gets \emptyset$\;
	
	Sample mini-batch $\tilde{\mathcal{D}}_{t, \cdot} \sim$ Replay$\left(t, \cdot, \mathcal{D}\right)$\;
	\For{$(t, \tau, x, u, y, X^{x, u}) \in \tilde{\mathcal{D}}_{t, \cdot}$}{
		\eIf{$t = T-1$}{
			$\xi^g_t \gets X^{x,u}$}{
			$\xi^g_t \gets \hat{g}_{t+1}(X^{x,u}, \hat{\pi}_{t+1}(X^{x,u}; \theta(t+1)); \text{w}(t+1; g))$}
		Set $\Xi_{t, \cdot} \gets \Xi_{t, \cdot} \cup (t, x, u, \xi^g_t)$\;
	}
	Solve $\text{w}^* \gets \argmin_{\text{w}} \sum_{\Xi_{t, \cdot}} \left(\xi_t^g - \hat{g}_t(x, u; \text{w})\right)^2$\;
	$\text{w}'(t; g) \gets \text{w}(t; g) + \alpha \left( \text{w}^* - \text{w}(t; g)\right)$\;
	\caption{UPDATE-$g$}
	\label{alg: update-g}
\end{algorithm}

\begin{algorithm}
	\SetAlgoLined
	\SetKwInOut{Input}{Input}
	\SetKwInOut{Output}{Output}
	\Input{$\w, \alpha, \boldsymbol\theta, t, \mathcal{D}$}
	\Output{$\text{w}'(t ; Q)$}
	
	Initialize $\Xi_{t, \cdot} \gets \emptyset$\;
	Sample mini-batch $\tilde{\mathcal{D}}_{t, \cdot} \sim$ Replay$\left(t, \cdot, \mathcal{D}\right)$\; \label{alg:updateQ-minibatch}
	
	\For{$(t, \tau, x, u, y, X^{x, u}) \in \tilde{\mathcal{D}}_{t, \cdot}$}{
		\eIf{$t = T-1$}{
			$\xi^Q_t \gets \hat{r}_t(x, u, x; \text{w}(t,t,t; r)) + \hat{f}_t(x,u,x; \text{w}(t,t; f)) + \mathcal{G}_t(x, \hat{g}_t(x, u; \text{w}(t; g)))$}{
			
			Set $\Delta \hat{r}_t \gets 0$\;
			\For{$m \gets t+1$ \KwTo $T-1$}{
				$\Delta \hat{r}_t \gets \Delta \hat{r}_t + \hat{r}_{t+1}(X^{x,u}, \hat{\pi}_{t+1}(X^{x,u}; \theta(t+1)), X^{x,u}; \text{w}(t+1, t+1, m; r))$\\
				$\qquad \quad \, - \hat{r}_t(x, u, x; \text{w}(t, t, m; r))$\; }
			
			$\Delta \hat{f}_t \gets \hat{f}_{t+1}(X^{x,u}, \hat{\pi}_{t+1}(X^{x,u}; \theta(t+1)), X^{x,u}; \text{w}(t+1, t+1; f))$\\
			$\qquad \quad \, - \hat{f}_t(x, u, x; \text{w}(t,t; f))$\;
			
			$\Delta \hat{g}_t \gets \mathcal{G}_{t+1}(X^{x,u}, \hat{g}_{t+1}(X^{x,u}, \hat{\pi}_{t+1}(X^{x,u}; \theta(t+1)); w(t+1; g))$\\
			$\qquad \quad \, - \mathcal{G}_t(x, \hat{g}_t(x, u; w(t; g))$\;
			
			$\xi^Q_t \gets \hat{r}_t(x, u, x; \text{w}(t,t,t; r)) + \Hat{Q}_{t+1}(X^{x,u}, \hat{\pi}_{t+1}(X^{x,u}; \theta(t+1)); \text{w}(t+1; Q))$\\
			$\quad \quad \; \; - (\Delta \Hat{r}_t + \Delta \Hat{f}_t + \Delta \Hat{g}_t)$\;
		}
		
		Set $\Xi_{t, \cdot} \gets \Xi_{t, \cdot} \cup (t, x, u, \xi^Q_t)$\;
	}
	Solve $\text{w}^* \gets \argmin_{\text{w}} \sum_{\Xi_{t, \cdot}} \left(\xi_t^Q - \hat{Q}_t(x, u; \text{w})\right)^2$\; \label{alg:updateQargmin}
	
	$\text{w}'(t; Q) \gets \text{w}(t; Q) + \alpha_{\text{w}} \left( \text{w}^* - \text{w}(t; Q)\right)$\; \label{alg:updateQGD}
	\caption{UPDATE-$Q$}
	\label{alg: update-Q}
\end{algorithm}

We discuss in SPERL context several implementation essentials that are common in dealing with function approximators and critic estimation.

\paragraph{Choice of Critic Approximator.} 
In Algorithms \ref{alg: SPERL Deterministic Actor-Critic}--\ref{alg: update-f}, we have incorporated tabularized weight representations for $f, r$; see Section \ref{sec: q-learning w/ FA} for details. Beyond this, we do not restrict how input space should be segregated or aggregated. For instance, neural networks and linear approximators can both be used to represent $\hat{f}_t(x, u, y; \text{w})$ depending on how amenable are the prediction problems at hand.

\paragraph{Critic Parameter Update.} We illustrate how critic parameters $\w$ are updated by the last two lines in Algorithms \ref{alg: update-r}--\ref{alg: update-Q}. For instance, we refer to Algorithm \ref{alg: update-Q}. To solve the $\argmin$ function in line \ref{alg:updateQargmin}, any least-squares solvers, such as stochastic gradient descent, Batch gradient descent, or simple regression, can be used. Line \ref{alg:updateQGD} then provides some flexibility to incorporate smoothening of parameter updates, i.e. $\alpha_{\text{w}} < 1$.

\paragraph{Replay Buffer.}
In the case of noisy input-target pairs to be used in critic estimation, keeping a replay buffer can help stabilize training by replaying past experiences. In our example algorithms, experiences are collected in the form of tuple $(t, \tau, x, u, y, X^{x,u})$, where $X^{x, u}$ denotes the \textit{next state} encountered after hitting state $x$ and acting $u$. The notation $X^{x, u}$ marks our use of the stationary transition probability assumption in Section \ref{sec: finite-horizon-TIC}. To contrast with the state-action-reward-state action (SARSA) experiences collection in standard RL context, we need to collect additional information about $t$ for our finite-horizon model and $\tau, y = x_{\tau}$ for our adjustment functions $r, f$; see Algorithms \ref{alg: update-r} and \ref{alg: update-f} for illustration on how these information (especially the latter) are used. Any replay techniques can then be used on the pool of experiences $\mathcal{D}$ in place of the function ``Replay" in Algorithms \ref{alg: update-r}-\ref{alg: update-Q}. In the case of on-policy sampling, we can simply replay the latest collected data in $\mathcal{D}$.
	
	\section{Experimental Results of Mean-Variance Analysis} \label{app:mv}
	\begin{figure}[H]
	\centering
	\includegraphics[width=15.2cm, trim={1cm 0 1cm 1cm}, clip]{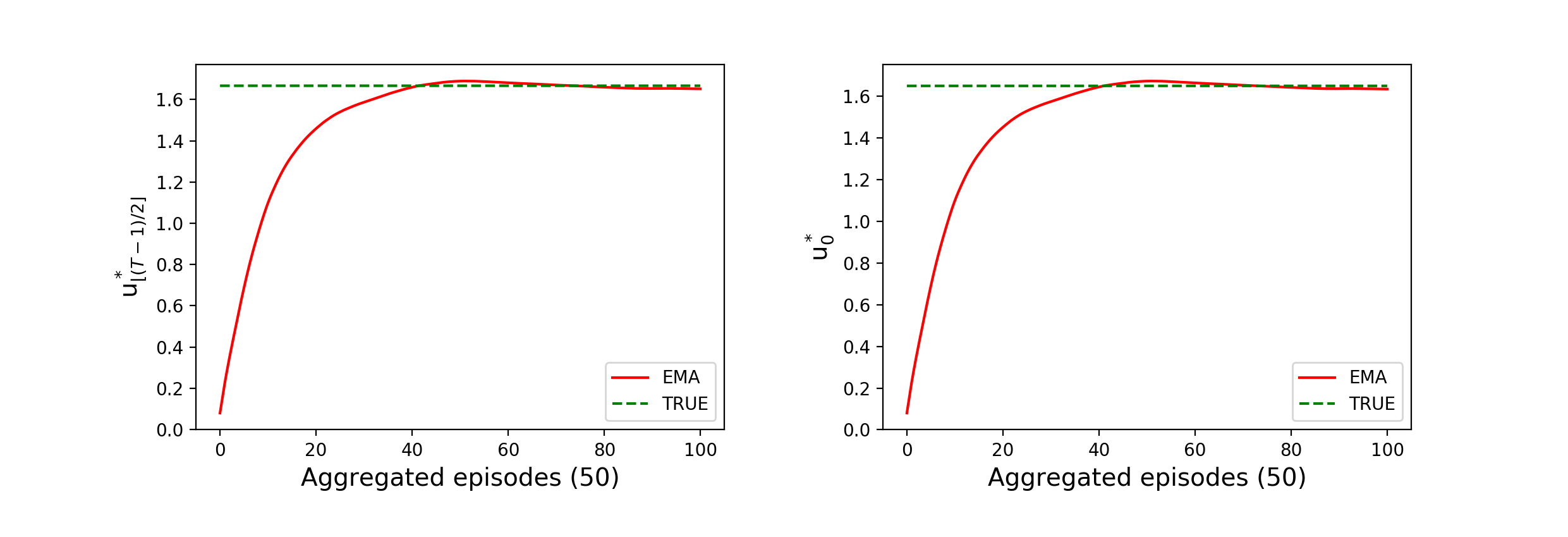}
	\caption{Actor learning curve $(\mu = 20\%)$}
	\label{fig: uopt_pos}
\end{figure}

\begin{figure}[H]
	\centering
	\includegraphics[width=15.2cm, trim={1cm 1cm 1cm 1cm}, clip]{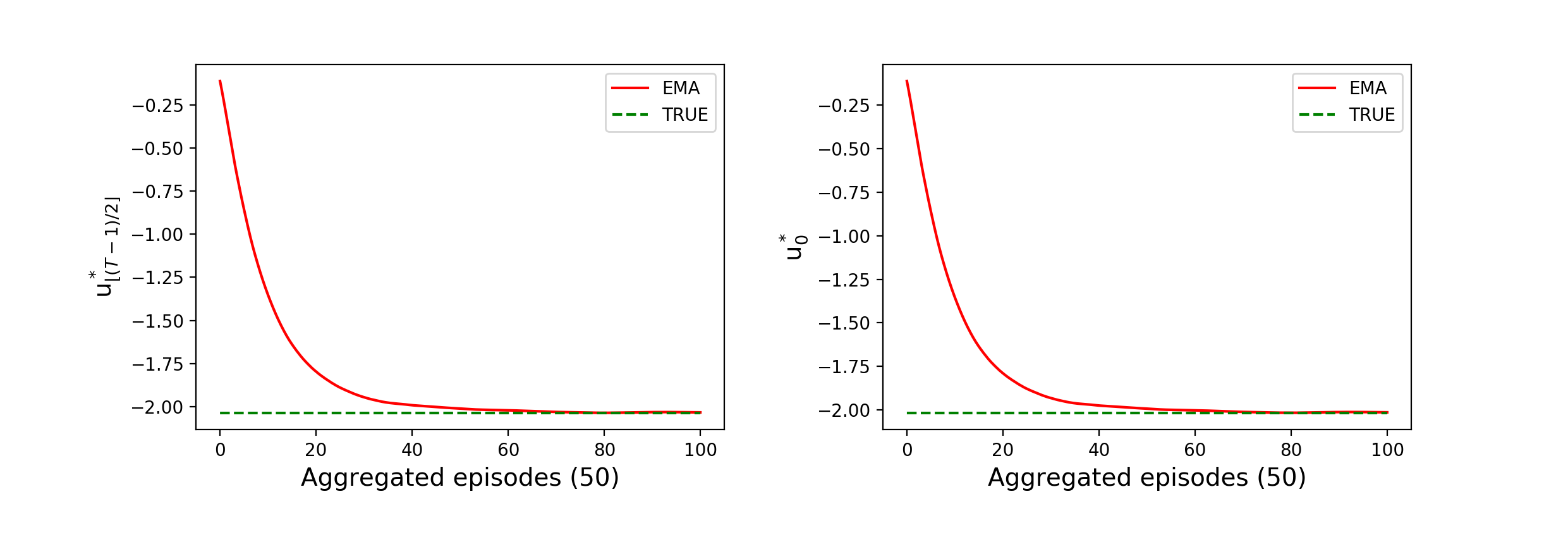}
	\caption{Actor learning curve $(\mu = -20\%)$}
	\label{fig: uopt_neg}
\end{figure}
	
	\vskip 0.2in
	\bibliography{refs}
\end{document}